%% file: main.tex
\documentclass{article} 
\usepackage{iclr2026_conference,times}

\usepackage{mymacros}

\title{Imitation Learning as Return Distribution Matching}

\author{%
Filippo Lazzati\\
Politecnico di Milano\\
Milan, Italy\\
\texttt{filippo.lazzati@polimi.it} \\
\And
Alberto Maria Metelli \\
Politecnico di Milano\\
Milan, Italy\\
}

\usepackage[ruled,linesnumbered]{algorithm2e}
\usepackage{subcaption}
\allowdisplaybreaks[4]

\iclrfinalcopy 
\begin{document}

\maketitle

\begin{abstract}
We study the problem of training a risk-sensitive reinforcement learning (RL)
agent through imitation learning (IL). Unlike standard IL, our goal is not only
to train an agent that matches the expert's expected return (i.e., its
\emph{average performance}) but also its \emph{risk attitude} (i.e., other
features of the return distribution, such as variance).
We propose a general formulation of the risk-sensitive IL problem in which the
objective is to match the expert's return distribution in Wasserstein distance.
We focus on the tabular setting and assume the expert's reward is \emph{known}.
After demonstrating the limited expressivity of Markovian policies for this
task, we introduce an efficient and sufficiently expressive subclass of
non-Markovian policies tailored to it.
Building on this subclass, we develop two provably efficient algorithms—\rsbc
and \rskt—for solving the problem when the transition model is unknown and
known, respectively. We show that \rskt achieves substantially lower sample
complexity than \rsbc by exploiting dynamics information.
We further demonstrate the sample efficiency of return distribution matching in
the setting where the expert's reward is \emph{unknown} by designing an
oracle-based variant of \rskt.
Finally, we complement our theoretical analysis of \rskt and \rsbc with
numerical simulations, highlighting both their sample efficiency and the
advantages of non-Markovian policies over standard sample-efficient IL
algorithms.
\end{abstract}

\section{Introduction}\label{sec: introduction}

Imitation Learning (IL) \citep{abbeel2004apprenticeship,osa2018IL} is the
problem of training an agent to behave by mimicking demonstrations from an
expert. By removing the need for designing a reward function for the task—which
is often a difficult challenge \citep{hadfield2017inverse}—IL has been
successfully applied in diverse domains, including robotics
\citep{argall2009surveyrobotlearning}, autonomous driving
\citep{lemero2022surveyautonomousdrivingIL}, finance
\citep{goluza2023ilfinance}, and LLMs \citep{zhao2025surveylargelanguagemodels}.

Most existing IL algorithms—including \bc (Behavioral Cloning)
\citep{pomerleau1988alvinn}, \texttt{GAIL}
\citep{ho2016generativeadversarialimitationlearning}, and others
\citep{Ziebart2010ModelingPA,reddy2020sqil,garg2021IQlearn}—focus on finding the
\emph{Markovian} policy that best matches the expert's \emph{occupancy measure}.
This focus is motivated by two observations. First, matching occupancy measures
guarantees that the \emph{expected return} of our policy is close to the
expert's, regardless of the expert's \emph{unknown} reward
\citep{abbeel2004apprenticeship}. Second, Markovian policies are sufficiently
expressive. Indeed, for any arbitrary policy, there exists a Markovian policy
with the same occupancy measure \citep{puterman1994markov}.

By focusing solely on the occupancy measure—which captures the expected value of
the return distribution—standard IL algorithms are inherently
\emph{risk-neutral}, ignoring other characteristics of the return distribution
such as the variance \citep{mannor2011meanvarianceoptimizationmarkovdecision}.
However, expert demonstrations often come from humans who, in domains like
finance \citep{foellmer2004stochastic} or autonomous driving
\citep{Bernhard2019AddressingIU}, exhibit \emph{risk-sensitive} behavior under
stochasticity.
In these settings, the key aspect of the demonstrations is the expert's
\emph{risk attitude}, encoded in the shape of the return distribution
\citep{bellemare2023distributional}, but overlooked by standard IL methods.

To address this, \citet{santara2017rail,lacotte2019rsgail} proposed extending
occupancy measure matching to risk-sensitive settings by additionally matching
the Conditional Value at Risk (CVaR) \citep{rockafellar2000cvar} at a chosen
level $\alpha$ of the expert's return distribution.
Intuitively, by seeking the \emph{Markovian} policy that best matches both the
expectation and the CVaR at level $\alpha$, these algorithms attempt to imitate
not only the expert's average performance but also its tail behavior.

However, this approach faces two main limitations: $(i)$ matching only the
expectation and the CVaR at level $\alpha$ captures a narrow slice of the
expert's full return distribution and thus provides a weak imitation of risk
attitude, and $(ii)$ Markovian policies are not expressive enough to capture
all relevant risk-sensitive behaviors \citep{bellemare2023distributional},
leading to misspecification error.
To overcome these challenges, we reformulate risk-sensitive IL as matching the
expert's \emph{entire} return distribution, and we design algorithms that
perform policy search efficiently in the space of non-Markovian policies.

\paragraph{Contributions.} Our main contributions are as follows:
\begin{itemize}[leftmargin=*, topsep=-2pt]
\item We formulate IL as the problem of matching the expert's return
distribution in Wasserstein distance. We motivate this setting and demonstrate
the importance of non-Markovian policies (Section~\ref{sec: IL as return distrib
matching}).
\item We introduce an efficient and expressive subclass of non-Markovian
policies for the tabular setting with a \emph{known} expert reward, and use it
to develop two provably efficient algorithms, \rsbc and \rskt, for the cases
where the transition model is unknown and known, respectively (Section~\ref{sec:
r known}).
\item We show that in the tabular setting with an \emph{unknown} expert reward
but a known transition model, sample efficiency can still be achieved by
devising an oracle-based variant of \rskt (Section~\ref{sec: r unknown}).
\item Finally, we conduct numerical simulations to empirically evaluate \rsbc
and \rskt, comparing them in particular against standard provably efficient IL
algorithms (Section~\ref{sec: num simulations}).
\end{itemize}
All proofs are provided in Appendix~\ref{apx: sec ret distrib match}--\ref{apx:
sec r unknown}, and additional related work is discussed in Appendix~\ref{apx:
rel work}.

\section{Preliminaries}\label{sec: preliminaries}

\paragraph{Notation.}

Given a natural number $n\in\Nat$, we define $\dsb{n}\coloneqq\{1,2,\dotsc,n\}$.
Given a real number $x\in\RR$, we let $\floor{x}\coloneqq\max_{m\in\bZ: m\le x}
m$ be the floor function.
Given two sets $\cX,\cY$, we denote by $\Delta^\cX$ and $\Delta_\cY^\cX$,
respectively, the set of probability measures on $\cX$ and the set of functions
from $\cY$ to $\Delta^\cX$.
Given probabilities $p,q\in\Delta^\cX$ on a finite support $\cX\subset\RR$, with
cumulative distributions $F_p$ and $F_q$, the (1-)Wasserstein distance is
$\cW(p,q)\coloneqq \int_{\RR} |F_p(x)-F_q(x)|dx$ \citep{Villani2008OptimalTO},
and the total variation distance is $\text{TV}(p,q)\coloneqq
\frac{1}{2}\sum_{x\in\cX}|p(x)-q(x)|$.
The CVaR at level $\alpha\in(0,1)$ of $p$ is
$\text{CVaR}_\alpha(p)\coloneqq\frac{1}{\alpha}\int_0^\alpha F_p^{-1}(u)du$,
where $F_p^{-1}(u)\coloneqq \inf_{z\in\RR:F_p(z)\ge u}z$
\citep{rockafellar2000cvar}.

\paragraph{Markov Decision Processes (MDPs).}

A tabular finite-horizon episodic Markov Decision Process without reward
(MDP$\setminus$R) \citep{puterman1994markov,abbeel2004apprenticeship} is a tuple
$\cM\coloneqq\tuple{\cS,\cA,H,s_0,p}$, where $\cS$ is the finite state space
($S\coloneqq |\cS|$), $\cA$ is the finite action space ($A\coloneqq |\cA|$),
$H\in\Nat$ is the horizon, $s_0\in\cS$ is the initial state, and
$p\in\Delta_{\SAH}^\cS$ is the transition model.
An \MDPr $\cM$ can be enriched with a reward $r:\SAH\to[0,1]$, to obtain an MDP
$\cM_r\coloneqq\tuple{\cS,\cA,H,s_0,p,r}$.
We denote the set of state-action trajectories of length $h-1$ as
$\Omega_h\coloneqq (\SA)^{h-1}$ for all $h\in\dsb{H+1}$, and define
$\Omega\coloneqq\bigcup_{h\in\dsb{H}} \Omega_h$.
For any trajectory $\omega=(s_1,a_1,\dotsc,s_h,a_h)$ and reward $r$, we let
$G(\omega;r)=\sum_{h'\in\dsb{h}} r_{h'}(s_{h'},a_{h'})$ denote the sum of rewards of $\omega$.
A policy $\pi$ prescribes actions in states. We denote by
$\Pi^{\text{NM}}\coloneqq\Delta_{\Omega\times\cS}^\cA$ the set of non-Markovian
(i.e., history-dependent) policies,\footnote{Neglecting past rewards in
$\Pi^{\text{NM}}$ is without loss of generality since we consider deterministic
rewards.} and by $\Pi^{\text{M}}\coloneqq\Delta_{\dsb{H}\times\cS}^\cA$ the set
of Markovian policies. Note that $\Pi^{\text{M}}\subseteq \Pi^{\text{NM}}$.
Playing a policy $\pi\in\Pi^{\text{NM}}$ in an \MDPr $\cM$ (or an MDP) induces a
probability distribution over trajectories $\P^\pi \in\Delta^{\Omega_{H+1}}$.
The occupancy measure $d^\pi$ of $\pi$ in $\cM$ is the marginal of $\P^\pi$ over
state-action pairs at a given stage: $d_h^\pi(s,a)\coloneqq
\P^\pi(s_h=s,a_h=a)$.
Given a reward $r$, the \emph{random} sum of rewards $\sum_{h=1}^H r_h(s_h,a_h)$
induced by the execution of $\pi$ is the return, and we denote its distribution,
called the \emph{return distribution} \citep{bellemare2023distributional}, as
$\eta^\pi_r(g)\coloneqq \P^\pi\bigr{\sum_{h=1}^H r_h(s_h,a_h)=g}$ for all
$g\in[0,H]$.
Lastly, we denote by $J^\pi_r$ the expectation of $\eta^\pi_r$.

\paragraph{Imitation Learning (IL).}

In IL, we are given a dataset $\cD^E=\{\omega_i\}_{i\in\dsb{N}}$ of $N$
trajectories $\omega_i\in\Omega_{H+1}$, collected by a (potentially
non-Markovian) expert policy $\pi^E\in\Pi^{\text{NM}}$, and the goal is to
find a policy $\widehat{\pi}$ with expected return close to that of $\pi^E$
under the expert's reward $r^E$ \citep{abbeel2004apprenticeship}:
\begin{align}\label{eq: IL known r}
    \widehat{\pi}\in\argmin\limits_{\pi\in\Pi^{\text{NM}}}
    \Biga{J^{\pi^E}_{r^E}-J^\pi_{r^E}}.
\end{align}
Since $r^E$ is usually unknown, the problem is reformulated in robust terms as
finding a policy that performs comparably to the expert for any possible reward:
\begin{align}\label{eq: standard IL minimax}
  \widehat{\pi}\in\argmin\limits_{\pi\in\Pi^{\text{NM}}}
  \max\limits_{r
  :\SAH\to[0,1]
  }\Biga{J^{\pi^E}_r-J^\pi_r}.
\end{align}
Interestingly, \citet{abbeel2004apprenticeship,ho2016generativeadversarialimitationlearning}
showed that Eq.~\eqref{eq: standard IL minimax} essentially reduces to finding
a policy $\widehat{\pi}$ whose occupancy measure $d^{\widehat{\pi}}$ is close to
the expert's $d^{\pi^E}$. Thus, Eq.~\eqref{eq: standard IL minimax} (and
Eq.~\ref{eq: IL known r}) can be addressed with Markovian policies
$\Pi^{\text{M}}$, which are known to be expressive enough for occupancy measure
matching problems (e.g., see \citet{laroche2023occupancy}).
Based on these insights, recent theoretical work
\citep{rajaraman2020fundamentalimitationlearning,foster2024bcallyouneed}
demonstrated that IL can be solved \emph{provably} efficiently.

\paragraph{Risk-sensitive IL.}

Optimizing Eqs.~\eqref{eq: IL known r}--\eqref{eq: standard IL minimax}
guarantees imitation of the expert's \emph{average performance}, i.e., its
expected return $J^{\pi^E}_{r^E}$, but does not guarantee imitation of its
\emph{risk attitude}, encoded in the shape of its return distribution
$\eta^{\pi^E}_{r^E}$.
For this reason, \citet{santara2017rail,lacotte2019rsgail} proposed
strengthening the standard IL formulation by also matching the CVaR at some
level $\alpha\in(0,1)$ of $\eta^{\pi^E}_{r^E}$ in addition to $J^{\pi^E}_{r^E}$.
Formally, in the unknown $r^E$ setting, they propose:\footnote{The formulation
of \citet{santara2017rail} is slightly different, as they require matching the
expectation while optimizing the CVaR; however, the high-level idea and the
issues with non-Markovian policies remain the same.}
\begin{align}\label{eq: formulation IL cvar}
  \widehat{\pi}\in\argmin\limits_{\pi\in\Pi^{\text{NM}}}\max\limits_{r
  :\SAH\to[0,1]  }
  \biggr{\Bigr{J^{\pi^E}_r-J^\pi_r}+\rho\Bigr{\text{CVaR}_\alpha
  (\eta^{\pi^E}_r)-\text{CVaR}_\alpha(\eta^{\pi}_r)}},
\end{align}
where {\thinmuskip=3mu \medmuskip=3mu \thickmuskip=3mu$\rho(x)=x$ if
$x\le0$ and $+\infty$} otherwise.
This extension, however, makes the problem substantially harder than standard
IL, since the optimal solution to Eq.~\eqref{eq: formulation IL cvar} cannot, in
general, be found among the Markovian policies $\Pi^{\text{M}}$ (even if $r^E$
was known).\footnote{Indeed, since the optimal policy to a CVaR optimization
problem is, in general, non-Markovian \citep{bauerle2011cvar}, the solution to
Eq.~\eqref{eq: formulation IL cvar} also belongs to $\Pi^{\text{NM}}$ due to the
hard constraint enforced by $\rho$.}
Nevertheless, \citet{santara2017rail,lacotte2019rsgail} ignored this aspect and
proposed algorithms that output Markovian policies, which, however, may not be
suited for general non-Markovian experts like humans
\citep{mandlekar2022whatmatters}.

\section{Return Distribution Matching}
\label{sec: IL as return distrib matching}

Our goal is to train agents that match not only the expert's expected return but
also its risk attitude. Standard IL is unsuitable since it ignores risk, while
existing risk-sensitive IL methods only capture a limited aspect of the expert's
return distribution, i.e., the CVaR at a fixed level.
We therefore propose an alternative formulation, called \emph{return
distribution matching} (RDM), which requires matching the \emph{entire} expert
return distribution \scalebox{0.93}{$\eta^{\pi^E}_{r^E}$} in Wasserstein
distance:
\begin{align}\label{eq: return distribution matching}
  \widehat{\pi}\in\argmin_{\pi\in\Pi^{\text{NM}}}
  \cW\Bigr{\eta_{r^E}^{\pi},\eta_{r^E}^{\pi^E}}.
\end{align}
This objective extends Eq.~\eqref{eq: IL known r} and assumes knowledge of the
expert reward $r^E$. Our focus will primarily be on this known-reward setting
(see Section~\ref{sec: r known}), both because it is of independent interest
(similar to inverse constrained RL \citep{malik2021icrl} and utility learning
\citep{lazzati2025utility}), and because it provides a foundation for the more
challenging unknown-reward case, which we next formalize.
When $r^E$ is unknown, following Eqs.~\eqref{eq: standard IL
minimax}--\eqref{eq: formulation IL cvar}, we propose a \emph{robust} version of
RDM:
\begin{align}\label{eq: robust return distribution matching}
  \widehat{\pi}\in\argmin_{\pi\in\Pi^{\text{NM}}}\max_{r:\SAH\to[0,1]}
  \cW\Bigr{\eta_{r}^{\pi},\eta_{r}^{\pi^E}},
\end{align}
which requires matching the expert's return distribution for \emph{all} possible
rewards (see Section~\ref{sec: r unknown}).
We now provide three key observations about RDM.
First, matching return distributions in Wasserstein distance is strictly more
general than matching expected return or CVaR. Indeed, if
$\cW(\eta_{r^E}^{\widehat{\pi}},\eta_{r^E}^{\pi^E})\le \epsilon$, then
$J_{r^E}^{\pi^E}-J_{r^E}^{\widehat{\pi}}\le\epsilon$ and
$\biga{\text{CVaR}_\alpha(\eta_{r^E}^{\pi^E})-
\text{CVaR}_\alpha(\eta_{r^E}^{\widehat{\pi}})}\le \epsilon/\alpha$ for
\emph{any} $\alpha\in(0,1)$ (see Appendix~\ref{apx: probl formulation}).
Second, Wasserstein distance is essential for favorable sample complexity.
Stronger metrics, such as total variation, require exponentially many expert
trajectories in some instances, even if the MDP $\cM_{r^E}$ is fully known (see
Appendix \ref{apx: lower bound total variation} for the proof):
\begin{restatable}{thr}{lowerboundtv}\label{thr: lower bound TV} There exist an
  MDP $\cM_{r^E}$ with $S,A,H\ge2$ and an expert policy
  $\pi^E\in\Pi^{\text{NM}}$ such that, even with $N= (S-1)^{H-1}-1$
  trajectories, any algorithm $\fA$ satisfies
  \begin{align*}
    \E_{\cD^E\sim\P^{\pi^E}}\text{TV}
    \Bigr{\eta^{\pi^E}_{r^E}, \eta^{\widehat{\pi}}_{r^E}}
    \ge \frac{1}{2e},
  \end{align*}
  where $\widehat{\pi}$ is the output of $\fA$ given in input $\cM_{r^E}$ and
  $\cD^E$.
\end{restatable}
Finally, we remark that Markovian policies are not expressive enough for RDM,
since they fail to capture the behavior of non-Markovian experts even for the
simpler risk-sensitive IL problem in Eq. \eqref{eq: formulation IL cvar}. Note
that the gap can be significant even with very short horizons (proof in Appendix
\ref{apx: lower bound total variation}):
\begin{restatable}{prop}{noilofnmpolicies}\label{prop: no il of nm policies}
  There exist an MDP $\cM_{r^E}$ with horizon $H=3$ and an expert policy
  $\pi^E\in\Pi^{\text{NM}}$ such that \emph{any} Markovian policy
  $\pi\in\Pi^{\text{M}}$ satisfies
  \begin{align*}
    \cW\Bigr{\eta^{\pi^E}_{r^E},\eta^{\pi}_{r^E}}\ge 0.5.
  \end{align*}
\end{restatable}
Therefore, new algorithms that output non-Markovian policies are needed to
tackle RDM effectively.

\section{Known-Reward Setting}\label{sec: r known}

In this section, we assume the expert's reward $r^E$ is known and present our
main contributions.
In Section~\ref{sec: policy class}, we introduce an efficient and sufficiently
expressive subset of non-Markovian policies for RDM.
Building on this, in Sections~\ref{sec: rsbc} and~\ref{sec: rskt}, we develop
two provably efficient algorithms, \rsbc and \rskt, for the cases where the
transition model is unknown and known, respectively.

\subsection{An Efficient Class of Non-Markovian Policies}
\label{sec: policy class}

Proposition~\ref{prop: no il of nm policies} shows that Markovian policies
$\Pi^{\text{M}}$ are not expressive enough for RDM.
At the same time, optimizing Eq.~\eqref{eq: return distribution matching} over
the entire set of non-Markovian policies $\Pi^{\text{NM}}$ is intractable due to
the curse of dimensionality.
In this section, we introduce a subclass of non-Markovian policies
$\Pi(r^E_\theta)$, lying between $\Pi^{\text{NM}}$ and $\Pi^{\text{M}}$, that
allows us to address RDM accurately without sacrificing efficiency.
The trade-off between accuracy and efficiency is controlled by a parameter
$\theta>0$.
To this end, we first establish some notation.
For any reward $r$, define $\Pi(r)\subseteq\Pi^{\text{NM}}$ as the set of
policies whose choice of action depends only on the current state $s$, stage
$h$, and the cumulative reward so far $G(\omega;r)$:
\begin{align*}
  \scalebox{0.935}{$  \displaystyle
    \Pi(r)\coloneqq\Bigc{\pi\in\Pi^{\text{NM}}\,\Big|\,
    \exists \phi\in\Delta^\cA_{\dsb{H}\times\cS\times\cG_{r}}:
    \,\pi(a|s,\omega)=\phi_h(a|s,G(\omega;r))\;
   \forall s\in\cS,a\in\cA,h\in\dsb{H},\omega\in\Omega_h
    },
    $}
\end{align*}
where $\cG_{r}\coloneqq\{g\in[0,H-1]\,| \,\exists \omega\in\Omega:\,
G(\omega;r)=g\}$ denotes the set of possible cumulative reward values attainable
at any stage except the last.\footnote{Note that $\cG_{r}$ is always finite in
tabular MDPs with deterministic rewards.}
Observe that each $\pi\in\Pi(r)$ can be interpreted as a Markovian policy in the
MDP obtained by augmenting the state space $\cS$ with the cumulative reward
$\cG_r$.
Next, for any reward $r$ and expert policy $\pi^E\in\Pi^{\text{NM}}$, define
$\pi_{r}\in\Pi(r)$ as the policy whose probability of taking an action $a$ in
state $s$ with history $\omega\in\Omega_h$ coincides with the ``average''
probability with which $\pi^E$ selects $a$ in $s$ after accumulating
$G(\omega;r)$ reward:
\begin{align}\label{eq: def policy imitate same occ meas}
    \pi_{r}(a|s,\omega)\coloneqq
      \frac{\P^{\pi^E}(s_h=s,\;a_h=a,\;\sum_{h'=1}^{h-1} r_{h'}(s_{h'},a_{h'})=G(\omega;r))}{
        \P^{\pi^E}(s_h=s,\;\sum_{h'=1}^{h-1} r_{h'}(s_{h'},a_{h'})=G(\omega;r))}.
\end{align}
If the denominator is zero, we set $\pi_{r}(a|s,\omega)=1/A$.
With these two ingredients, $\Pi(r)$ and $\pi_{r}$, we can state the following
important intermediate result (see Appendix~\ref{apx: proofs lemmas policy
class} for the proof):
\begin{restatable}{lemma}{lemmasameretdistrib}\label{lemma: same return distribution}
    Let $\cM_{r^E}$ be any MDP and let $\pi^E\in\Pi^{\text{NM}}$ be any policy.
    Then, the policy $\pi_{r^E}\in\Pi(r^E)$ satisfies
    $\eta^{\pi_{r^E}}_{r^E}(g)=\eta^{\pi^E}_{r^E}(g)$ for all $g\in[0,H]$.
\end{restatable}
In words, Lemma~\ref{lemma: same return distribution} guarantees that
$\Pi(r^E)$ always contains at least one policy with \emph{exactly} the same
return distribution as the expert, namely one that minimizes
Eq.~\eqref{eq: return distribution matching}.
Moreover, it provides an analytical expression for such a policy,
$\pi_{r^E}$ (see Eq.~\ref{eq: def policy imitate same occ meas}).
Unfortunately, $\Pi(r^E)$ is not always desirable. As discussed in
Appendix~\ref{apx: too many policies}, there exist reward functions $r^E$ for
which $\Pi(r^E)=\Pi^{\text{NM}}$, and, thus, $\pi_{r^E}$ may be an arbitrary
non-Markovian policy, inefficient to store.
Intuitively, this occurs when each trajectory yields a different return value,
leading to the \emph{exponential} dependence $|\cG_{r^E}|=(SA)^{H-1}$.
To overcome this limitation, inspired by prior work
\citep{bastani2022discretizationrisk,lazzati2025utility}, we adopt a
\emph{discretization} approach.
Given a parameter $\theta\in(0,1]$, we define a $\theta$-covering of the
interval $[0,h-1]$ as $\cY_h^\theta\coloneqq\{0,\theta,2\theta,
\dotsc,\floor{h-1/\theta}\theta\}$ for all $h\in\dsb{H+1}$, and set
$\cY^\theta\coloneqq\cY^\theta_{H+1}$.
Then, for any reward $r$, we construct the discretized reward $r_\theta$ as
(breaking ties arbitrarily):
\begin{align}\label{eq: discretized r}
   r_{\theta,h}(s,a)\coloneqq \argmin\limits_{x\in\cY^\theta_2}
    |x-r_h(s,a)|, \qquad \forall (s,a,h)\in\SAH.
\end{align}
Crucially, note that $\cG_{r_\theta}\subseteq\cY^\theta$ for any reward $r$,
since summing discretized rewards always yields discretized values.
Because $\cY^\theta$ has ``small'' (polynomial) size,
$|\cY^\theta|=\cO(H/\theta)$, the policy set $\Pi(r^E_\theta)$ is also small,
$|\Pi(r^E_\theta)|\ll|\Pi^{\text{NM}}|$, and every policy in
$\Pi(r^E_\theta)$, including $\pi_{r_\theta^E}$, can be efficiently stored (with
$\cO(SAH|\cY^\theta|)$ memory).
The following lemma shows that the approximation error introduced by using
policies in $\Pi(r_\theta^E)$ instead of $\Pi(r^E)$ for RDM can be tightly
controlled by~$\theta$ (see Appendix~\ref{apx: proofs lemmas policy class} for
the proof):
\begin{restatable}{lemma}{lemmasameretdistribwithlesspolicies}\label{lemma: apx policies}
  Let $\theta\in(0,1]$.
  Let $\cM_{r^E}$ be any MDP and $\pi^E\in\Pi^{\text{NM}}$ any policy.
  Then, the policy $\pi_{r_\theta^E}\in\Pi(r_\theta^E)$ satisfies
  $    \cW\bigr{\eta^{\pi_{r_\theta^E}}_{r^E},\eta^{\pi^E}_{r^E}}
    \le H\theta$.
\end{restatable}
In short, Lemma~\ref{lemma: apx policies} shows that \emph{efficient and
accurate} solutions to the RDM problem can be sought within $\Pi(r_\theta^E)$. In
particular, $\Pi(r_\theta^E)$ contains $\pi_{r_\theta^E}$, whose error can be
reduced by decreasing $\theta$, at the cost of increased memory requirements for
storing the policy, which scale as $\cO(1/\theta)$.
In the next two sections, we show how to build efficient RDM algorithms based on
$\Pi(r_\theta^E)$ and $\pi_{r_\theta^E}$.

\subsection{No-Interaction Setting}\label{sec: rsbc}

\begin{figure}[t!]
\centering
\begin{minipage}[t]{0.98\linewidth}
\input{rs_bc.tex}
\end{minipage}
\end{figure}

In this section, we present \rsbc (\rsbclong, Algorithm~\ref{alg: rsbc}),\footnote{The
``behavior cloning'' part in \rsbc comes from the intuition that \rsbc can be
seen as performing \bc after augmenting the state space with the (discretized)
cumulative reward.} a provably efficient algorithm for RDM in the no-interaction
(offline) setting, where we neither know nor have access to the transition model
of the environment $\cM$, and are instead given only a dataset $\cD^E$ of $N$
expert trajectories together with the expert's reward $r^E$.
The idea of \rsbc is simple: directly use $\cD^E$ to estimate policy
$\pi_{r_\theta^E}$, whose return distribution is guaranteed by
Lemma~\ref{lemma: apx policies} to be close to that of the expert.
\rsbc estimates $\pi_{r_\theta^E}(a|s,\omega)$ as the fraction of the times in
$\cD^E$ that the expert took action $a$ in state $s$ after collecting
$G(\omega;r_\theta^E)$ discretized cumulative reward (see Line~\ref{line: bc
retrieve policy}).
This ``empirical'' estimator follows closely the definition of $\pi_{r_\theta^E}$
in Eq.~\eqref{eq: def policy imitate same occ meas}, with probability terms
replaced by counts $M$ (computed at Line~\ref{line: bc init N}).
The next result shows that \rsbc is sample-efficient by providing a worst-case
upper bound on its sample complexity (proof in Appendix~\ref{apx: proof thr
rsbc}):
\begin{restatable}{thr}{thrbcsamplecompknownr}
\label{thr: rsbc}
Let $\epsilon\in(0,H]$ and $\delta\in(0,1)$.
Let $\cM_{r^E}$ be any MDP and let $\pi^E\in\Pi^{\text{NM}}$ be any policy.
Then, choosing $\theta=\epsilon/(4H)$, with probability at least $1-\delta$, the
policy $\widehat{\pi}$ output by Algorithm~\ref{alg: rsbc} satisfies
$\cW(\eta^{\pi^E}_{r^E}, \eta^{\widehat{\pi}}_{r^E})\le\epsilon$,
with a number of samples:
\begin{align}\label{eq: sample complexity rsbc}
    N\le \widetilde{\cO}\biggr{\frac{SH^6\ln\frac{1}{\delta}}
{\epsilon^3}\Bigr{A+ \ln\frac{1}{\delta}}}.
\end{align}
\end{restatable}
In words, Theorem~\ref{thr: rsbc} shows that \rsbc requires a polynomial (in the
quantities of interest $S,A,H,1/\epsilon,\ln(1/\delta)$) number of samples to
output a good imitation policy for RDM with high
probability.\footnote{$\widetilde{\cO}$ notation omits logarithmic terms in
$S,A,H,1/\epsilon,\ln(1/\delta)$.}
Compared to the best existing upper bound for standard IL,
$\widetilde{\cO}(SAH^3/\epsilon^2\ln(1/\delta))$ (Corollary~3.1 of
\citet{foster2024bcallyouneed}), we observe a gap of
$\cO(H^3/\epsilon\ln(1/\delta))$. This is reasonable, as RDM appears more
complex than occupancy measure matching (e.g., it requires non-Markovian
policies).
We conjecture that the $\epsilon$ gap is unimprovable, while the $\cO(H^6)$
dependence may be large but is comparable to the $\cO(H^5)$ rate in the related
setting of IL from observation alone (Theorem~3.3 of
\citet{Sun2019ProvablyEI}). See Appendix~\ref{apx: more discussion rsbc} for
further discussion.
Finally, \rsbc is also computationally efficient, since both
Lines~\ref{line: bc init N}--\ref{line: bc retrieve policy} require only
$\cO(SAH|\cY^\theta|)$ iterations and memory.

\subsection{Known-Transition Setting}\label{sec: rskt}

\begin{figure}[t!]
  \centering
  \begin{minipage}[t]{0.98\linewidth}
      \input{rs_kt.tex}
  \end{minipage}
\end{figure}

In this section, we present \rskt (\rsktlong, Algorithm~\ref{alg: rskt}), a
provably efficient algorithm for RDM in the known-transition setting, where we
have access to the transition model $p$ of the environment, in addition to the
expert's dataset $\cD^E$ and reward $r^E$.
By leveraging knowledge of $p$, \rskt achieves a drastic reduction in sample
complexity compared to \rsbc.
The idea behind \rskt is straightforward.
First, use the expert dataset $\cD^E$ to compute an estimate $\widehat{\eta}$ of
the expert's return distribution $\eta^{\pi^E}_{r^E}$ (Line~\ref{line: kt
estimate expert ret distrib}).
Then, exploit knowledge of $r^E$ and $p$ to identify the policy in
$\Pi(r_\theta^E)$ whose return distribution is closest to $\widehat{\eta}$
(Line~\ref{line: kt compute policy}).
We now show that \rskt is sample efficient (see Appendix~\ref{apx: proof thr
rskt} for proof):
\begin{restatable}{thr}{rsktupperbound}\label{thr: rskt}
Let $\epsilon\in(0,H]$ and $\delta\in(0,1)$.
Let $\cM_{r^E}$ be any MDP and $\pi^E\in\Pi^{\text{NM}}$ any policy.
Assume that the optimization problem in
Line~\ref{line: kt compute policy} is solved exactly.
Then, choosing $\theta=\epsilon/(7H)$, with probability $1-\delta$, the policy
$\widehat{\pi}$ output by Algorithm~\ref{alg: rskt} satisfies
$\cW(\eta^{\pi^E}_{r^E}, \eta^{\widehat{\pi}}_{r^E})\le\epsilon$,
with:
\begin{align}\label{eq: sample complexity rskt}
    N\le \cO\biggr{\frac{H^2}{\epsilon^2}\ln\frac{1}{\delta}}.
\end{align}
\end{restatable}
Interestingly, Theorem~\ref{thr: rskt} shows that \rskt has sample complexity
\emph{independent} of $S$ and $A$, making it substantially more sample efficient
than any algorithm for standard IL in large MDPs, where $\Omega(S)$ samples are
required even with knowledge of $p$ (Theorem~5.1 of
\citet{rajaraman2020fundamentalimitationlearning}).
We believe the $\cO(1/\epsilon^2)$ rate is tight, as it matches the lower bound
for estimating a distribution in Wasserstein distance (Theorem~3.1 of
\citet{bobkov2019onedimensional}).
Compared to \rsbc, the reduction in sample complexity is drastic,
$\cO(SAH^4/\epsilon)$.
Extending \rskt to settings where $p$ is unknown but can be estimated from
online interaction with the environment is an interesting direction for future
work (see \citet{xu2023AILunknownp} for the standard IL case).
If Line~\ref{line: kt compute policy} is solved with error
$\epsilon_{\text{apx}}$, then Theorem~\ref{thr: rskt} guarantees
$\cW(\eta^{\pi^E}_{r^E}, \eta^{\widehat{\pi}}_{r^E}) \le \epsilon +
\epsilon_{\text{apx}}$.
Finally, we refer the reader to Appendices~\ref{apx: rE in finite set} and
\ref{apx: rE linear} for extensions of \rsbc and \rskt to the cases where $r^E$
is unknown but either belongs to a given finite set or is linear in a given
feature map, and to Appendix \ref{apx: gen bc arb problems} for greater
generalization.

We now turn to the computational complexity of \rskt. Crucially, the
optimization problem in Line~\ref{line: kt compute policy} can be formulated as
a linear program (LP) with a polynomial number
$\cO\!\left(SAH|\cY^\theta|\right)$ of variables and constraints, and therefore
can be solved efficiently.
To see this, let $\overline{\cM}\coloneqq(\overline{\cS}, \cA,H,\overline{s}_0,
\overline{p})$ be the \MDPr with augmented state space
$\overline{\cS}\coloneqq\cS\times\cY^\theta$, initial state
$\overline{s}_0=(s_0,0)\in\overline{\cS}$, and transition model $\overline{p}$
defined as $\overline{p}_h(s,g|s',g',a)\coloneqq
p_h(s'|s,a)\indic{g=g'+r^E_{\theta,h}(s',a')}$, i.e., the probability of
reaching $(s,g)\in\overline{\cS}$ by playing $a$ in $(s',g')$ at stage $h$
coincides with $p_h(s'|s,a)$ if the reward received in $s',a'$ is $g-g'$, and 0
otherwise.
Since $\Pi(r_\theta^E)$ coincides with the set of Markovian policies in
$\overline{\cM}$ \citep{bauerle2014more,lazzati2025utility}, we can rewrite
Line~\ref{line: kt compute policy} as a variant of the occupancy measure
matching problem in $\overline{\cM}$:
\begin{align}\label{eq: opt problem LP rskt}
  \min_{d\in\cK,\eta\in\Delta^{\cY^\theta}}\cW\!\left(\eta,\widehat{\eta}\right)
  \quad\text{s.t. }\;\eta(g)=\sum_{(s,a)\in\SA}
  d_H(s,g-r_{\theta,H}^E(s,a),a)
  \quad \forall g\in \cY^\theta,
\end{align}
where the constraint enforces that $\eta$ is the return distribution induced by
the occupancy measure $d$, and $\cK$ denotes the set of feasible occupancy
measures in $\overline{\cM}$ \citep{puterman1994markov}:
\begin{align*}
  \scalebox{0.95}{$  \displaystyle
  \cK\coloneqq\Bigc{d\in\Delta_{\dsb{H}}^{\overline{\cS}\times\cA}\,\Big|\,
  \sum_a d_1(\overline{s}_0,a)=1
  \wedge\forall \overline{s}\in\overline{\cS},h\ge 2:\;
  \sum_a d_h(\overline{s},a)=
\sum_{\overline{s}',a'}d_{h-1}(\overline{s}',a')\overline{p}_{h-1}(\overline{s}|\overline{s}',a')}
  $}.
\end{align*}
In words, Eq.~\eqref{eq: opt problem LP rskt} searches for an occupancy measure
$d\in\cK$ that induces the return distribution $\eta$ closest to
$\widehat{\eta}$.
From such a solution, a policy $\widehat{\pi}\in\Pi(r_\theta^E)$ with occupancy
$d^{\widehat{\pi}}=d$ (and thus return distribution
$\eta^{\widehat{\pi}}_{r^E_\theta}=\eta$) can be recovered via
\begin{align*}
\widehat{\pi}(a|s,\omega)=\frac{d_h(s,G(\omega;r_\theta^E),a)}{
\sum_{a'}d_h(s,G(\omega;r_\theta^E),a')} \quad
\forall h\in\dsb{H},\;s\in\cS,\;a\in\cA,\;\omega\in\Omega_h,
\end{align*}
when the denominator is nonzero, and $\widehat{\pi}(a|s,\omega)=1/A$
otherwise \citep{syed2008allinear}.
Observe that all the constraints in Eq.~\eqref{eq: opt problem LP rskt} are
linear in $d$ and $\eta$, and that the objective
$\cW\!\left(\eta,\widehat{\eta}\right)$ can also be written linearly (see
\citet{peyre2019computationalOT} and Appendix~\ref{apx: details LP}).

\section{Statistical Insights on the Unknown-Reward Setting}\label{sec: r
unknown}

In this section, we assume that the expert's reward $r^E$ is unknown, and
provide some \emph{statistical} insights on RDM in the ``robust'' form of Eq.
\eqref{eq: robust return distribution matching}.
Specifically, we show that, perhaps surprisingly, this complex problem requires
only a polynomial number of expert demonstrations to be accurately solved when
the transition model is known, even in the \emph{worst case}.
To establish this, we first prove that a polynomial number of expert
demonstrations suffices to accurately estimate the expert's return distribution
\emph{under any reward} (proof in Appendix \ref{apx: sec r unknown}):

\begin{restatable}{thr}{estanypolicyallrewards}
  \label{thr: est any policy all rewards}
Let $\epsilon\in(0,H]$ and $\delta\in(0,1)$.
Let $\cM_{r^E}$ be any MDP and $\pi^E\in\Pi^{\text{NM}}$ any policy.
Then, choosing $\theta=\epsilon/(2H)$, a number of samples
\begin{align}\label{eq: sample complexity est any r}
    N\le \widetilde{\cO}\biggr{\frac{SAH^3}{\epsilon^2}\ln\frac{1}{\delta}},
\end{align}
suffices to guarantee that, with probability at least $1-\delta$, for the
estimator $\widehat{\eta}_r(g)\coloneqq \frac{1}{N} \sum_{\omega\in\cD^E}
\indic{G(\omega;{r_\theta})=g}$ $\forall g,r$, we have:
\begin{align*}
    \max\limits_{r:\SAH\to[0,1]}\cW\Bigr{
      \eta^{\pi^E}_r,\widehat{\eta}_r
    }\le \epsilon.
\end{align*}
\end{restatable}

In brief, an expert dataset $\cD^E$ of size in Eq.~\eqref{eq: sample complexity
est any r} suffices to accurately estimate the expert's return distribution
$\eta^{\pi^E}_r$ under any reward $r$ via the estimator $\widehat{\eta}_r$.
As a consequence, any policy $\widehat{\pi}$ that induces return distributions
close to $\widehat{\eta}_r$ for all possible rewards $r$ accurately solves the
robust RDM problem:

\begin{restatable}{thr}{upperboundexpcompl}\label{thr: upper bound exp compl}
  Under the conditions of Theorem \ref{thr: est any policy all rewards}, assume
  access to a computational oracle that takes as input the dataset $\cD^E$ and
  the transition model $p$, and outputs a solution to
  \begin{align}\label{eq: oracle}
    \widehat{\pi}\in\argmin_{\pi\in\Pi^{\text{NM}}}\max_{r:\SAH\to[0,1]}
    \cW\Bigr{ \eta^{\pi}_r,\widehat{\eta}_r }.
  \end{align}
  Then, with probability at least $1-\delta$, using the number of samples in
  Eq.~\eqref{eq: sample complexity est any r}, it holds that
  \begin{align*}
    \max_{r:\SAH\to[0,1]}\cW\Bigr{
      \eta^{\pi^E}_r,\eta^{\widehat{\pi}}_r
    }\le 2\epsilon.
  \end{align*}
\end{restatable}
In words, Theorem \ref{thr: upper bound exp compl} establishes that the robust
RDM problem is sample efficient for any algorithm that accurately solves
Eq.~\eqref{eq: oracle}.
Intuitively, such an algorithm can be viewed as an extension of \rskt to the
unknown-reward setting.
Note that the minimization in Eq.~\eqref{eq: oracle} is over $\Pi^{\text{NM}}$,
since a satisfactory policy may no longer exist in the policy class
$\Pi(r^E_\theta)$, which is also unknown due to the unknown reward.
We leave to future work the interesting question of whether an (approximate)
solution to Eq.~\eqref{eq: oracle} can be computed efficiently, and note that
restricting the optimization to a subset $\Pi\subset\Pi^{\text{NM}}$ can
reduce computation time at the cost of some misspecification error.

\section{Numerical Simulations}\label{sec: num simulations}

\begin{table}[!t]
  \centering
  \resizebox{0.9\columnwidth}{!}{
  \begin{tabular}{||c | c c c c c||} 
   \hline
    & $N=20$ & $N=80$ & $N=300$ & $N=1000$ & $N=10000$\\
   \hline
   \rsbc & \small \bf 0.081±0.039 & \small \bf0.038±0.016 & \small
   \bf0.022±0.013 & \small \bf0.012±0.005 & \small \bf 0.005±0.002\\
   \hline
    \rskt & \small \bf0.095±0.036 & \small \bf0.049±0.017 & \small \bf0.03±0.013
    & \small \bf0.019±0.007 & \small \bf 0.011±0.006 \\
   \hline
    \bc & \small \bf0.099±0.056 & \small 0.076±0.054 & \small 0.072±0.056 & \small 0.069±0.058 & \small 0.068±0.058 \\
   \hline
    \mimic & \small 0.127±0.062 & \small 0.086±0.055 & \small 0.074±0.056 & \small 0.07±0.057 & \small 0.068±0.058 \\
   \hline
  \end{tabular}%
  }
\vspace{0.7em}\\
  \centering
  \resizebox{0.9\columnwidth}{!}{
  \begin{tabular}{||c | c c c c c||} 
   \hline
    & $N=20$ & $N=80$ & $N=300$ & $N=1000$ & $N=10000$\\
   \hline
   \rsbc & \small \bf0.087±0.04 & \small \bf0.051±0.022 & \small \bf0.035±0.015 & \small \bf0.027±0.016 & \small \bf0.022±0.016\\
   \hline
    \rskt & \small 0.144±0.053 & \small 0.119±0.039 & \small 0.109±0.04 & \small 0.108±0.038 & \small 0.106±0.039 \\
   \hline
    \bc & \small 0.103±0.057 & \small 0.08±0.053 & \small 0.072±0.056 & \small 0.069±0.058 & \small 0.068±0.058 \\
   \hline
    \mimic & \small 0.132±0.065 & \small 0.09±0.055 & \small 0.076±0.055 & \small 0.071±0.057 & \small 0.068±0.058 \\
   \hline
  \end{tabular}%
  }
\vspace{0.7em}\\
  \centering
  \resizebox{0.9\columnwidth}{!}{
  \begin{tabular}{||c | c c c c c||} 
   \hline
    & $N=20$ & $N=80$ & $N=300$ & $N=1000$ & $N=10000$\\
   \hline
   \rsbc & \small 0.102±0.031 & \small 0.052±0.015 & \small \bf0.026±0.008 & \small \bf0.015±0.005 & \small \bf0.004±0.001\\
   \hline
    \rskt & \small 0.118±0.036 & \small 0.059±0.017 & \small 0.031±0.009 &
    \small 0.021±0.007 & \small \bf 0.01±0.004 \\
   \hline
    \bc & \small \bf0.085±0.035 & \small \bf0.041±0.016 & \small \bf0.021±0.008 & \small \bf0.012±0.005 & \small \bf0.003±0.002 \\
   \hline
    \mimic & \small 0.132±0.052 & \small 0.06±0.022 & \small 0.03±0.01 & \small \bf0.016±0.006 & \small \bf0.005±0.002 \\
   \hline
  \end{tabular}%
  }
\vspace{0.7em}\\
  \centering
  \resizebox{0.9\columnwidth}{!}{
  \begin{tabular}{||c | c c c c c||} 
   \hline
   & $N=20$ & $N=80$ & $N=300$ & $N=1000$ & $N=10000$\\
   \hline
   \rsbc & \small 0.169±0.079 & \small 0.168±0.079 & \small 0.165±0.081 & \small 0.165±0.081 & \small 0.166±0.081\\
   \hline
    \bc & \small 0.168±0.078 & \small 0.166±0.078 & \small 0.169±0.085 & \small 0.177±0.091 & \small 0.174±0.093 \\
   \hline
    $\qquad\,\widehat{\eta}\,\qquad$ & \small 0.169±0.049 & \small \bf0.08±0.018 & \small
    \bf0.043±0.01 & \small \bf0.024±0.006 & \small \bf 0.008±0.002 \\
   \hline
  \end{tabular}%
  } \caption{Results of the simulations described in Section~\ref{sec: num
  simulations}. The best results in each column are highlighted in bold. (Top)
  Simulation with $S, A, H = (2,2,5)$ for Q1. (Upper middle) Simulation with
  $\theta = 0.5$ for Q2. (Lower middle) Simulation with a Markovian expert for
  Q3. (Bottom) Simulation with $S, A, H = (300,5,5)$ for Q4.}
\label{table: Q}
\end{table}

In this section, we study \rsbc and \rskt from a practical perspective by
conducting simulations aimed at answering the following four questions:
\begin{enumerate}
  \item What is the performance improvement of \rsbc and \rskt on the RDM
  problem compared to standard IL algorithms?
  \item How are the results affected by the choice of $\theta$?
  \item What happens if the expert's policy is Markovian?
  \item Does \rskt truly reduce sample complexity compared to \rsbc, as
  predicted by theory?
\end{enumerate}
We select \bc \citep{foster2024bcallyouneed} and \mimic
\citep{rajaraman2020fundamentalimitationlearning,rajaraman2021provablybreakingquadraticerror}
as baseline IL algorithms for, respectively, the no-interaction and
known-transition settings, and address all questions by conducting various
simulations following the next three-step process.
$(i)$ First, we randomly generate 50 MDPs and expert policies with fixed size
$S,A,H$ (details in Appendix \ref{apx: details on sampling}).
$(ii)$ Second, for each number of expert trajectories $N\in\{20, 80, 300, 1000,
10000\}$, we collect three different expert datasets of $N$ trajectories for
each MDP, provide them as input to the four algorithms \rsbc, \rskt, \bc, and
\mimic, and record the average (over the three seeds) Wasserstein distance
between the expert's return distribution and the return distribution induced by
each algorithm's output policy.
$(iii)$ Finally, we average these distances across all 50 MDPs, obtaining, for
each algorithm and value of $N$, a number representing the expected error of
that algorithm when given access to $N$ expert trajectories in the considered
setting.

Below, we describe the specific simulations conducted and the results obtained
for each question.

\paragraph{Question 1 (Q1).}

To address Q1, we conducted three simulations with different problem sizes
$S,A,H\in\{(2,2,5),(50,5,5),(2,2,20)\}$, all with non-Markovian policies and
$\theta=0.05$.
The results for $(2,2,5)$ are reported in Table \ref{table: Q} (top), while the
other two are shown in Tables \ref{table: exp 3}--\ref{table: exp 6} in Appendix
\ref{apx: details Q1}.
Crucially, Table \ref{table: Q} (top) reveals that \bc and \mimic, by relying on
\emph{Markovian} policies, are biased and cannot match the expert's return
distribution satisfactorily, even with a large dataset of $N=10000$ trajectories. 
In contrast, by leveraging our efficient non-Markovian policy class, \rsbc and
\rskt continue to reduce the error as the number of trajectories $N$ increases,
up to a limit determined by our choice of $\theta$.
A similar pattern is observed for larger $S,A$ in Table \ref{table: exp 3},
and especially for larger horizons $H$ in Table \ref{table: exp 6}, where the
limited expressivity of Markovian policies becomes even more pronounced.
We also note that larger $S,A,H$ slows down \rskt significantly due to solving
an LP with poly($S,A,H$) variables and constraints, and increases the
approximation error due to $\theta$, as discussed in Appendix
\ref{apx: details Q1} and in Q2.

\paragraph{Question 2 (Q2).}

To address Q2, we first conducted a simulation with an increased $\theta=0.5$,
keeping $S,A,H=(2,2,5)$ and a non-Markovian expert. As shown in Table
\ref{table: Q} (upper middle), a larger $\theta$ consistently increases the
approximation error, causing \rsbc and \rskt to perform worse than with
$\theta=0.05$ (see Table \ref{table: Q} (top)). Nevertheless, while \rsbc
remains fairly robust and continues to outperform \bc and \mimic—intuitively
because it corresponds to \bc with a more expressive policy class—\rskt tends
to reach the worst-case approximation error predicted by Lemma
\ref{lemma: apx policies}, $H\theta/2 \approx 0.2$, as also discussed in
Appendix \ref{apx: details Q1}.
Additionally, we conducted three further simulations with values of $\theta$
small enough to eliminate approximation error, observing a consistent increase
in performance, particularly for \rskt in settings with larger $S,A,H$ (see
Appendix \ref{apx: details Q2}).

\paragraph{Question 3 (Q3).}

We ran a simulation with a \emph{Markovian} expert and $S,A,H=(2,2,5)$, with
results reported in Table \ref{table: Q} (lower middle). Interestingly, \bc
outperforms all other algorithms in terms of both sample and computational
efficiency. Intuitively, this occurs because it operates on a much smaller
hypothesis space (i.e., $\Pi^{\text{M}}$) than \rsbc and \rskt, and it is
\emph{unbiased}, since the expert is Markovian. The policy output by \bc aims
to match the expert's trajectory distribution \citep{foster2024bcallyouneed},
and thus also its return distribution for any reward, which is not the case for
\mimic, explaining its comparatively worse performance.

\paragraph{Question 4 (Q4).} 

According to our theoretical results (Theorems \ref{thr: rsbc} and
\ref{thr: rskt}), \rskt should outperform \rsbc in terms of sample complexity
for large $S,A$, as its performance does not depend on them. To verify this, we
ran a simulation with large $S,A,H=(300,5,5)$, a non-Markovian expert, and
$\theta=0.05$.
To speed up computation (particularly the LP in \rskt), we avoided running \rskt
directly and instead compared the expert's return distribution with the estimate
$\widehat{\eta}$ computed at Line \ref{line: kt estimate expert ret distrib}. By
the triangle inequality, this guarantees that the error between the output of
\rskt and the expert is at most twice the error of $\widehat{\eta}$.
Results are reported in Table \ref{table: Q} (bottom), showing a dramatic
improvement in sample complexity. In particular, both \rsbc and \bc struggle
even with $N= 10000$, while \rskt achieves significant performance with as few as
$N=300$ or $1000$ trajectories.

In summary, the key takeaways of this section are: 
\begin{itemize}[leftmargin=*, topsep=-2pt]
  \item \rsbc and \rskt generally outperform \bc and \mimic due to the use of
  more expressive non-Markovian policies and reward information, particularly
  for large $H$.
  \item \rsbc is faster than \rskt because it does not require solving a linear
  program and is more robust to large $\theta$, while \rskt is much more sample
  efficient for large MDPs.
  \item \bc performs well when the expert is Markovian, even without access to
  reward information.
\end{itemize}

\section{Conclusion}\label{sec: conclusion}

In this paper, we introduced and analyzed RDM, a general formulation of
(risk-sensitive) IL as the problem of matching the expert's return distribution.
Remarkably, we showed that both the known- and unknown-reward settings are
statistically tractable. For the known-reward case, we proposed two algorithms,
\rsbc and \rskt, which not only come with strong theoretical guarantees but also
empirically outperform standard IL methods at accurately matching the expert's
return distribution.

\textbf{Limitations and future directions.}~~%
This work has several limitations. We focus only on the tabular setting, and our
theoretical analysis lacks lower bounds, so it remains unclear whether the upper
bounds in Theorems \ref{thr: rsbc}, \ref{thr: rskt}, and \ref{thr: upper bound
exp compl} are tight. Furthermore, the unknown-reward setting lacks a practical
algorithm, and our empirical study does not include real-world data.
Future work could extend our results to state-only feedback settings
\citep{Sun2019ProvablyEI}, develop practical and scalable versions of \rsbc and
\rskt for large or continuous environments, and design algorithms for the
unknown-reward setting.

\bibliography{refs.bib}
\bibliographystyle{iclr2026_conference}

\newpage
\appendix

\section{Additional Related Work}\label{apx: rel work}


\paragraph{Standard IL.}

There are two main approaches to address the standard occupancy measure matching
formulation of IL: Behavior Cloning (BC), that treats IL as a supervised
learning problem, directly learning a mapping from states to actions
\citep{pomerleau1988alvinn,ross2010reductions}, and Inverse Reinforcement
Learning (IRL) that infers a reward function that reflects the expert's
preferences, and then derives a policy via planning
\citep{ng2000algorithms,abbeel2004apprenticeship,Ziebart2010ModelingPA,finn2016guidedcostlearning,Fu2017LearningRR}.
Other popular occupancy measure matching methods include
\citet{ho2016generativeadversarialimitationlearning,kostrikov2018discriminatoractorcritic,reddy2020sqil,Brantley2020Disagreement-Regularized,garg2021IQlearn,dadashi2021primal}.
Recently, there have been various efforts into providing theoretical guarantees
on BC \citep{rajaraman2020fundamentalimitationlearning,foster2024bcallyouneed}
and IRL \citep{metelli2021provably,metelli2023towards}, and also for related IL
algorithms \citep{Sun2019ProvablyEI,viano2024imitationlearningdiscountedlinear}.
In this paper we provide a theoretical analysis analogous to that in
\citet{rajaraman2020fundamentalimitationlearning,foster2024bcallyouneed}, but
for the novel return distribution matching setting.
%

\paragraph{Risk-sensitive RL.} 

The first paper to address risk sensitivity in Markov Decision Processes (MDPs)
is \citet{howard1972risk}. Since then, various researchers have explored this
problem (see the survey
\citet{prashant2022risksensitivereinforcementlearningpolicy} and the recent book
\citet{bellemare2023distributional}).
Notably, \citet{mannor2011meanvarianceoptimizationmarkovdecision} emphasizes the
importance of non-Markovian policies in mean-variance optimization.
Additionally, \citet{bauerle2011cvar} and \citet{bauerle2014more} show that
optimal behavior in CVaR and expected utility planning problems can generally be
achieved using non-Markovian policies that base actions on both the current
state and the cumulative reward. This idea is exploited in Section~\ref{sec: r
known} to construct our algorithms.
Some risk-sensitive RL algorithms employing such non-Markovian policies include
\citet{haskell2015convexrisk,chow2018risk}.

\paragraph{Risk-Sensitive IL.} 

Some works extend occupancy measure matching with additional objectives to
capture the expert's risk attitude.
\citet{santara2017rail} and \citet{lacotte2019rsgail} are most similar to ours,
as they aim to match both expected return and CVaR. However, as noted in the
introduction, they restrict to Markovian policies, which limits their ability to
fully capture the expert's risk preferences.
\citet{ratliff2017inverse} propose a risk-sensitive IRL algorithm focused on
learning risk parameters, also assuming a Markovian expert.
\citet{majumdar2017risk} study risk-sensitive IL with a form of non-Markovian
policy, but in environments far simpler than tabular MDPs.
\citet{lazzati2025utility} consider non-Markovian experts and imitation, but
primarily aim to recover the expert's utility, and their IL approach struggles
when demonstrations come from a single environment. Nevertheless, our policy
class in Section \ref{sec: r known} draws inspiration on their discretization
approach.
\citet{muni2025what} also mentions the importance of non-Markovian policies for
risk-sensitive IL, but do not present any algorithm.

\paragraph{Non-Markovian IL.}

The importance of adopting non-Markovian policies for IL has already been
recognized by some IL works, especially in the field of robotics.
\citet{mandlekar2020iris,mandlekar2022whatmatters} identify one cause of
non-Markovianity in human demonstrations in partial observability,
and consider variants of BC with recurrent neural networks for imitation.
\citet{zhao2023chunking} and subsequent literature (e.g., see
\citet{torne2025learninglongcontextdiffusionpolicies,ren2025diffusion}), learn
non-Markovian policies implicitly through action chunking, an open-loop control
technique in which at each state our policy outputs a ``chunk'' (i.e., a
sequence) of actions.
We mention \citet{qin2023nonmarkIL}, which learn non-Markovian policies as
energy-based priors from state-only sequences, and \citet{block2023ILnonmark},
which study imitation in non-linear systems.
Importantly, none of these works address non-Markovian policies arising from
risk sensitivity, which is a novel aspect of ours. Moreover, note that fitting
general non-Markovian policies may require an amount of data exponential in the
horizon in the worst-case (see Appendix~\ref{apx: no eff non-markov}).

\section{Additional Results and Proofs for Section \ref{sec: IL as return distrib matching}}
\label{apx: sec ret distrib match}

In this appendix, we first show that matching the return distribution in
Wasserstein distance is strictly more expressive than matching the expectation
or the CVaR at a given level (Appendix \ref{apx: probl formulation}). Next, we
provide the proof of other results in Section \ref{sec: IL as return distrib
matching} (Appendix \ref{apx: lower bound total variation}). Finally, we prove
that matching the trajectory distribution of arbitray non-Markovian policies is
sample inefficient (Appendix \ref{apx: no eff non-markov}).

\subsection{Additional Insights on RDM}
\label{apx: probl formulation}

We show here that matching the expert's return distribution in Wasserstein
distance implies closeness in terms of expected return, the variance of the
return, and the CVaR at any level.

First, observe that, for any MDP $\cM_{r^E}$ and policies $\pi^E,\widehat{\pi}$:
\begin{align*}
    J^{\pi^E}_{r^E}-J^{\widehat{\pi}}_{r^E}\le
\cW\bigr{\eta_{r^E}^{\pi^E},\eta_{r^E}^{\widehat{\pi}}}.
\end{align*}
This follows after having realized that the identity function (denote below as
$I(\cdot)$) is $1-$Lipschitz, and using the dual form of the Wasserstein
distance (see Eq. 6.3 of \citet{Villani2008OptimalTO}).
\begin{align*}
    J^{\pi^E}_{r^E}-J^{\widehat{\pi}}_{r^E}&=\E_{X\sim \eta_{r^E}^{\pi^E}}[X]-
    \E_{Y\sim \eta_{r^E}^{\widehat{\pi}}}[Y]\\
    &=\E_{X\sim \eta_{r^E}^{\pi^E}}[\com{I(X)}]-
    \E_{Y\sim \eta_{r^E}^{\widehat{\pi}}}[\com{I(Y)}]\\
    &\le\com{\sup\limits_{f:\|f\|_{\text{Lip}}\le1}}\E_{X\sim \eta_{r^E}^{\pi^E}}[\com{f(X)}]-
    \E_{Y\sim \eta_{r^E}^{\widehat{\pi}}}[\com{f(Y)}]\\
&=\cW\bigr{\eta_{r^E}^{\pi^E},\eta_{r^E}^{\widehat{\pi}}}.
\end{align*}

Second, for any $\alpha\in(0,1)$, we have:
\begin{align*}
    |\text{CVaR}_\alpha(\eta_{r^E}^{\pi^E})-
\text{CVaR}_\alpha(\eta_{r^E}^{\widehat{\pi}})|\le \frac{1}{\alpha}
\cW\bigr{\eta_{r^E}^{\pi^E},\eta_{r^E}^{\widehat{\pi}}}.
\end{align*}
This follows using an alternative expression for the Wasserstein
distance. Formally, for any pair of distributions $p,q$ with support on $[0,H]$:
\begin{align*}
    \biga{\text{CVaR}_\alpha(p)-\text{CVaR}_\alpha(q)}&=
    \Biga{\frac{1}{\alpha}
    \int^\alpha_0 \Bigr{F^{-1}_p(x)-F^{-1}_q(x)}dx}\\
    &\le \frac{1}{\alpha}
    \int^\alpha_0 \com{\Big|}F^{-1}_p(x)-F^{-1}_q(x)\com{\Big|}dx\\
    &\le \frac{1}{\alpha}
    \com{\int_0^1} \Biga{F^{-1}_p(x)-F^{-1}_q(x)}dx\\
    &\markref{(1)}{=} \frac{1}{\alpha}\cW(p,q),
\end{align*}
where at (1) we use that the 1-Wasserstein distance can be written this way
\citep{Panaretos2019wasserstein}, where recall that $F_p^{-1}(x)\coloneqq
\inf_{z\in\RR:F_p(z)\ge x}z$.

Lastly, we observe that closedness in Wasserstein distance also implies
closedness between the variance of returns:
\begin{align*}
    |\text{var}(\eta_{r^E}^{\pi^E})-
\text{var}(\eta_{r^E}^{\widehat{\pi}})|\le 4H\cdot
\cW\bigr{\eta_{r^E}^{\pi^E},\eta_{r^E}^{\widehat{\pi}}}.
\end{align*}
To see it, for any pair of distributions $p,q$ with support on $[0,H]$, we can write:
\begin{align*}
    &|\text{var}(p)-\text{var}(q)|\\
    &\qquad\qquad=\Biga{
        \E_{X\sim p}[X^2]-(\E_{X\sim p}[X])^2
        -\E_{Y\sim q}[Y^2]+(\E_{Y\sim q}[Y])^2
    }\\
    &\qquad\qquad\le \biga{\E_{X\sim p}[X^2]-\E_{Y\sim q}[Y^2]}
    + \biga{(\E_{X\sim p}[X])^2-(\E_{Y\sim q}[Y])^2}\\
    &\qquad\qquad\markref{(1)}{\le} 2H\sup\limits_{f:\|f\|_{\text{Lip}}\le1}
    \biga{\E_{X\sim p}[f(X)]-\E_{Y\sim q}[f(Y)]}\\
    &\qquad\qquad\qquad\qquad+
    \biga{(\E_{X\sim p}[X]-\E_{Y\sim q}[Y])\cdot(\E_{X\sim p}[X]+\E_{Y\sim q}[Y])}\\
    &\qquad\qquad\markref{(2)}{\le} 2H\cdot \cW(p,q)+2H \biga{\E_{X\sim p}[X]-\E_{Y\sim q}[Y]}\\
    &\qquad\qquad\markref{(3)}{\le} 4H\cdot\cW(p,q),
\end{align*}
where at (1) we use that the function $f:x\to x^2$ is $2H-$Lipschitz on $[0,H]$
(since $\sup_{x\neq y}\frac{|f(x)-f(y)|}{|x-y|}=\sup_{x\neq
y}\frac{|x^2-y^2|}{|x-y|}=\sup_{x\neq y}(x+y)=2H$ for $x,y\in[0,H]$), and we
rescaled to obtain functions that are $1-$Lipschitz, at (2) we use the dual form
of the Wasserstein distance (see Eq. 6.3 of \citet{Villani2008OptimalTO}) after
having noticed that the absolute value can be removed as set
$\{f:\|f\|_{\text{Lip}}\le1\}$ is symmetric, and we upper bounded $\E_{X\sim
p}[X]+\E_{Y\sim q}[Y]\le 2H$. Finally, at (3) we use the result proved earlier.

\subsection{Proofs}
\label{apx: lower bound total variation}

\lowerboundtv*
\begin{proof}
  To prove this result, we simply provide a reward function $r^E:\SAH\to[0,1]$
  for which, in the ``hard'' MDP used in the proof of Theorem \ref{thr: lower
  bound distrib over trajectories}, we have:
  \begin{align*}
    \text{TV}\Bigr{\eta^{\pi^E}_{r^E},\eta^{\widehat{\pi}}_{r^E}}=\frac{1}{2}
    \Bign{\P^{\pi^E}-\P^{\widehat{\pi}}}_1,
  \end{align*}
  for any pair of policies. Then, the claim in the theorem follows directly from
  the result in Theorem \ref{thr: lower bound distrib over trajectories}.

  First, let us associate a different integer $0,1,\dotsc,S-1$ to each state
  $\{s_1,\dotsc,s_S\}$, and a different integer $0,1,\dotsc,A-1$ to each action
  $\{a_1,\dotsc,a_A\}$, and denote $x_s:\cS\to\Nat$ and $x_a:\cA\to\Nat$ these
  mappings. Then, the reward that we provide is the following:
  \begin{align*}
    r^E_h(s,a)\coloneqq 10^{-hSA-x_s(s)A-x_a(a)}.
  \end{align*}
  Simply, observe that every triple $s,a,h$ is associated a reward value
  belonging to a different power of 10, thus when we sum them to compute the
  return of trajectory we obtain a different return value for every possible
  trajectory:
  \begin{align*}
    \text{TV}\Bigr{\eta^{\pi^E}_{r^E},\eta^{\widehat{\pi}}_{r^E}}&\coloneqq
    \frac{1}{2}\sum\limits_{g}\Biga{\eta^{\pi^E}_{r^E}(g)-\eta^{\widehat{\pi}}_{r^E}(g)}\\
    &=
    \frac{1}{2}\sum\limits_{g}\Biga{\sum\limits_{\omega: G(\omega;r^E)=g}
    \P^{\pi^E}(\omega)-\P^{\widehat{\pi}}(\omega)}\\
    &=
    \frac{1}{2}\sum\limits_{\omega}\Biga{
    \P^{\pi^E}(\omega)-\P^{\widehat{\pi}}(\omega)}\\
    &=\frac{1}{2}
    \Bign{\P^{\pi^E}-\P^{\widehat{\pi}}}_1.
  \end{align*}
  This concludes the proof.
\end{proof}

\noilofnmpolicies*
\begin{proof}
  \begin{figure}[t!]
    \centering
    \begin{tikzpicture}[node distance=3.5cm]
    \node[state,initial] at (-1,0) (s0) {$s_{\text{init}}$};
    \node[state] at (3,1) (s1) {$s_1$};
    \node[state] at (3,-1) (s2) {$s_2$};
    \node[state] at (6,0) (s3) {$s_3$};
    \node[state, draw=none] at (9,1.) (ss4) {};
    \node[state, draw=none] at (9,-1.) (ss5) {};
    \node[draw=none, fill=black] at (1.5,0) (ss0) {};
    \draw (s0) edge[-, solid, above] node{\scriptsize$a_1,a_2,r=0$} (ss0);
    \draw (ss0) edge[->, solid, above] node{\scriptsize$1/2$} (s1);
    \draw (ss0) edge[->, solid, above] node{\scriptsize$1/2$} (s2);
    \draw (s1) edge[->, solid, above, sloped] node{\scriptsize$a_1,a_2,r=1$} (s3);
    \draw (s2) edge[->, solid, above, sloped] node{\scriptsize$a_1,a_2,r=0$} (s3);
    \draw (s3) edge[->, solid, above, sloped] node{\scriptsize$a_1,r=0$} (ss4);
    \draw (s3) edge[->, solid, above, sloped] node{\scriptsize$a_2,r=1$} (ss5);
  \end{tikzpicture}
  \caption{MDP for the proof of Proposition \ref{prop: no il of nm policies}.}
  \label{fig: mdp for proof no il of nm policies}
  \end{figure}
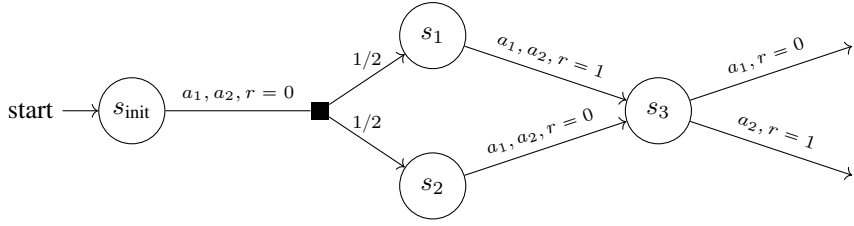

  Consider the MDP $\cM_{r^E}=\tuple{\cS,\cA,H,s_0,p,r^E}$ in Fig. \ref{fig: mdp
  for proof no il of nm policies}, where $\cS=\{s_{\text{init}},s_1,s_2,s_3\}$,
  $\cA=\{a_1,a_2\}$, $H=3$, $s_0=s_{\text{init}}$, the transition model $p$ is
  such that:
  \begin{align*}
    &p_1(s_1|s_{\text{init}},a)=p_1(s_2|s_{\text{init}},a)=1/2 \quad\forall a\in\cA,\\
    &p_2(s_3|s_1,a)=p_2(s_3|s_2,a)=1 \quad\forall a\in\cA,\\
    &p_3(s_3|s_3,a)=1,
  \end{align*}
  and the reward function $r^E$ is defined as:
  \begin{align*}
    &r^E_1(s_{\text{init}},a)=0\quad\forall a\in\cA,\\
    &r^E_2(s_1,a)=1\quad\forall a\in\cA,\\
    &r^E_2(s_2,a)=0\quad\forall a\in\cA,\\
    &r^E_3(s_3,a_1)=0,\\
    &r^E_3(s_3,a_2)=1.
  \end{align*}
  Let $\pi^E$ be the deterministic non-Markovian policy that plays always action
  $a_1$ every time it is not in $s_3$, and then in $s_3$ plays $a_1$ if
  previously we passed through $s_1$, otherwise play $a_2$. Formally, for any
  $\omega\in\Omega$ and $s\in\cS$:
  \begin{align*}
    \pi^E(a_1|s,\omega)=\begin{cases}
      1&\text{if }s\neq s_3\vee (s=s_3\wedge \omega=(s_{\text{init}},a_1,s_1,a_1))\\
      0&\text{otherwise}
    \end{cases}.
  \end{align*}
  It is easy to note that the return distribution of $\pi^E$ in $\cM_{r^E}$ is:
  \begin{align*}
    \eta^{\pi^E}_{r^E}=\delta_1,
  \end{align*}
  where notation $\delta_x$ denotes the Dirac delta on $x$. With $\delta_1$, we
  always get return 1.

  Now, let $\pi^\alpha$ be any Markovian policy parameterized by
  $\alpha\in[0,1]$ as (we do not specify values in other states and stages,
  because they are not relevant for the return distribution):
  \begin{align*}
    \pi_3^\alpha(a_1|s_3)=\alpha.
  \end{align*}
  The return distribution of $\pi^\alpha$ in $\cM_{r^E}$ is:
  \begin{align*}
    \eta^{\pi^\alpha}_{r^E}=\frac{\alpha}{2}\delta_0 + \frac{1}{2}\delta_1
    +\frac{1-\alpha}{2}\delta_2,
  \end{align*}
  namely, irrespective of the policy $\pi^\alpha$, the return distribution is 1
  w.p. $1/2$.

  Let us compute the Wasserstein distance between these return distributions:
  \begin{align*}
    \cW\Bigr{\eta^{\pi^E}_{r^E},\eta^{\pi^\alpha}_{r^E}}
    &=\int_{-\infty}^{+\infty}\Biga{F_{\eta^{\pi^E}_{r^E}}(g)
    -F_{\eta^{\pi^\alpha}_{r^E}}(g)}dg\\
    &=1\cdot\Biga{0-\frac{\alpha}{2}}
    +1\cdot\Biga{1-\frac{1+\alpha}{2}}\\
    &=\frac{1}{2},
  \end{align*}
  irrespective of $\alpha$. This concludes the proof.
  
\end{proof}

\subsection{Statistical Inefficiency of General non-Markovian Policies}\label{apx:
no eff non-markov}

In this appendix, we provide an explicit proof that the problem of matching the
\emph{distribution over trajectories} of an arbitrary non-Markovian expert's
policies may require an exponential (in the horizon) amount of data. Then, the
proof of Theorem \ref{thr: lower bound TV} can be simply obtained through a
reduction to this problem. The lower bound is provided in the next theorem,
which makes use of the family of hard instances in Fig. \ref{fig: hard instance
trajectories}.

\begin{restatable}{thr}{lowerboundgeneral}\label{thr: lower bound distrib over
trajectories}
  There exist an \MDPr $\cM=\tuple{\cS,\cA,H,s_0,p}$ with $S\ge 2$, $A\ge 2,H\ge2$ and
  a non-Markovian expert's policy $\pi^E\in\Pi^{\text{NM}}$, such that any
  learning algorithm $\fA$ taking in input a dataset of $N= (S-1)^{H-1}-1$
  expert's trajectories $\cD^E=\{\omega_i\}_{i\in\dsb{N}}$ satisfies:
  \begin{align*}
    \E_{\cD^E \sim \P^{\pi^E}}\Bign{\P^{\pi^E}-\P^{\widehat{\pi}}}_1
    \ge\frac{1}{e},
  \end{align*}
  where $\widehat{\pi}$ is the policy outputted by $\fA$ when taking in input
  both $\cD^E$ and $\cM$.
\end{restatable}
\begin{proof}
  The proof draws inspiration from that of Theorem 5.1 in
  \citet{rajaraman2020fundamentalimitationlearning} for the known-transition
  setting.

    \begin{figure}[t]
     \centering
     \begin{tikzpicture}
     \node[state,initial] at (0,0) (s0) {$s_1$};
     \node[state] at (3,2) (s1) {$s_1$};
     \node[state] at (3,1) (s2) {$s_2$};
     \node[state, draw=none] at (3,0) (dots) {$\dotsc$};
     \node[state] at (3,-1) (ss1) {\scriptsize$s_{S-1}$};
     \node[state] at (3,-2) (ss) {$s_S$};
     \node[state] at (6,2) (s21) {$s_1$};
     \node[state] at (6,1) (s22) {$s_2$};
     \node[state, draw=none] at (6,0) (dots) {$\dotsc$};
     \node[state] at (6,-1) (s2s1) {\scriptsize$s_{S-1}$};
     \node[state] at (6,-2) (s2s) {$s_S$};
     \node[state] at (10,2) (s31) {$s_1$};
     \node[state] at (10,1) (s32) {$s_2$};
     \node[state, draw=none] at (10,0) (dots) {$\dotsc$};
     \node[state] at (10,-1) (s3s1) {\scriptsize$s_{S-1}$};
     \node[state] at (10,-2) (s3s) {$s_S$};
     \node[draw=none, fill=black] at (1.5,0) (ss0) {};
     \node[draw=none, fill=black] at (4.,2) (s2s0) {};
     \node[draw=none, fill=black] at (4.,1) (s3s0) {};
     \node[draw=none, fill=black] at (4.,-1) (s4s0) {};
     \node[draw=none] at (8,2) (dottss) {$\dotsc$};
     \node[draw=none] at (8,1) (dottss) {$\dotsc$};
     \node[draw=none] at (8,0) (dottss) {$\dotsc$};
     \node[draw=none] at (8,-1) (dottss) {$\dotsc$};
     \node[draw=none] at (8,-2) (dottss) {$\dotsc$};
     \draw (s0) edge[-, solid, above] node{\scriptsize$a$} (ss0);
     \node[draw=none, fill=black] at (7.,2) (ss2s1) {};
     \node[draw=none, fill=black] at (7.,1) (ss3s1) {};
     \node[draw=none, fill=black] at (7.,-1) (ss4s1) {};
     \node[draw=none] at (7.,-2) (ss5s1) {};
     \draw (ss0) edge[->, solid, above] node{\scriptsize$\zeta$} (s1);
     \draw (ss0) edge[->, solid, above] node{\scriptsize$\zeta$} (s2);
     \draw (ss0) edge[->, solid, above] node{\scriptsize$\zeta$} (ss1);
     \draw (ss0) edge[->, solid, below,sloped] node{\scriptsize$1-(S-1)\zeta$}
     (ss);
     \draw (s1) edge[-, solid, above] node{\scriptsize$a$} (s2s0);
     \draw (s2) edge[-, solid, above] node{\scriptsize$a$} (s3s0);
     \draw (s21) edge[-, solid, above] node{\scriptsize$a$} (ss2s1);
     \draw (s22) edge[-, solid, above] node{\scriptsize$a$} (ss3s1);
     \draw (s2s1) edge[-, solid, above] node{\scriptsize$a$} (ss4s1);
     \draw (s2s) edge[->, solid, above] node{\scriptsize$a$} (ss5s1);
     \draw (ss1) edge[-, solid, above] node{\scriptsize$a$} (s4s0);
     \draw (s2s0) edge[->, solid, above] node{\scriptsize$\zeta$} (s21);
     \draw (s2s0) edge[->, solid, above] node{\scriptsize$\zeta$} (s22);
     \draw (s2s0) edge[->, solid, above] node{\scriptsize$\zeta$} (s2s1);
     \draw (s2s0) edge[->, solid, below,sloped] node{\scriptsize$1-(S-1)\zeta$}
     (s2s);
     \draw (ss) edge[->, solid, above,sloped] node{\scriptsize$a$}
     (s2s);
   \end{tikzpicture}
   \caption{The \MDPr $\cM$ used in the proof of Theorem \ref{thr: lower bound
   distrib over trajectories}.}
   \label{fig: hard instance trajectories}
   \end{figure}

We begin by describing the class of hard instances that will be considered for
proving this lower bound. Note that we are considering here the ``IL'' problem
of computing a policy with trajectory distribution close to that of the expert.
Thus, a problem instance is a pair \MDPr-expert's policy.
  Let $\fP=\{(\cM,\pi^E_i)\}_i$ be the family of problem instances where the
  \MDPr $\cM=(\cS,\cA,H,s_0,p)$ is represented in Fig. \ref{fig: hard instance
  trajectories}, and the expert's policy $\pi^E_i$ is any deterministic policy
  that for $h<H$ always plays $a_1$, while at the last stage can play an
  arbitrary action. Formally, $\cS\supseteq\{s_1,s_S\}$ (i.e., it has $S\ge 2$
  at least $s_1,s_S$, and potentially other states $s_2,\dotsc$),
  $\cA=\{a_1,a_2\}$, $H\ge 2$, $s_0=s_1$, and the transition model $p$ is
  defined as a function of a scalar $\zeta\in[0,\frac{1}{S-1}]$ that we will
  choose later:
   \begin{align*}
    p_h(s'|s,a)=\begin{cases}
      \zeta&\text{ if }s\neq s_S\wedge s'\neq s_S\\
      1-(S-1)\zeta&\text{ if }s\neq s_S\wedge s'= s_S\\
      1&\text{ otherwise}\\
    \end{cases},
   \end{align*}
   for any action $a\in\cA$.
   Then, the expert's policy $\pi^E_i$ is defined as any policy in
   $\Pi^{\text{hard}}$, defined as:
  \begin{align*}
    &\Pi^{\text{hard}}\coloneqq\Bigc{\pi\in\Pi^{\text{NM}}\,\Big|\,
    \forall h\in\dsb{H-1},\forall s\in\cS,\forall \omega\in\Omega_h:\, \pi(a_1|s,\omega)=1}.
  \end{align*}
  Then, $\fP$ is formally defined as:
  \begin{align*}
    \fP\coloneqq\Bigc{(\cM',\pi')\,\Big|\, \cM'=\cM\wedge \pi'\in\Pi^{\text{hard}}}.
  \end{align*}
  We set $\zeta\coloneqq\frac{1}{(N+1)^{1/(H-1)}}$, and observe that, to
  guarantee that $\cM$ exists, we need to enforce that $(S-1)\zeta$, i.e., the
  total probability assigned to reaching a state $\neq s_S$, is smaller than 1.
  So:
  \begin{align*}
    (S-1)\zeta\le 1 \iff \frac{S-1}{(N+1)^{\frac{1}{H-1}}}\le 1
    \iff N \ge (S-1)^{H-1}-1.
  \end{align*}

  To proceed, we choose a distribution $\cP$ over problem instances in $\fP$. We
  select $\cP\in\Delta^\fP$ as the uniform probability distribution over the
  family of problems $\fP$, i.e.:
  \begin{align*}
    \forall (\cM',\pi')\in\fP:\; \cP(\cM',\pi') = \frac{1}{|\fP|}.
  \end{align*}
  Then, if we can show that:
  \begin{align*}
    \E_{(\cM',\pi')\sim \cP} \E_{\cD^E\sim \P^{\pi'}}\Bign{\P^{\pi'}-\P^{\widehat{\pi}}}_1
    \ge\frac{(S-1)^{H-1}}{e(N+1)},
  \end{align*}
  then we can conclude that there is at least a problem instance
  $(\overline{\cM},\overline{\pi})\in\fP$ such that:
  \begin{align*}
    \E_{\cD^E\sim \P^{\overline{\pi}}}\Bign{\P^{\overline{\pi}}-\P^{\widehat{\pi}}}_1
    \ge\frac{(S-1)^{H-1}}{e(N+1)},
  \end{align*}
  from which if we insert the choice $N= (S-1)^{H-1}-1$, we obtain the
  claim of the theorem.

To this aim, we can write:
    \begin{align*}
    &\E_{(\cM',\pi')\sim \cP} \E_{\cD^E\sim \P^{\pi'}}\Bign{\P^{\pi'}-\P^{\widehat{\pi}}}_1\\
    &\qquad\coloneqq\E_{(\cM',\pi')\sim \cP} \E_{\cD^E\sim \P^{\pi'}} \sum\limits_{\omega\in\Omega_{H+1}}
    \Biga{\P^{\pi'}(\omega)-\P^{\widehat{\pi}}(\omega)}\\
    &\qquad\markref{(1)}{=}\com{\E_{\cD^E\sim q}
    \E_{(\cM',\pi')\sim \cP'(\cD^E)}} \sum\limits_{\omega\in\Omega_{H+1}}
    \Biga{\P^{\pi'}(\omega)-\P^{\widehat{\pi}}(\omega)}\\
    &\qquad=\E_{\cD^E\sim q}
    \E_{(\cM',\pi')\sim \cP'(\cD^E)}
    \com{\sum\limits_{s^1,a^1,\dotsc,s^H,a^H}}
    \Big|\P^{\pi'}(\com{s^1,a^1,\dotsc,s^H,a^H})\\
    &\qquad\qquad
    -\P^{\widehat{\pi}}(\com{s^1,a^1,\dotsc,s^H,a^H})\Big|\\
    &\qquad\markref{(2)}{\ge}
    \E_{\cD^E\sim q}
    \E_{(\cM',\pi')\sim \cP'(\cD^E)}
    \com{\sum\limits_{s^2,s^3,\dotsc,s^H}\sum\limits_{a\in\cA}}
    \Big|\P^{\pi'}(\com{s_1,a_1,s^2,a_1,\dotsc,s^H,a})\\
    &\qquad\qquad
    -\P^{\widehat{\pi}}(\com{s_1,a_1,s^2,a_1,\dotsc,s^H,a})\Big|\\
    &\qquad\markref{(3)}{=}
    \E_{\cD^E\sim q}
    \E_{(\cM',\pi')\sim \cP'(\cD^E)}
    \sum\limits_{s^2,s^3,\dotsc,s^H}\sum\limits_{a\in\cA}
    \Big|\com{\rho(s^2)\cdot\rho(s^3)\cdot\dotsc\cdot\rho(s^H)}\\
    &\qquad\qquad
    \com{\cdot\pi'(a|s_1,a_1,s^2,a_1,\dotsc,s^H)}
    -\P^{\widehat{\pi}}(s_1,a_1,s^2,a_1,\dotsc,s^H,a)\Big|\\
    &\qquad\markref{(4)}{=}
    \E_{\cD^E\sim q}
    \E_{(\cM',\pi')\sim \cP'(\cD^E)}
    \sum\limits_{s^2,s^3,\dotsc,s^H}\sum\limits_{a\in\cA}
    \Big|\rho(s^2)\cdot\rho(s^3)\cdot\dotsc\cdot\rho(s^H)\\
    &\qquad\qquad
    \cdot\pi'(a|s_1,a_1,s^2,a_1,\dotsc,s^H)
    -\com{\rho(s^2)\cdot\rho(s^3)\cdot\dotsc\cdot\rho(s^H)}\\
    &\qquad\qquad
    \com{\cdot \widehat{\pi}
    (a|s_1,a_1,s^2,a_1,\dotsc,s^H)
    \cdot \prod\limits_{h\in\dsb{H-1}}\widehat{\pi}
    (a_1|s_1,a_1,s^2,a_1,\dotsc,s^h)}
    \Big|\\
    &\qquad\markref{(5)}{=}
    \E_{\cD^E\sim q}
    \E_{(\cM',\pi')\sim \cP'(\cD^E)}
    \sum\limits_{s^2,s^3,\dotsc,s^H}\com{\rho(s^2)\cdot\rho(s^3)\cdot\dotsc\cdot\rho(s^H)}
    \sum\limits_{a\in\cA}
    \Big|\\
    &\qquad\qquad
    \pi'(a|s_1,a_1,s^2,a_1,\dotsc,s^H)
    - \widehat{\pi}
    (a|s_1,a_1,s^2,a_1,\dotsc,s^H)
    \cdot \com{K}
    \Big|\\
    &\qquad=
    \E_{\cD^E\sim q}
    \sum\limits_{s^2,s^3,\dotsc,s^H}\rho(s^2)\cdot\rho(s^3)\cdot\dotsc\cdot\rho(s^H)
    \com{\E_{(\cM',\pi')\sim \cP'(\cD^E)}}\sum\limits_{a\in\cA}
    \Big|\\
    &\qquad\qquad
    \pi'(a|s_1,a_1,s^2,a_1,\dotsc,s^H)
    - \widehat{\pi}
    (a|s_1,a_1,s^2,a_1,\dotsc,s^H)
    \cdot K
    \Big|\\
    &\qquad\ge
    \E_{\cD^E\sim q}
    \sum\limits_{s^2,s^3,\dotsc,s^H}\rho(s^2)\cdot\rho(s^3)\cdot\dotsc\cdot\rho(s^H)
    \com{\indic{(s_1,a_1,s^2,a_1,\dotsc,s^H)\notin \cD^E}}\\
    &\qquad\qquad\E_{(\cM',\pi')\sim \cP'(\cD^E)}\sum\limits_{a\in\cA}
    \Big|\pi'(a|s_1,a_1,s^2,a_1,\dotsc,s^H)\\
    &\qquad\qquad
    - \widehat{\pi}
    (a|s_1,a_1,s^2,a_1,\dotsc,s^H)
    \cdot K
    \Big|\\
    &\qquad\markref{(6)}{=}
    \E_{\cD^E\sim q}
    \sum\limits_{s^2,s^3,\dotsc,s^H}\rho(s^2)\cdot\rho(s^3)\cdot\dotsc\cdot\rho(s^H)
    \indic{(s_1,a_1,s^2,a_1,\dotsc,s^H)\notin \cD^E}\\
    &\qquad\qquad\com{\sum\limits_{b\in\cA}\frac{1}{A}}
    \sum\limits_{a\in\cA}
    \Big|\com{\indic{a=b}}
    - \widehat{\pi}
    (a|s_1,a_1,s^2,a_1,\dotsc,s^H)
    \cdot K
    \Big|\\
    &\qquad=
    \E_{\cD^E\sim q}
    \sum\limits_{s^2,s^3,\dotsc,s^H}\rho(s^2)\cdot\rho(s^3)\cdot\dotsc\cdot\rho(s^H)
    \indic{(s_1,a_1,s^2,a_1,\dotsc,s^H)\notin \cD^E}\\
    &\qquad\qquad
    \cdot\sum\limits_{b\in\cA}\frac{1}{A}\com{\bigg(
    \Bigr{1
    - \widehat{\pi}
    (b|s_1,a_1,s^2,a_1,\dotsc,s^H)
    \cdot K}}\\
    &\qquad\qquad
    \com{+
    \sum\limits_{a\in\cA\setminus\{b\}}
    \widehat{\pi}
    (a|s_1,a_1,s^2,a_1,\dotsc,s^H)
    \cdot K
    \bigg)}\\
    &\qquad=
    \E_{\cD^E\sim q}
    \sum\limits_{s^2,s^3,\dotsc,s^H}\rho(s^2)\cdot\rho(s^3)\cdot\dotsc\cdot\rho(s^H)
    \indic{(s_1,a_1,s^2,a_1,\dotsc,s^H)\notin \cD^E}\\
    &\qquad\qquad
    \cdot
    \bigg(\com{1+\frac{K}{A}}\sum\limits_{b\in\cA}\Big(\sum\limits_{a\in\cA\setminus\{b\}}
    \widehat{\pi}
    (a|s_1,a_1,s^2,a_1,\dotsc,s^H)\\
    &\qquad\qquad-
    \widehat{\pi}
    (b|s_1,a_1,s^2,a_1,\dotsc,s^H)\Big)
    \bigg)\\
    &\qquad=
    \E_{\cD^E\sim q}
    \sum\limits_{s^2,s^3,\dotsc,s^H}\rho(s^2)\cdot\rho(s^3)\cdot\dotsc\cdot\rho(s^H)
    \indic{(s_1,a_1,s^2,a_1,\dotsc,s^H)\notin \cD^E}\\
    &\qquad\qquad
    \cdot
    \bigg(1+\frac{K}{A}\sum\limits_{b\in\cA}\Big(\com{1-2\widehat{\pi}
    (b|s_1,a_1,s^2,a_1,\dotsc,s^H)}\Big)
    \bigg)\\
    &\qquad=
    \E_{\cD^E\sim q}
    \sum\limits_{s^2,s^3,\dotsc,s^H}\rho(s^2)\cdot\rho(s^3)\cdot\dotsc\cdot\rho(s^H)
    \indic{(s_1,a_1,s^2,a_1,\dotsc,s^H)\notin \cD^E}\\
    &\qquad\qquad
    \cdot
    \big(1+\com{K-2K/A}
    \big)\\
    &\qquad\markref{(7)}{\ge}
    \E_{\cD^E\sim q}
    \sum\limits_{s^2,s^3,\dotsc,s^H}\rho(s^2)\cdot\rho(s^3)\cdot\dotsc\cdot\rho(s^H)
    \indic{(s_1,a_1,s^2,a_1,\dotsc,s^H)\notin \cD^E}\\
    &\qquad\markref{(8)}{=}
    \sum\limits_{s^2,s^3,\dotsc,s^H}\rho(s^2)\cdot\rho(s^3)\cdot\dotsc\cdot\rho(s^H)
    \com{q\bigr{(s_1,a_1,s^2,a_1,\dotsc,s^H)\notin \cD^E}}\\
    &\qquad\markref{(9)}{=}
    \sum\limits_{s^2,s^3,\dotsc,s^H}\rho(s^2)\cdot\rho(s^3)\cdot\dotsc\cdot\rho(s^H)
    \com{\Bigr{1-\rho(s^2)\cdot\rho(s^3)\cdot\dotsc\cdot\rho(s^H)}^N}\\
    &\qquad\ge
    \mathop{\com{\sum}}\limits_{\substack{\com{s^2,s^3,\dotsc,s^H:
    }\\\com{s^h\neq s_S\forall h\in\{2,\dotsc,H\}}}}
    \rho(s^2)\cdot\rho(s^3)\cdot\dotsc\cdot\rho(s^H)
    \Bigr{1-\rho(s^2)\cdot\rho(s^3)\cdot\dotsc\cdot\rho(s^H)}^N\\
    &\qquad=
    \sum\limits_{\substack{s^2,s^3,\dotsc,s^H\\s^h\neq s_S\forall h\in\{2\dotsc,H\}}}
    \com{\frac{1}{N+1}}
    \Bigr{1-\com{\frac{1}{N+1}}}^N\\
    &\qquad\markref{(10)}{\ge}
    \com{\frac{1}{e}}\sum\limits_{\substack{s^2,s^3,\dotsc,s^H\\s^h\neq s_S\forall h\in\{2\dotsc,H\}}}
    \frac{1}{N+1}\\
    &\qquad=
    \frac{1}{e}
    \frac{\com{(S-1)^{H-1}}}{N+1},
  \end{align*}
  where at (1) we define $q\in\Delta^{\Omega_{H+1}}$ as a distribution that
  generates datasets $\cD^E$ as if we sampled first $(\cM'',\pi'')\sim\cP$, and
  then we collected $N$ trajectories from $\pi''$ in $\cM''$, and we define
  $\cP'(\cD^E)$ as the uniform distribution over the set of deterministic
  policies:
  \begin{align*}
    \Pi^{\text{mimic}}(\cD^E)\coloneqq\Bigc{\pi\in\Pi^{\text{hard}}\,\Big|\,
    \forall h\in\dsb{H-1}, \forall \omega\in \Omega^s_h(\cD^E):\,
    \pi(\omega)=\pi^{\cD^E}(\omega)
    },
  \end{align*}
  where $\Omega^s_h(\cD^E)$ denotes the set of trajectories
  $(s_1,a_1,\dotsc,s_{h-1},a_{h-1},s_h)$ visited in some trajectory in $\cD^E$,
  and $\pi^{\cD^E}(\omega)$ denotes the corresponding action present in $\cD^E$,
  and $\pi(\omega)$ denotes the deterministic action taken by $\pi$.
  At (2) we lower bound by considering the error only for the trajectories in
  which at all $h<H$ the action played is $a_1$, and we note that in $\cM'=\cM$,
  the initial state is always $s_1$.
  At (3) we define distribution $\rho\in\Delta^\cS$ over states as
  $\rho(s)=\zeta$ if $s\neq s_S$, and $\rho(s_S)=1-(S-1)\zeta$. Note that since
  any policy $\pi''\in\Pi^{\text{hard}}$ plays action $a_1$ for $h<H$, then we
  set this probability to 1.
  At (4) we rewrite also $\P^{\widehat{\pi}}$ using $\rho$ and chain
  rule of conditional probabilities.
  At (5) we bring the $\rho$ terms outside and define $K\coloneqq
  \prod_{h\in\dsb{H-1}}\widehat{\pi}
  (a_1|s_1,a_1,s^2,a_1,\dotsc,s^h)$ for brevity.
  At (6) we use that $\cP'(\cD^E)$ gives always $\cM'=\cM$, and that it gives,
  when the trajectory is not observed (i.e., it is not in $\cD^E$), with equal
  weight, policies that play any action $b\in\cA$ given trajectory
  $s_1,a_1,s^2,a_1,\dotsc,s^H$. Thus, since there is no dependence on the
  actions assigned by such policies given different trajectories, then the
  expectation simplifies to just $\sum_{b\in\cA}\frac{1}{A}(\dotsc)$. Then, note
  that, using this notation, $\pi'(a|s_1,a_1,s^2,a_1,\dotsc,s^H) = \indic{a=b}$.
  At (7) we use that $K$ is a product of probabilities, thus it lies in $[0,1]$,
  and so, since for $A\ge 2$ we have $1-2/A=(A-2)/A\ge 0$, then the quantity is lower
  bounded by using $K=0$.
  At (8) we bring the expectation inside and use the fact that the expectation
  of the indicator is the probability.
  At (9) we realize that the probability of a trajectory does not depend on the
  specific policy (inside $\Pi^{\text{hard}}$) that generated it. Thus, we can
  replace $q$ using $\rho$, and use the complementary of an event of a Binomial
  distribution.
  At (10) we proceed as in Lemma A.21 of
  \cite{rajaraman2020fundamentalimitationlearning} by lower bounding the
  exponential term by $1/e$.
\end{proof}

\section{Additional Results and Proofs for Section \ref{sec: r known}}
\label{apx: sec r known}

In Appendix \ref{apx: policy class}, we collect additional results and proofs
for Section \ref{sec: policy class}, while in Appendices \ref{apx: no
interaction setting}-\ref{apx: known transition one r}, we collect those for,
respectively, Sections \ref{sec: rsbc}-\ref{sec: rskt}. Lastly, in Appendices
\ref{apx: rE in finite set}-\ref{apx: rE linear}, we extend our results on \rsbc
and \rskt result to the settings in which, respectively, $r^E$ is known to
belong to a given finite set $\cR$ and it is known to be linear in a known
feature map $\phi$. Finally, in Appendix \ref{apx: gen bc arb problems}, we
sketch how to extend \rsbc to arbitrary IL problems.

\subsection{On the Class of Policies in Section \ref{sec: policy class}}
\label{apx: policy class}

In Appendix \ref{apx: proofs lemmas policy class}, we report the proof of Lemmas
\ref{lemma: same return distribution} and \ref{lemma: apx policies}, while in
Appendix \ref{apx: too many policies} we discuss the size of the policy class
$\Pi(r^E)$.

\subsubsection{Proofs}\label{apx: proofs lemmas policy class}

\lemmasameretdistrib*
\begin{proof}
Thanks to Lemma \ref{lemma: Psg equal Psg}, we know that:
    \begin{align*}
    \P^{\pi_{r^E}}\Bigr{G_H=g\wedge s_H=s}=
    \P^{\pi^E}\Bigr{G_H=g\wedge s_H=s}\qquad\forall g\in[0,H-1],\forall s\in\cS,
  \end{align*}
  where $G_H\coloneqq\sum_{h'=1}^{H-1} r_{h'}^E(s_{h'},a_{h'})$ denotes the
  random return at stage $H$.
  Then, for any $g'\in[0,H]$, we can write:
  \begin{align*}
    \eta^{\pi_{r^E}}_{r^E}(g')&=\P^{\pi_{r^E}}(G_H+r^E_H(s_H,a_H)=g')\\
    &\markref{(1)}{=}\sum\limits_{g\in\cG_{r^E,H}}\sum\limits_{s\in\cS}
    \P^{\pi_{r^E}}(G_H=g,s_H=s,r^E_H(s,a_H)=g'-g)\\
    &\markref{(2)}{=}\sum\limits_{g\in\cG_{r^E,H}}\sum\limits_{s\in\cS}
    \com{\sum\limits_{a\in\cA:r^E_H(s,a)=g'-g}}
    \P^{\pi_{r^E}}(G_H=g,s_H=s)\com{\pi_{r^E}(a|s,g)}\\
    &\markref{(3)}{=}\sum\limits_{g\in\cG_{r^E,H}}\sum\limits_{s\in\cS}
    {\sum\limits_{a\in\cA:r^E_H(s,a)=g'-g}}
    \com{\P^{\pi^E}}(G_H=g,s_H=s){\pi_{r^E}(a|s,g)}\\
    &=\sum\limits_{g\in\cG_{r^E,H}}\sum\limits_{s\in\cS}
    \com{\P^{\pi^E}(G_H=g,s_H=s)}{\sum\limits_{a\in\cA:r^E_H(s,a)=g'-g}}{\pi_{r^E}(a|s,g)}\\
    &\markref{(4)}{=}\sum\limits_{g\in\cG_{r^E,H}}\sum\limits_{s\in\cS}
    {\P^{\pi^E}(G_H=g,s_H=s)}{\sum\limits_{a\in\cA:r^E_H(s,a)=g'-g}}
    \com{\frac{\P^{\pi^E}(G_H=g,s_H=s,a_H=a)}{\P^{\pi^E}(G_H=g,s_H=s)}}\\
    &=\sum\limits_{g\in\cG_{r^E,H}}\sum\limits_{s\in\cS}
    {\sum\limits_{a\in\cA}}
    \P^{\pi^E}(G_H=g,s_H=s,a_H=a)\com{\indic{r^E_H(s,a)=g'-g}}\\
    &=\P^{\pi^E}\Bigr{\sum_{h'=1}^{H} r_{h'}^E(s_{h'},a_{h'})=g'}\\
    &=\eta^{\pi^E}_{r^E}(g'),
  \end{align*}
  where at (1) we define symbol $\cG_{r,H}\coloneqq\{g\in[0,H-1]\,| \,\exists
\omega\in\Omega_H:\, G(\omega;r^E)=g\}$, at (2) we recognize that, by
definition, $\pi_{r^E}$ takes actions only depending on the current state, stage
and past rewards, and we denote with brevity this fact with $\pi_{r^E}(a|s,g)$,
at (3) we use the aforementioned result from Lemma \ref{lemma: Psg equal Psg},
at (4) we use the definition of $\pi_{r^E}(a|s,g)$ where the denominator is not
0, noting that, in that case, the entire expression evaluates to 0, 
\end{proof}

\lemmasameretdistribwithlesspolicies*
\begin{proof}
  We can write:
  \begin{align*}
    \cW\Bigr{\eta^{\pi_{r_\theta^E}}_{r^E},\eta^{\pi^E}_{r^E}}&\markref{(1)}{\le}
    \cW\Bigr{\eta^{\pi_{r_\theta^E}}_{r^E},\com{\eta^{\pi_{r_\theta^E}}_{r^E_\theta}}}
    + \cW\Bigr{\com{\eta^{\pi_{r_\theta^E}}_{r^E_\theta}},\com{\eta^{\pi^E}_{r^E_\theta}}}
    +\cW\Bigr{\com{\eta^{\pi^E}_{r^E_\theta}},\eta^{\pi^E}_{r^E}}\\
    &\markref{(2)}{\le}
    2H\|r^E-r^E_\theta\|_\infty
    +
    \cW\Bigr{\com{\eta^{\pi_{r_\theta^E}}_{r^E_\theta}},{\eta^{\pi^E}_{r^E_\theta}}}\\
    &\markref{(3)}{\le}
    H\theta
    +
    \cW\Bigr{\com{\eta^{\pi_{r_\theta^E}}_{r^E_\theta}},{\eta^{\pi^E}_{r^E_\theta}}}\\
    &\markref{(4)}{\le}
    H\theta,
  \end{align*}
  where at (1) we apply twice the triangle's inequality, at (2) we apply twice
  Lemma \ref{lemma: different r same p}, at (3) we realize that, by definition
  of $r_\theta^E$, it holds that $\|r^E-r^E_\theta\|_\infty\le\theta/2$, and
  finally, at (4), we apply Lemma \ref{lemma: same return distribution} with
  reward $r^E_\theta$ and expert's policy $\pi^E$.
\end{proof}

\begin{restatable}{lemma}{differentrsamep}
    \label{lemma: different r same p}
    Let $\cM$ be any \MDPr and let $\pi\in\Pi^{\text{NM}}$ be any policy.
    Let $r^1,r^2:\SAH\to[0,1]$ be any pair of reward functions.
    Then, it holds that:
    \begin{align*}
      \cW(\eta^{\pi}_{r^1},
  \eta^{\pi}_{r^2})\le H\|r^1-r^2\|_\infty.
    \end{align*}
\end{restatable}
\begin{proof}
  We can write:
    \begin{align*}
    \cW(\eta^{\pi}_{r^1},
  \eta^{\pi}_{r^2})
  &\markref{(1)}{=}\sup\limits_{f:\|f\|_{\text{Lip}}\le 1}
    \Bigr{
    \int_{[0,H]}f d\eta^{\pi}_{r^1}
    -\int_{[0,H]}f d\eta^{\pi}_{r^2}}\\
  &\markref{(2)}{=}\sup\limits_{f:\|f\|_{\text{Lip}}\le 1}
    \Bigr{
    \mathop{\com{\sum}}\limits_{g\in\text{supp}(\eta^{\pi}_{r^1})}
    f(g)\eta^{\pi}_{r^1}(g)
    -\mathop{\com{\sum}}\limits_{g\in\text{supp}(\eta^{\pi}_{r^2})}
    f(g)\eta^{\pi}_{r^2}(g)}\\
  &\markref{(3)}{=}\sup\limits_{f:\|f\|_{\text{Lip}}\le 1}
    \Bigr{
    \sum\limits_{g\in\text{supp}(\eta^{\pi}_{r^1})}
    f(g)\com{\sum\limits_{\substack{\omega\in\Omega_{H+1}:\\G(\omega;r^1)=g}}\P^{\pi}(\omega)}\\
    &\qquad\qquad\qquad\qquad\qquad-\sum\limits_{g\in\text{supp}(\eta^{\pi}_{r^2})}
    f(g)\com{\sum\limits_{\substack{\omega\in\Omega_{H+1}:\\G(\omega;r^2)=g}}\P^{\pi}(\omega)}}\\
  &=
  \sup\limits_{f:\|f\|_{\text{Lip}}\le 1}
    \Bigr{
      \com{\sum\limits_{\omega\in\Omega_{H+1}}f(G(\omega;r^1))}
      \P^{\pi}(\omega)\\
    &\qquad\qquad\qquad\qquad\qquad-\com{\sum\limits_{\omega\in\Omega_{H+1}}f(G(\omega;r^2))}
      \P^{\pi}(\omega)}\\
    &=
  \sup\limits_{f:\|f\|_{\text{Lip}}\le 1}
      \sum\limits_{\omega\in\Omega_{H+1}}\P^{\pi}(\omega)\Bigr{f(G(\omega;r^1))
      -f(G(\omega;r^2))}
      \\
    &\le
  \sup\limits_{f:\|f\|_{\text{Lip}}\le 1}
      \sum\limits_{\omega\in\Omega_{H+1}}\P^{\pi}(\omega)\com{\Big|}f(G(\omega;r^1))
      -f(G(\omega;r^2))\com{\Big|}
      \\
    &\markref{(4)}{\le}
      \sum\limits_{\omega\in\Omega_{H+1}}\P^{\pi}(\omega)\com{\Big|G(\omega;r^1)
      -G(\omega;r^2)\Big|}
      \\
    &\markref{(5)}{=}
  \sum\limits_{\omega\in\Omega_{H+1}}\P^{\pi}(\omega)\Big|
      \com{\sum\limits_{h\in\dsb{H}}\Bigr{r^1_{h}(s_{h},a_{h})-r^2_{h}(s_{h},a_{h})}}\Big|
  \\
    &\le
  \sum\limits_{\omega\in\Omega_{H+1}}\P^{\pi}(\omega)
      \sum\limits_{h\in\dsb{H}}\com{\Big|}r^1_{h}(s_{h},a_{h})-r^2_{h}(s_{h},a_{h})\com{\Big|}
  \\
    &\le
  \sum\limits_{\omega\in\Omega_{H+1}}\P^{\pi}(\omega)
      \sum\limits_{h\in\dsb{H}}\com{\max\limits_{(s,a)\in\SA}}\Big|r^1_{h}(s,a)-r^2_{h}(s,a)\Big|
  \\
    &=\sum\limits_{h\in\dsb{H}}\max\limits_{(s,a)\in\SA}\Big|r^1_{h}(s,a)-r^2_{h}(s,a)\Big|
    \cdot 
  \sum\limits_{\omega\in\Omega_{H+1}}\P^{\pi}(\omega)
  \\
  &=\sum\limits_{h\in\dsb{H}}\max\limits_{(s,a)\in\SA}\Big|r^1_{h}(s,a)-r^2_{h}(s,a)\Big|
  \\
  &\le \com{H \max\limits_{(s,a,h)\in\SAH}}\Big|r^1_{h}(s,a)-r^2_{h}(s,a)\Big|,
    \end{align*}
where at (1) we apply the duality formula for the Wasserstein distance (see Eq.
6.3 in Section 6 of \citet{Villani2008OptimalTO}). Note that we interpret the
return distributions $\eta^{\pi}_{r^1}$ and $\eta^{\pi}_{r^2}$ as probability
measures, and that a function $f:\RR\to\RR$ is L-Lipschitz (i.e.,
$\|f\|_{\text{Lip}}=L$) if, for all $x,y\in\RR$, we have $|f(x)-f(y)|\le
L|x-y|$.
At (2) we realize that $\eta^{\pi}_{r^1}$ and $\eta^{\pi}_{r^2}$
have finite supports, since in a tabular MDP with deterministic reward the
number of trajectories is finite, and, thus, also the number of corresponding
returns must be finite. We denote by supp$(\cdot)$ the support of a distribution.
At (3) we use the definition of return distributions, at (4) we use the
definition of Lipschitz functions, at (5) we denote by
$(s_1,a_1,\dotsc,s_H,a_H)$ the state-action trajectory $\omega\in\Omega_{H+1}$,
and we use the definition of operators $G(\cdot;r^1)$ and $G(\cdot;r^2)$.
\end{proof}

\begin{restatable}{lemma}{etanmequaletam}\label{lemma: Psg equal Psg}
Let $\cM_r$ be any MDP and let $\pi\in\Pi^{\text{NM}}$ be any expert's policy.
Let $\pi_r\in\Pi(r)$ be the policy defined as in Eq. \eqref{eq: def policy
imitate same occ meas} for expert's policy $\pi$.
Then, for all $h\in\dsb{H}$, $s\in\cS$ and $g\in[0,h-1]$, it holds that:
  \begin{align*}
    \P^{\pi_r}\Bigr{\sum\limits_{h'=1}^{h-1} r_{h'}(s_{h'},a_{h'})=g\wedge s_h=s}=
    \P^{\pi}\Bigr{\sum\limits_{h'=1}^{h-1} r_{h'}(s_{h'},a_{h'})=g\wedge s_h=s}.
  \end{align*}
\end{restatable}
\begin{proof}
  For simplicity, let us denote by $G_h$ the \emph{random} return at time step
  $h$ under reward $r$. Formally:
  \begin{align*}
    G_h\coloneqq\sum\limits_{h'=1}^{h-1} r_{h'}(s_{h'},a_{h'})\quad\forall h\in\dsb{H}.
  \end{align*}
  In this way, the claim of the theorem can be easily rewritten as:
  \begin{align*}
    \P^{\pi_r}\Bigr{G_h=g\wedge s_h=s}=
    \P^{\pi}\Bigr{G_h=g\wedge s_h=s}.
  \end{align*}
We prove the result by induction.
  Let us begin with the base case: $h=1$. For all $s\in\cS$ and $g\in\{0\}$, we have:
  \begin{align*}
    \P^{\pi_r}\Bigr{G_1=g\wedge s_1=s}=\indic{g=0}\indic{s=s_0}=
    \P^{\pi}\Bigr{G_1=g\wedge s_1=s},
  \end{align*}
  where we noticed that, for $h=1$, no action is taken yet.
  Now, let us consider any stage $h\in\{2,3,\dotsc,H\}$, and let us make the
  induction hypothesis that, for all $h'\in\dsb{h-1}$, for all $s\in\cS$ and
  $g\in[0,h'-1]$, it holds that:
  \begin{align*}
    \P^{\pi_r}\Bigr{G_{h'}=g\wedge s_{h'}=s}=
    \P^{\pi}\Bigr{G_{h'}=g\wedge s_{h'}=s}.
  \end{align*}
  Then, for any $s'\in\cS$ and $g'\in[0,h-1]$, we can write:
  \begin{align*}
    &\P^{\pi_r}(G_h=g'\wedge s_h=s')\\
    &\qquad\qquad\markref{(1)}{=}
    \sum\limits_{\substack{\omega\in\Omega_{h-1},(s,a)\in\SA:\\G(\omega;r)+r_{h-1}(s,a)=g'}}
    \P^{\pi_r}(\omega_{h-1}=\omega\wedge s_{h-1}=s\wedge a_{h-1}=a\wedge s_h=s')\\
    &\qquad\qquad\markref{(2)}{=}\com{\sum\limits_{g\in\cG_{r,h-1}}
    \sum\limits_{\substack{\omega\in\Omega_{h-1}:\\
    G(\omega;r)=g}}}
    \com{\sum\limits_{\substack{(s,a)\in\SA:\\
    r_{h-1}(s,a)=g'-g}}}\P^{\pi_r}(\omega_{h-1}=\omega\wedge s_{h-1}=s\wedge a_{h-1}=a\wedge s_h=s')\\
    &\qquad\qquad\markref{(3)}{=}
    \sum\limits_{g\in\cG_{r,h-1}}
    \sum\limits_{\substack{\omega\in\Omega_{h-1}:\\
    G(\omega;r)=g}}
    \sum\limits_{\substack{(s,a)\in\SA:\\
    r_{h-1}(s,a)=g'-g}}\com{\P^{\pi_r}(\omega_{h-1}=\omega\wedge s_{h-1}=s)}\\
    &\qquad\qquad\qquad\qquad
    \com{\cdot\P^{\pi_r}(a_{h-1}=a|\omega,s)\P^{\pi_r}(s_h=s'|\omega,s,a)}\\
    &\qquad\qquad\markref{(4)}{=} \sum\limits_{g\in\cG_{r,h-1}}
    \sum\limits_{\substack{\omega\in\Omega_{h-1}:\\
    G(\omega;r)=g}}
    \sum\limits_{\substack{(s,a)\in\SA:\\
    r_{h-1}(s,a)=g'-g}}\P^{\pi_r}(\omega_{h-1}=\omega\wedge s_{h-1}=s)\\
    &\qquad\qquad\qquad\qquad
    \cdot\P^{\pi_r}(a_{h-1}=a|\omega,s)\com{p_{h-1}(s'|s,a)}\\
    &\qquad\qquad\markref{(5)}{=} \sum\limits_{g\in\cG_{r,h-1}}
    \sum\limits_{\substack{\omega\in\Omega_{h-1}:\\
    G(\omega;r)=g}}
    \sum\limits_{\substack{(s,a)\in\SA:\\
    r_{h-1}(s,a)=g'-g}}\P^{\pi_r}(\omega_{h-1}=\omega\wedge s_{h-1}=s)\\
    &\qquad\qquad\qquad\qquad
    \cdot\com{\pi_r(a|\omega,s)}p_{h-1}(s'|s,a)\\
    &\qquad\qquad\markref{(6)}{=}\sum\limits_{g\in\cG_{r,h-1}}
    \com{\sum\limits_{\substack{(s,a)\in\SA:\\
    r_{h-1}(s,a)=g'-g}}\sum\limits_{\substack{\omega\in\Omega_{h-1}:\\
    G(\omega;r)=g}}}\P^{\pi_r}(\omega_{h-1}=\omega\wedge s_{h-1}=s)\\
    &\qquad\qquad\qquad\qquad
    \cdot\pi_r(a|\omega,s)p_{h-1}(s'|s,a)\\
    &\qquad\qquad\markref{(7)}{=}\sum\limits_{g\in\cG_{r,h-1}}
    \sum\limits_{\substack{(s,a)\in\SA:\\
    r_{h-1}(s,a)=g'-g}}\com{\pi_r(a|g,s)p_{h-1}(s'|s,a)}\sum\limits_{\substack{\omega\in\Omega_{h-1}:\\
    G(\omega;r)=g}}
    \P^{\pi_r}(\omega_{h-1}=\omega\wedge s_{h-1}=s)\\
    &\qquad\qquad=\sum\limits_{g\in\cG_{r,h-1}}
    \sum\limits_{\substack{(s,a)\in\SA:\\
    r_{h-1}(s,a)=g'-g}}\pi_r(a|g,s)p_{h-1}(s'|s,a)
    \com{\P^{\pi_r}(G_{h-1}=g\wedge s_{h-1}=s)}\\
    &\qquad\qquad\markref{(8)}{=}\sum\limits_{g\in\cG_{r,h-1}}
    \sum\limits_{\substack{(s,a)\in\SA:\\
    r_{h-1}(s,a)=g'-g}}\pi_r(a|g,s)p_{h-1}(s'|s,a)
   \com{\P^{\pi}}(G_{h-1}=g\wedge s_{h-1}=s)\\
    &\qquad\qquad\markref{(9)}{=}\sum\limits_{g\in\cG_{r,h-1}}
    \sum\limits_{\substack{(s,a)\in\SA:\\
    r_{h-1}(s,a)=g'-g}}
    \com{\frac{\P^{\pi}(G_{h-1}=g\wedge s_{h-1}=s\wedge a_{h-1}=a)}{
        \P^{\pi}(G_{h-1}=g\wedge s_{h-1}=s)
    }}\\
    &\qquad\qquad\qquad\qquad
    \cdot p_{h-1}(s'|s,a)\P^{\pi}(G_{h-1}=g\wedge s_{h-1}=s)\\
    &\qquad\qquad=\sum\limits_{g\in\cG_{r,h-1}}
    \sum\limits_{\substack{(s,a)\in\SA:\\
    r_{h-1}(s,a)=g'-g}}
    \P^{\pi}(G_{h-1}=g\wedge s_{h-1}=s\wedge a_{h-1}=a)
     p_{h-1}(s'|s,a)\\
    &\qquad\qquad=\sum\limits_{g\in\cG_{r,h-1}}\sum\limits_{\substack{(s,a)\in\SA:\\
    r_{h-1}(s,a)=g'-g}}
    \com{\P^{\pi}(G_{h-1}=g\wedge s_{h-1}=s\wedge a_{h-1}=a\wedge s_h=s')}\\
    &\qquad\qquad=\sum\limits_{g\in\cG_{r,h-1}}
    \com{\P^{\pi}(G_{h-1}=g\wedge r_{h-1}(s_{h-1},a_{h-1})=g'-g\wedge s_h=s')}\\
    &\qquad\qquad=
    \P^{\pi}(G_h=g'\wedge s_h=s'),
  \end{align*}
  where at (1) we use symbol $\omega_{h''}$ to denote the random trajectory long
  $h''$ stages, i.e., whose realizations belong to $\Omega_{h''}$, for any
  $h''\in\dsb{H}$.
  At (2) we define symbols $\cG_{r,h}\coloneqq\{g\in[0,h-1]\,| \,\exists
\omega\in\Omega_h:\, G(\omega;r)=g\}$ for any $h\in\dsb{H}$, denoting the set of
possible values of cumulative reward obtainable at stage $h$, and we sum over
all such values (note that they are finite in tabular MDPs with deterministic
rewards), at (3) we use the chain rule of conditional probabilities, at (4) we
use the Markovianity of the environment, at (5) we note that
$\P^{\pi_r}(a_{h-1}=a|\omega,s)$ actually is $\pi_r(a|\omega,s)$, at (6) we
exchange the two summations, at (7) we recognize that, by definition,
$\pi_r(a|\omega,s)$ takes on the same value for all the trajectories $\omega$
with the same value of return, and thus we can bring this quantity outside the
summation over the $\omega$. We use symbol $\pi_r(a|g,s)$ to denote this fact
for brevity. We do the same also for $p_{h-1}(s'|s,a)$.
  At (8) we use the induction hypothesis, at (9) we replace $\pi_r(a|g,s)$ with
  its definition when $\P^{\pi}(G_{h-1}=g\wedge s_{h-1}=s)>0$ as in
  the opposite case the entire formula takes on value zero.
\end{proof}

\subsubsection{$\Pi(r^E)$ is too Large for some $r^E$}
\label{apx: too many policies}

Consider any reward $r^E:\SAH\to[0,1]$ that assigns, to every possible
trajectory $\omega\in\Omega$, a different return value
$G(\omega;r^E)\in[0,H-1]$. Observe that rewards $r^E$ of this kind exist as the
number of possible trajectories $|\Omega|=\cO((SA)^{H-1})$ is finite, while the
set of possible return values that we can assign to each $s,a,h$ is infinite
(continuous) $[0,1]$.
An example of a reward of this kind was provided in the proof of Theorem
\ref{thr: lower bound TV}:
\begin{align*}
  r^E_h(s,a)=10^{-hSA-x_s(s)A-x_a(a)}\qquad\forall (s,a,h)\in\SAH,
\end{align*}
where $x_s:\cS\to\Nat,x_a:\cA\to\Nat$ are arbitrary injective functions mapping
the set of states and actions to the set of natural numbers. Intuitively, given
the return $G(\omega;r^E)$ of any trajectory:
\begin{align*}
  \omega=(s_1,a_1,s_2,a_2,\dotsc,s_{h-1},a_{h-1}),
\end{align*}
we can easily reconstruct which $s,a,h$ actually belong to $\omega$ by checking
which decimal numbers in $G(\omega;r^E)$ are ``flagged''.

Given a reward $r^E$ of this kind, we have that $|\cG_{r^E}|$ is in the same
order as $|\Omega|=\cO((SA)^{H-1})$, and so the set of policies $\Pi(r^E)$
defined in Section \ref{sec: policy class} reduces to the whole set of
non-Markovian policies $\Pi^{\text{NM}}$, with all its disadvantages.

\subsection{No Interaction Setting}\label{apx: no interaction setting}

In Appendix \ref{apx: proof thr rsbc}, we prove Theorem \ref{thr: rsbc}, while
in Appendix \ref{apx: more discussion rsbc} we provide
additional discussion on Theorem \ref{thr: rsbc}.

\subsubsection{Proof of Theorem \ref{thr: rsbc}}\label{apx:
proof thr rsbc}

\thrbcsamplecompknownr*
\begin{proof}
We can write:
\begin{align*}
  \cW\Bigr{\eta^{\pi^E}_{r^E},
  \eta^{\widehat{\pi}}_{r^E}}&\markref{(1)}{\le}
  \cW\Bigr{\eta^{\pi^E}_{r^E},
    \com{\eta^{\pi_{r^E_\theta}}_{r^E}}}
  +
  \cW\Bigr{\com{\eta^{\pi_{r^E_\theta}}_{r^E}},
  \com{\eta^{\pi_{r^E_\theta}}_{r^E_\theta}}}
  +
  \cW\Bigr{\com{\eta^{\pi_{r^E_\theta}}_{r^E_\theta}},
\com{\eta^{\widehat{\pi}}_{r^E_\theta}}}
  + \cW\Bigr{\com{\eta^{\widehat{\pi}}_{r^E_\theta}},
  \eta^{\widehat{\pi}}_{r^E}}\\
&\markref{(2)}{\le}
  \com{2H\theta}
  + \cW\Bigr{\eta^{\pi_{r^E_\theta}}_{r^E_\theta},
  \eta^{\widehat{\pi}}_{r^E_\theta}}\\
  &\markref{(3)}{\le}
  2H\theta
  +\com{H\Big\|}\eta^{\pi_{r^E_\theta}}_{r^E_\theta}-
  \eta^{\widehat{\pi}}_{r^E_\theta}\com{\Big\|_1}\\
  &\markref{(4)}{\le}
  2H\theta+H\com{\sum\limits_{h\in\dsb{H}}\sum\limits_{g\in \cG_{r^E_\theta,h}}\sum\limits_{s\in\cS}
  \P^{\pi_{r^E_\theta}}(G_h=g\wedge s_{h}=s)
  \Bign{\pi_{r^E_\theta,h}(\cdot|s,g)-\widehat{\pi}_{h}(\cdot|s,g)}_1}\\
  &\markref{(5)}{=}
  2H\theta+H\sum\limits_{h\in\dsb{H}}\sum\limits_{g\in \cG_{r^E_\theta,h}}\sum\limits_{s\in\cS}
\com{\P^{\pi^E}}(G_h=g\wedge s_{h}=s)
  \Bign{\pi_{r^E_\theta,h}(\cdot|s,g)-\widehat{\pi}_{h}(\cdot|s,g)}_1\\
  &\markref{(6)}{\le}
  2H\theta+H\epsilon',
\end{align*}
where at (1) we apply triangle's inequality, at (2) we apply Lemma \ref{lemma:
apx policies} and Lemma \ref{lemma: different r same p} twice, at (3) we use
Particular Case 6.13 of \citet{Villani2008OptimalTO}, which tells us that we can
upper bound the Wasserstein distance between two distributions supported on set
$\cX$ by the diameter of $\cX$ ($\max_{x,x'\in\cX}|x-x'|$) times the one norm
between the two distributions. Since $\cX=[0,H]$ in our case, we get the
expression written above.
At (4) we apply Lemma \ref{lemma: error propagation} with the notation defined
in that lemma, observing that policies $\pi_{r^E_\theta}$ and $\widehat{\pi}$
satisfy the hypothesis for reward $r^E_\theta$.
At (5) we use Lemma \ref{lemma: Psg equal Psg} observing that the random return
$G_h$ is in terms of reward $r^E_\theta$, and recalling the definition of
$\pi_{r^E_\theta}$.
Lastly, at (6) we apply Lemma \ref{lemma: concentration} with accuracy $\epsilon'$.

The result follows by imposing that $\epsilon'\le\frac{\epsilon}{2H}$ and
$2H\theta\le\frac{\epsilon}{2}$, which can be achieved by taking
$\epsilon'=\frac{\epsilon}{2H}$ and $\theta=\frac{\epsilon}{4H}$, and by
observing that:
\begin{align*}
  \overline{\cG}&\coloneqq \sum_{h\in\dsb{H}}|\cG_{r^E_\theta,h}|\\
  &\le \sum_{h\in\dsb{H}}|\cY^\theta_h|\\
  &= \sum_{h\in\dsb{H}}|\{0,\theta,2\theta,\dotsc,
  \floor{(h-1)/\theta}\theta\}|\\
  &\le \sum_{h\in\dsb{H}}(1+(h-1)/\theta)\\
  &= H(1-1/\theta)+\frac{1}{\theta}\sum_{h\in\dsb{H}}h\\
  &\markref{(10)}{=} H(1-1/\theta)+\frac{H(H+1)}{2\theta}\\
  &\le \frac{H^2}{\theta}\\
  &\markref{(11)}{\le} \frac{4H^3}{\epsilon},
\end{align*}
where at (10) we used the formula for arithmetic sums, and at (11) we used the
previous choice $\theta=\frac{\epsilon}{4H}$.

Replacing into the number of samples in Lemma \ref{lemma: concentration} (and
also $\epsilon'=\frac{\epsilon}{2H}$) we get the result:
\begin{align*}
  N\ge \frac{3072 SH^6\ln\frac{8SH^3}{\delta\epsilon}}
{\epsilon^3}\biggr{
  \ln\frac{8SH^3}{\delta\epsilon}+(A-1)
  \ln\Bigr{\frac{2048eSH^6\ln\frac{8SH^3}{\delta\epsilon}}
{\epsilon^3}}
}.
\end{align*}
By using $\widetilde{\cO}$ notation to hide logarithmic terms in
$S,A,H,\frac{1}{\epsilon},\ln\frac{1}{\delta}$, we get the result.
\end{proof}

\begin{restatable}[Error Propagation]{lemma}{errorpropagation}
\label{lemma: error propagation}
Let $\cM_r$ be any MDP.
For any pair of policies $\pi,\pi'\in\Pi^{\text{NM}}$ such that, for all
$h\in\dsb{H}$, $a\in\cA$, $s\in\cS$ and $\omega,\omega'\in\Omega_h$ with
$G(\omega;r)=G(\omega';r)$:
\begin{align*}
  \pi(a|\omega,s)=\pi(a|\omega',s)\qquad\wedge\qquad \pi'(a|\omega,s)=\pi'(a|\omega',s),
\end{align*}
it holds that:
\begin{align*}
  \Big\|\eta^{\pi}_{r}-
  \eta^{\pi'}_{r}\Big\|_1
  &\le \sum\limits_{h\in\dsb{H}}\sum\limits_{g\in \cG_{r,h}}\sum\limits_{s\in\cS}
  \P^{\pi}(G_h=g\wedge s_{h}=s)
  \Bign{\pi_{h}(\cdot|s,g)-\pi_{h}'(\cdot|s,g)}_1,
\end{align*}
where $\cG_{r,h}\coloneqq\{g\in[0,h-1]\,| \,\exists \omega\in\Omega_h:\,
G(\omega;r)=g\}$ for all $h\in\dsb{H+1}$,
$G_h\coloneqq\sum_{h'=1}^{h-1}r_{h'}(s_{h'},a_{h'})$ denotes the \emph{random}
return at stage $h$, and $\pi_{h}(\cdot|s,g)$ and $\pi_{h}'(\cdot|s,g)$ denote
the unique probability with which the policies $\pi$ and $\pi'$ prescribe
actions in $s$ at $h$ under any trajectory $\omega\in\Omega_h$ with
$G(\omega;r)=g$. 
\end{restatable}
\begin{proof}
  To prove the result, we first demonstrate by induction that, for
  all $h\in\dsb{2,H}$, it holds that:
  \begin{align*}
  &\sum\limits_{\substack{g\in \cG_{r,h}}}\sum\limits_{s\in\cS}\Big|
  \P^{\pi}(G_h=g\wedge s_{h}=s)-
  \P^{\pi'}(G_h=g\wedge s_{h}=s)
  \Big|\\
  &\qquad\qquad
  \le \sum\limits_{h'\in\dsb{h-1}}\sum\limits_{\substack{g\in \cG_{r,h'}}}\sum\limits_{s\in\cS}
  \P^{\pi}(G_{h'}=g\wedge s_{h'}=s)
  \Bign{\pi_{h'}(\cdot|s,g)-\pi_{h'}'(\cdot|s,g)}_1.
\end{align*}
We remark that, by hypothesis, both $\pi$ and $\pi'$ prescribe the same actions
when faced with trajectories $\omega,\omega'$ with the same return
$G(\omega;r)=G(\omega';r)$ that are followed by the same state $s$.

We begin with the base case $h=2$. We can write:
\begin{align*}
  &\sum\limits_{\substack{g\in \cG_{r,2}}}\sum\limits_{s'\in\cS}\Big|
  \P^{\pi}(G_2=g\wedge s_{2}=s')-
  \P^{\pi'}(G_2=g\wedge s_{2}=s')
  \Big|\\
  &\qquad\qquad=
  \sum\limits_{\substack{g\in \cG_{r,2}}}\sum\limits_{s'\in\cS}\Big|
  \P^{\pi}(\com{r_1(s_1,a_1)}=g\wedge s_{2}=s')-
  \P^{\pi'}(\com{r_1(s_1,a_1)}=g\wedge s_{2}=s')
  \Big|\\
  &\qquad\qquad=
  \sum\limits_{\substack{g\in \cG_{r,2}}}\sum\limits_{s'\in\cS}\Big|
  \mathop{\com{\sum}}\limits_{\substack{\com{(s,a)\in\SA:}\\\com{r_1(s,a)=g}}}\Big(
  \com{\P^{\pi}(s_{2}=s'|s,a)\P^{\pi}(a_1=a|s)\P^{\pi}(s_1=s)}\\
  &\qquad\qquad\qquad\qquad-
  \com{\P^{\pi'}(s_{2}=s'|s,a)\P^{\pi'}(a_1=a|s)\P^{\pi'}(s_1=s)}\Big)
  \Big|\\
  &\qquad\qquad\markref{(1)}{=}
  \sum\limits_{\substack{g\in \cG_{r,2}}}\sum\limits_{s'\in\cS}\Big|
  \sum\limits_{\substack{(s,a)\in\SA:\\r_1(s,a)=g}}\Big(
  \com{p(s'|s,a)\indic{s=s_0}\pi(a|s)}\\
  &\qquad\qquad\qquad\qquad-
  \com{p(s'|s,a)\indic{s=s_0}\pi'(a|s)}\Big)
  \Big|\\
  &\qquad\qquad\markref{(2)}{\le}
  \sum\limits_{\substack{g\in \cG_{r,2}}}\mathop{\com{\sum}}\limits_{
    \substack{\com{a\in\cA:}\\\com{r_1(s_0,a)=g}}}
  \sum\limits_{s'\in\cS}p(s'|\com{s_0},a)\com{\Big|}
  \pi(a|\com{s_0})-\pi'(a|\com{s_0})
  \com{\Big|}\\
  &\qquad\qquad=
  \sum\limits_{\substack{g\in \cG_{r,2}}}\sum\limits_{\substack{a\in\cA:\\r_1(s_0,a)=g}}
  \Big|\pi(a|s_0)-\pi'(a|s_0)\Big|\\
  &\qquad\qquad=
  \com{\sum\limits_{a\in\cA}}
  \Big|\pi(a|s_0)-\pi'(a|s_0)\Big|\\
  &\qquad\qquad=
  \Big\|\pi(\cdot|s_0)-\pi'(\cdot|s_0)\Big\|_1\\
  &\qquad\qquad\markref{(3)}{=}
  \sum\limits_{h'\in\dsb{1}}\sum\limits_{\substack{g\in \cG_{r,h'}}}\sum\limits_{s\in\cS}
  \P^{\pi}(G_{h'}=g\wedge s_{h'}=s)
  \Bign{\pi_{h'}(\cdot|s,g)-\pi_{h'}'(\cdot|s,g)}_1,
\end{align*}
where at (1) we realize that in $\cM_r$ the initial state is always $s_0$, and
that the transition model is Markovian and independent of the policy, at (2) we
apply triangle's inequality and keep only $s_0$ because of the indicator, and at
(3) we have simply rewritten the expression in a more convenient way for proving
the result (note that $\dsb{1}=\{1\}$ and $\cG_{r',1}=\{0\}$ and
$\P^{\pi}(G_{h'}=g\wedge s_{1}=s)=\indic{g=0\wedge s=s_0}$).

Now, let us consider any stage $h\in\dsb{3,H}$. Let us make the inductive
hypothesis that:
  \begin{align*}
  &\sum\limits_{\substack{g\in \cG_{r,h-1}}}\sum\limits_{s\in\cS}\Big|
  \P^{\pi}(G_{h-1}=g\wedge s_{h-1}=s)-
  \P^{\pi'}(G_{h-1}=g\wedge s_{h-1}=s)
  \Big|\\
  &\qquad\qquad
  \le \sum\limits_{h'\in\dsb{h-2}}\sum\limits_{\substack{g\in \cG_{r,h'}}}\sum\limits_{s\in\cS}
  \P^{\pi}(G_{h'}=g\wedge s_{h'}=s)
  \cdot\Bign{\pi_{h'}(\cdot|s,g)-\pi_{h'}'(\cdot|s,g)}_1.
\end{align*}
Then, we can write (we use symbol $\omega_h$ to denote the random trajectory
$(s_1,a_1,\dotsc,s_h,a_h)$ up to stage $h$):
\begin{align*}
  &\sum\limits_{\substack{g'\in \cG_{r,h}}}\sum\limits_{s'\in\cS}\Big|
  \P^{\pi}(G_h=g'\wedge s_{h}=s')-
  \P^{\pi'}(G_h=g'\wedge s_{h}=s')
  \Big|\\
  &\qquad\qquad=
  \sum\limits_{g'\in \cG_{r,h}}\sum\limits_{s'\in\cS}
  \Big|
    \com{\sum\limits_{g\in \cG_{r,h-1}}}
    \Big(\\
  &\qquad\qquad\qquad\qquad
      \P^{\pi}(\com{G_{h-1}=g\wedge r_{h-1}(s_{h-1},a_{h-1})=g'-g} \wedge s_{h}=s')\\
  &\qquad\qquad\qquad\qquad
      -\P^{\pi'}(\com{G_{h-1}=g\wedge r_{h-1}(s_{h-1},a_{h-1})=g'-g} \wedge s_{h}=s')
    \Big)\Big|\\
  &\qquad\qquad=
  \sum\limits_{g'\in \cG_{r,h}}\sum\limits_{s'\in\cS}
  \Big|
    \sum\limits_{g\in \cG_{r,h-1}}
    \mathop{\com{\sum}}\limits_{\substack{\com{\omega\in\Omega_{h-1}:}\\\com{G(\omega;r)=g}}}
    \mathop{\com{\sum}}\limits_{\substack{\com{(s,a)\in\SA:}\\\com{r_{h-1}(s,a)=g'-g}}}
    \Big(\\
  &\qquad\qquad\qquad\qquad
      \P^{\pi}(\com{\omega_{h-2}=\omega\wedge s_{h-1}=s\wedge a_{h-1}=a}\wedge s_{h}=s')\\
  &\qquad\qquad\qquad\qquad
      -\P^{\pi'}(\com{\omega_{h-2}=\omega\wedge s_{h-1}=s\wedge a_{h-1}=a}\wedge s_{h}=s')
    \Big)\Big|\\
  &\qquad\qquad\markref{(4)}{=}
  \sum\limits_{g'\in \cG_{r,h}}\sum\limits_{s'\in\cS}
  \Big|
    \sum\limits_{g\in \cG_{r,h-1}}
    \sum\limits_{\substack{\omega\in\Omega_{h-1}:\\G(\omega;r)=g}}
    \sum\limits_{\substack{(s,a)\in\SA:\\r_{h-1}(s,a)=g'-g}}
    \Big(\\
  &\qquad\qquad\qquad\qquad
      \P^{\pi}(\com{\omega_{h-2}=\omega\wedge s_{h-1}=s})
      \P^{\pi}(a_{h-1}=a\wedge s_{h}=s'|\com{\omega,s})\\
  &\qquad\qquad\qquad\qquad
      -\P^{\pi'}(\com{\omega_{h-2}=\omega\wedge s_{h-1}=s})
      \P^{\pi'}(a_{h-1}=a\wedge s_{h}=s'|\com{\omega,s})
    \Big)\Big|\\
  &\qquad\qquad\markref{(5)}{=}
  \sum\limits_{g'\in \cG_{r,h}}\sum\limits_{s'\in\cS}
  \Big|
    \sum\limits_{g\in \cG_{r,h-1}}
    \sum\limits_{\substack{\omega\in\Omega_{h-1}:\\G(\omega;r)=g}}
    \sum\limits_{\substack{(s,a)\in\SA:\\r_{h-1}(s,a)=g'-g}}
    \Big(\\
  &\qquad\qquad\qquad\qquad
      \P^{\pi}(\omega_{h-2}=\omega\wedge s_{h-1}=s)
      \com{\pi(a|\omega,s)p_{h-1}(s'|s,a)}\\
  &\qquad\qquad\qquad\qquad
      -\P^{\pi'}(\omega_{h-2}=\omega\wedge s_{h-1}=s)
      \com{\pi'(a|\omega,s)p_{h-1}(s'|s,a)}
    \Big)\Big|\\
  &\qquad\qquad=
  \sum\limits_{g'\in \cG_{r,h}}\sum\limits_{s'\in\cS}
  \Big|
    \sum\limits_{g\in \cG_{r,h-1}}
    \sum\limits_{\substack{\omega\in\Omega_{h-1}:\\G(\omega;r)=g}}
    \sum\limits_{\substack{(s,a)\in\SA:\\r_{h-1}(s,a)=g'-g}}
    \com{p_{h-1}(s'|s,a)}\Big(\\
  &\qquad\qquad\qquad\qquad
      \P^{\pi}(\omega_{h-2}=\omega\wedge s_{h-1}=s)
      \pi(a|\omega,s)\\
  &\qquad\qquad\qquad\qquad
      -\P^{\pi'}(\omega_{h-2}=\omega\wedge s_{h-1}=s)
      \pi'(a|\omega,s)\\
  &\qquad\qquad\qquad\qquad
  \com{\pm\P^{\pi}(\omega_{h-2}=\omega\wedge s_{h-1}=s)
      \pi'(a|\omega,s)}
    \Big)\Big|\\
  &\qquad\qquad\markref{(6)}{\le}
  \sum\limits_{g'\in \cG_{r,h}}\sum\limits_{s'\in\cS}
    \sum\limits_{g\in \cG_{r,h-1}}
    \sum\limits_{\substack{\omega\in\Omega_{h-1}:\\G(\omega;r)=g}}
    \sum\limits_{\substack{(s,a)\in\SA:\\r_{h-1}(s,a)=g'-g}}
    p_{h-1}(s'|s,a)\\
  &\qquad\qquad\qquad\qquad
  \cdot\P^{\pi}(\omega_{h-2}=\omega\wedge s_{h-1}=s)\com{\Big|
      \pi(a|\omega,s)-\pi'(a|\omega,s)\Big|}\\
  &\qquad\qquad\qquad\qquad
  +\sum\limits_{g'\in \cG_{r,h}}\sum\limits_{s'\in\cS}
    \sum\limits_{g\in \cG_{r,h-1}}\com{\Big|}
    \sum\limits_{\substack{\omega\in\Omega_{h-1}:\\G(\omega;r)=g}}
    \sum\limits_{\substack{(s,a)\in\SA:\\r_{h-1}(s,a)=g'-g}}
    p_{h-1}(s'|s,a)\pi'(a|\omega,s)
      \\
  &\qquad\qquad\qquad\qquad
  \cdot\Bigr{\com{\P^{\pi}(\omega_{h-2}=\omega\wedge s_{h-1}=s)-
  \P^{\pi'}(\omega_{h-2}=\omega\wedge s_{h-1}=s)}}\com{\Big|}\\
  &\qquad\qquad\markref{(7)}{\le}
  \sum\limits_{g'\in \cG_{r,h}}
    \sum\limits_{g\in \cG_{r,h-1}}
    \sum\limits_{\substack{\omega\in\Omega_{h-1}:\\G(\omega;r)=g}}
    \sum\limits_{\substack{(s,a)\in\SA:\\r_{h-1}(s,a)=g'-g}}
    \\
  &\qquad\qquad\qquad\qquad
  \P^{\pi}(\omega_{h-2}=\omega\wedge s_{h-1}=s)\Big|
      \pi(a|\omega,s)-\pi'(a|\omega,s)\Big|\\
  &\qquad\qquad\qquad\qquad
  +\sum\limits_{g'\in \cG_{r,h}}\sum\limits_{s'\in\cS}
    \sum\limits_{g\in \cG_{r,h-1}}
    \mathop{\com{\sum}}\limits_{\substack{\com{(s,a)\in\SA:}\\\com{r_{h-1}(s,a)=g'-g}}}
    \com{p_{h-1}(s'|s,a)}
    \com{\Big|}
    \sum\limits_{\substack{\omega\in\Omega_{h-1}:\\G(\omega;r)=g}}
    \pi'(a|\omega,s)
      \\
  &\qquad\qquad\qquad\qquad
  \cdot\Bigr{\P^{\pi}(\omega_{h-2}=\omega\wedge s_{h-1}=s)-
  \P^{\pi'}(\omega_{h-2}=\omega\wedge s_{h-1}=s)}\com{\Big|}\\
  &\qquad\qquad\markref{(8)}{=}
    \sum\limits_{g\in \cG_{r,h-1}}
    \com{\sum\limits_{\substack{(s,a)\in\SA}}\Big|
      \pi_{h-1}(a|s,g)-\pi_{h-1}'(a|s,g)\Big|}
    \\
  &\qquad\qquad\qquad\qquad
  \cdot\sum\limits_{\substack{\omega\in\Omega_{h-1}:\\G(\omega;r)=g}}
  \P^{\pi}(\omega_{h-2}=\omega\wedge s_{h-1}=s)\\
  &\qquad\qquad\qquad\qquad
  +\sum\limits_{g'\in \cG_{r,h}}
    \sum\limits_{g\in \cG_{r,h-1}}
    \sum\limits_{\substack{(s,a)\in\SA:\\r_{h-1}(s,a)=g'-g}}
    \com{\pi_{h-1}'(a|s,g)}\Big|
    \sum\limits_{\substack{\omega\in\Omega_{h-1}:\\G(\omega;r)=g}}
      \\
  &\qquad\qquad\qquad\qquad
  \cdot\Bigr{\P^{\pi}(\omega_{h-2}=\omega\wedge s_{h-1}=s)-
  \P^{\pi'}(\omega_{h-2}=\omega\wedge s_{h-1}=s)}\Big|\\
  &\qquad\qquad=
    \sum\limits_{g\in \cG_{r,h-1}}
    \sum\limits_{\substack{(s,a)\in\SA}}\Big|
      \pi_{h-1}(a|s,g)-\pi_{h-1}'(a|s,g)\Big|
  \com{\P^{\pi}(G_{h-1}=g\wedge s_{h-1}=s)}\\
  &\qquad\qquad\qquad\qquad
  +\sum\limits_{g'\in \cG_{r,h}}
    \sum\limits_{g\in \cG_{r,h-1}}
    \sum\limits_{\substack{(s,a)\in\SA:\\r_{h-1}(s,a)=g'-g}}
    \pi_{h-1}'(a|s,g)
      \\
  &\qquad\qquad\qquad\qquad
  \Big|\com{\P^{\pi}(G_{h-1}=g\wedge s_{h-1}=s)}-
  \com{\P^{\pi'}(G_{h-1}=g\wedge s_{h-1}=s)}\Big|\\
  &\qquad\qquad\markref{(9)}{=}
    \sum\limits_{g\in \cG_{r,h-1}}\sum\limits_{\com{s\in\cS}}
    \P^{\pi}(G_{h-1}=g\wedge s_{h-1}=s)
    \com{\Big\|
      \pi_{h-1}(\cdot|s,g)-\pi_{h-1}'(\cdot|s,g)\Big\|_1} \\
  &\qquad\qquad\qquad\qquad
  +\sum\limits_{g\in \cG_{r,h-1}}\sum\limits_{\com{s\in\cS}}
  \Big|\P^{\pi}(G_{h-1}=g\wedge s_{h-1}=s)-
  \P^{\pi'}(G_{h-1}=g\wedge s_{h-1}=s)\Big|\\
  &\qquad\qquad\markref{(10)}{\le}
  \sum\limits_{g\in \cG_{r,h-1}}\sum\limits_{s\in\cS}
    \P^{\pi}(G_{h-1}=g\wedge s_{h-1}=s)\Big\|
      \pi_{h-1}(\cdot|s,g)-\pi_{h-1}'(\cdot|s,g)\Big\|_1
    \\
  &\qquad\qquad\qquad\qquad  
  +\com{\sum\limits_{h'\in\dsb{h-2}}\sum\limits_{\substack{g\in \cG_{r,h'}}}\sum\limits_{s\in\cS}
  \P^{\pi}(G_{h'}=g\wedge s_{h'}=s)}\\
  &\qquad\qquad\qquad\qquad  
  \com{\cdot\Bign{\pi_{h'}(\cdot|s,g)-\pi_{h'}'(\cdot|s,g)}_1}\\
  &\qquad\qquad=
  \sum\limits_{h'\in\dsb{\com{h-1}}}\sum\limits_{\substack{g\in \cG_{r,h'}}}\sum\limits_{s\in\cS}
  \P^{\pi}(G_{h'}=g\wedge s_{h'}=s)
  \cdot\Bign{\pi_{h'}(\cdot|s,g)-\pi_{h'}'(\cdot|s,g)}_1,
\end{align*}
where at (4) we use the chain rule of conditional probabilities, at (5) we do it
again, and we recognize the policies $\pi$ and $\pi'$, and also that the
transition model is Markovian, at (6) we use triangle's inequality to split the
summations and bring the absolute value inside, at (7), in the first term, we
note that $p_{h-1}(s'|s,a)$ is the only term that depends on $s'$ and that it
sums to 1, while in the second term we exchange the order of two summations and
apply triangle's inequality to bring one inside, at (8), in the first term, we
first remove the summation on $g'$ along with the indicator function that
forces us to consider a subset of state-action pairs, and then we exchange two
other summations and note that the policies do not depend by hypothesis on the
entire past trajectory, but just on the return so far. Instead, in the second
term, we use that $\sum_{s'\in\cS}p_{h-1}(s'|s,a)=1$, and also that, by
hypothesis, $\pi'$ does not depend on the entire past trajectory, but just on
$g$. At (9), i.a., we use that $\sum_{g'\in
\cG_{r,h}}\indic{r_{h-1}(s,a)=g'-g}=1$ and that
$\sum_{a\in\cA}\pi_{h-1}'(a|s,g)=1$. Finally, at (10), we apply the inductive
hypothesis.

Thanks to this result, we can finally prove the claim in the lemma, using
passages analogous to those above, with the difference that we do not have the
summation over the states at the current stage (i.e., $H+1$):
\begin{align*}
  &\Big\|\eta^{\pi}_{r}-
  \eta^{\pi'}_{r}\Big\|_1
  =\sum\limits_{g\in \cG_{r,H+1}}
  \Big|\eta^{\pi}_{r}(g)-\eta^{\pi'}_{r}(g)\Big|\\
  &\qquad\qquad=\sum\limits_{g\in \cG_{r,H+1}}
  \Big|\com{\P^{\pi}(G_{H+1}=g)}-\com{\P^{\pi'}(G_{H+1}=g)}\Big|\\
  &\qquad\qquad=\sum\limits_{g\in \cG_{r,H+1}}
  \Big|\com{\sum\limits_{s'\in\cS}\sum\limits_{a'\in\cA}\sum\limits_{g'\in\cG_{r,H}}
  \sum\limits_{\omega\in\Omega_H}
  \indic{r_H(s',a')=g-g', G(\omega;r)=g'}}\Big(\\
  &\qquad\qquad\qquad\qquad
  \cdot\P^{\pi}(\com{\omega_{H-1}=\omega\wedge s_{H}=s'\wedge a_H=a'})\\
  &\qquad\qquad\qquad\qquad-\P^{\pi'}(\com{\omega_{H-1}=\omega\wedge s_{H}=s'\wedge a_H=a'}\Big)\Big|\\
  &\qquad\qquad=\sum\limits_{g\in \cG_{r,H+1}}
  \Big|\sum\limits_{s'\in\cS}\sum\limits_{a'\in\cA}\sum\limits_{g'\in\cG_{r,H}}
  \sum\limits_{\omega\in\Omega_H}
  \indic{r_H(s',a')=g-g', G(\omega;r)=g'}\Big(\\
  &\qquad\qquad\qquad\qquad
  \cdot\P^{\pi}(\omega_{H-1}=\omega\wedge s_{H}=s') \com{\pi(a'|\omega,s')}\\
  &\qquad\qquad\qquad\qquad-\P^{\pi'}(\omega_{H-1}=\omega\wedge s_{H}=s')\com{\pi'(a'|\omega,s')}\\
  &\qquad\qquad\qquad\qquad
  \com{\pm \P^{\pi}(\omega_{H-1}=\omega\wedge s_{H}=s')\pi'(a'|\omega,s')}
  \Big)\Big|\\
  &\qquad\qquad\markref{(11)}{\le}\sum\limits_{g\in \cG_{r,H}}\sum\limits_{s\in\cS}
    \P^{\pi}(G_{H}=g\wedge s_{H}=s)\Big\|
      \pi_{H}(\cdot|s,g)-\pi_{H}'(\cdot|s,g)\Big\|_1
    \\
  &\qquad\qquad\qquad\qquad
  +\sum\limits_{g\in \cG_{r,H}}\sum\limits_{{s\in\cS}}
  \Big|\P^{\pi}(G_{H}=g\wedge s_{H}=s)-
  \P^{\pi'}(G_{H}=g\wedge s_{H}=s)\Big|\\
  &\qquad\qquad\markref{(12)}{\le}
  \sum\limits_{h'\in\dsb{H}}\sum\limits_{\substack{g\in \cG_{r,h'}}}\sum\limits_{s\in\cS}
  \P^{\pi}(G_{h'}=g\wedge s_{h'}=s)
  \Bign{\pi_{h'}(\cdot|s,g)-\pi_{h'}'(\cdot|s,g)}_1,
\end{align*}
where at (11) we made the same passages as above, all in one, and at (12) we
applied the result proved earlier by induction.

This concludes the proof. Note that, from a high-level perspective, our proof
approach resembles the ``reduction to supervised learning'' approach made in
\citet{rajaraman2020fundamentalimitationlearning} for stochastic policies, with
the difference that we work in an augmented state-space MDP.
\end{proof}

\begin{restatable}[Concentration]{lemma}{concentration}
\label{lemma: concentration}
Let $\epsilon\in(0,H]$ and $\delta\in(0,1)$.
Let $\cM_{r^E}$ be any MDP, $\pi^E\in\Pi^{\text{NM}}$ be any expert's policy,
and $\widehat{\pi}$ be the output of Algorithm \ref{alg: rsbc}.
Then, with probability $1-\delta$, we have that (we use the notation in Lemma
\ref{lemma: error propagation}):
\begin{align*}
  \sum\limits_{h\in\dsb{H}}\sum\limits_{g\in \cG_{r^E_\theta,h}}\sum\limits_{s\in\cS}
\P^{\pi^E}(G_h=g\wedge s_{h}=s)
  \Bign{\pi_{r^E_\theta,h}(\cdot|s,g)-\widehat{\pi}_{h}(\cdot|s,g)}_1\le\epsilon,
\end{align*}
with a number of samples:
\begin{align*}
  N\le \frac{193SH\overline{\cG}\ln\frac{2S\overline{\cG}}{\delta}}
{\epsilon^2}\biggr{
  \ln\frac{2S\overline{\cG}}{\delta}+(A-1)
  \ln\Bigr{\frac{128eSH\overline{\cG}\ln\frac{2S\overline{\cG}}{\delta}}
{\epsilon^2}}
},
\end{align*}
where $\overline{\cG}\coloneqq \sum_{h\in\dsb{H}}|\cG_{r^E_\theta,h}|$.
\end{restatable}
\begin{proof}
  We can write:
  \begin{align*}
    &  \sum\limits_{h\in\dsb{H}}\sum\limits_{g\in \cG_{r^E_\theta,h}}\sum\limits_{s\in\cS}
\P^{\pi^E}(G_h=g\wedge s_{h}=s)
  \Bign{\pi_{r^E_\theta,h}(\cdot|s,g)-\widehat{\pi}_{h}(\cdot|s,g)}_1\\
  &\qquad\qquad=
  \sum\limits_{h\in\dsb{H}}\sum\limits_{g\in \cG_{r^E_\theta,h}}\sum\limits_{s\in\cS}
  \P^{\pi^E}(G_h=g\wedge s_h=s)
  \sum\limits_{a\in\cA}\Big|\pi_{r^E_\theta,h}(a|g,s)\\
  &\qquad\qquad\qquad
  -\Bigr{
    \com{\frac{M_{h}(s,g,a)}{\sum_{a'}M_{h}(s,g,a')}\bigindic{\sum_{a'}M_{h}(s,g,a')>0}
    +\frac{1}{A}\bigindic{\sum_{a'}M_{h}(s,g,a')=0}}
  }\Big|\\
  &\qquad\qquad\markref{(1)}{\le}
  \sum\limits_{h\in\dsb{H}}\sum\limits_{g\in \cG_{r^E_\theta,h}}\sum\limits_{s\in\cS}
  \P^{\pi^E}(G_h=g\wedge s_h=s)
  \\
  &\qquad\qquad\qquad\qquad
  \cdot \com{2\sqrt{2}\sqrt{\frac{\ln\frac{2S\overline{\cG}}{\delta}+
  (A-1)\ln\bigr{e\bigr{1+\frac{\sum_{a'}M_{h}(s,g,a')}{A-1}}}}{\sum_{a'}M_{h}(s,g,a')}}}
  \\
  &\qquad\qquad\markref{(2)}{\le}
  2\sqrt{2}\sqrt{\ln\frac{2S\overline{\cG}}{\delta}+
  (A-1)\ln\Bigr{e\Bigr{1+\frac{\com{N}}{A-1}}}}\\
  &\qquad\qquad\qquad\qquad
  \cdot \sum\limits_{h\in\dsb{H}}\sum\limits_{g\in \cG_{r^E_\theta,h}}\sum\limits_{s\in\cS}
  \P^{\pi^E}(G_h=g\wedge s_h=s)
  \sqrt{\frac{1}{\sum_{a'}M_{h}(s,g,a')}}
  \\
  &\qquad\qquad\markref{(3)}{\le}
  2\sqrt{2}\sqrt{\ln\frac{2S\overline{\cG}}{\delta}+
  (A-1)\ln\Bigr{e\Bigr{1+\frac{N}{A-1}}}}\\
  &\qquad\qquad\qquad\qquad
  \cdot \sum\limits_{h\in\dsb{H}}\sum\limits_{g\in \cG_{r^E_\theta,h}}\sum\limits_{s\in\cS}
  \sqrt{\com{\frac{8\ln\frac{2S\overline{\cG}}{\delta}\P^{\pi^E}(G_h=g\wedge s_h=s)}{N}}}
  \\
  &\qquad\qquad=
  \com{8\sqrt{\frac{\ln\frac{2S\overline{\cG}}{\delta}}{N}}}
  \sqrt{\ln\frac{2S\overline{\cG}}{\delta}+
  (A-1)\ln\Bigr{e\Bigr{1+\frac{N}{A-1}}}}\\
  &\qquad\qquad\qquad\qquad
  \cdot \sum\limits_{h\in\dsb{H}}\sum\limits_{g\in \cG_{r^E_\theta,h}}\sum\limits_{s\in\cS}
  \sqrt{\com{\P^{\pi^E}(G_h=g\wedge s_h=s)}}
  \\
  &\qquad\qquad\markref{(4)}{\le}
  8\sqrt{\frac{\ln\frac{2S\overline{\cG}}{\delta}}{N}}
  \sqrt{\ln\frac{2S\overline{\cG}}{\delta}+
  (A-1)\ln\Bigr{e\Bigr{1+\frac{N}{A-1}}}}\\
  &\qquad\qquad\qquad\qquad
  \cdot \sum\limits_{h\in\dsb{H}}
  \com{\sqrt{S|\cG_{r^E_\theta,h}|}}
  \sqrt{\com{\sum\limits_{\substack{g\in \cG_{r^E_\theta,h}}}\sum\limits_{s\in\cS}}
  \P^{\pi^E}(G_h=g\wedge s_h=s)}
  \\
  &\qquad\qquad=
  8\sqrt{\frac{\com{S}\ln\frac{2S\overline{\cG}}{\delta}}{N}}
  \sqrt{\ln\frac{2S\overline{\cG}}{\delta}+
  (A-1)\ln\Bigr{e\Bigr{1+\frac{N}{A-1}}}}
  \com{\sum\limits_{h\in\dsb{H}}
  \sqrt{|\cG_{r^E_\theta,h}|}}
  \\
  &\qquad\qquad\markref{(5)}{\le}
  8\sqrt{\frac{S\com{H\overline{\cG}}\ln\frac{2S\overline{\cG}}{\delta}}{N}}
  \sqrt{\ln\frac{2S\overline{\cG}}{\delta}+
  (A-1)\ln\Bigr{e\Bigr{1+\frac{N}{A-1}}}},
  \end{align*}
  where at (1) we use that, if $\sum_{a'}M_{h}(s,g,a')=0$, then:
  \begin{align*}
    &\sum\limits_{a\in\cA}\Big|\pi_{r^E_\theta,h}(a|g,s)-\Bigr{
    \frac{M_{h}(s,g,a)}{\sum_{a'}M_{h}(s,g,a')}\indic{\sum_{a'}M_{h}(s,g,a')>0}+\frac{1}{A}\indic{\sum_{a'}M_{h}(s,g,a')=0}
  }\Big|\\
  &\qquad\qquad=\sum\limits_{a\in\cA}\Big|\pi_{r^E_\theta,h}(a|g,s)-\frac{1}{A}\Big|\\
  &\qquad\qquad\le2,
  \end{align*}
  as we the total variation distance between two probability distributions
  cannot exceed 1. Instead, if $\sum_{a'}M_{h}(s,g,a')>0$, \emph{conditioning}
  on $\sum_{a'}M_{h}(s,g,a')$, at all $s,g$ where $\P^{\pi^E}(G_h=g \wedge
  s_{h}=s)>0$, we note that $M_{h}(s,g,a)/\sum_{a'}M_{h}(s,g,a')$ is the
  empirical vector of probabilities of $\pi_{r^E_\theta,h}(a|g,s)$ (recall its
  definition from Eq. \ref{eq: def policy imitate same occ meas}), thus we can
  apply Lemma 8 of \citet{kaufmann2021adaptive} to get that, for any
  $\delta\in(0,1)$:
  \begin{align*}
    &\P^{\pi^E}\Big(
      KL\Bigr{\frac{M_{h}(s,g,\cdot)}{\sum_{a'}M_{h}(s,g,a')}\Big\| \pi_{r^E_\theta,h}(\cdot|g,s)}\\
     &\qquad\qquad \le
      \frac{\ln\frac{1}{\delta}+ (A-1)\ln\bigr{e\bigr{1+\frac{\sum_{a'}M_{h}(s,g,a')}{A-1}}}}{\sum_{a'}M_{h}(s,g,a')}
    \Big)\ge 1-\delta.
  \end{align*}
  Combining this result with the Pinsker's inequality, that tells us that
  $\|x-y\|_1\le \sqrt{2 KL(x\| y)}$, and with a union bound over all
  $h\in\dsb{H}$, $s\in\cS$, $g\in\cG_{r^E_\theta,h}$, we get the passage in (1)
  w.p. $1-\delta/2$. Note that we add an additional 2 for the case
  $\sum_{a'}M_{h}(s,g,a')=0$, and we define $\overline{\cG}\coloneqq
  \sum_{h\in\dsb{H}}|\cG_{r^E_\theta,h}|$.
  At (2) we bound $\sum_{a'}M_{h}(s,g,a')\le N$, and bring that quantity
  outside, at (3) we apply Lemma A.1 of \citet{xie2021bridging}, after having
  noticed that $\sum_{a'}M_{h}(s,g,a')\sim\text{Bin}\Bigr{ N,
  \P^{\pi^E}(G_h=g\wedge s_{h}=s)}$, and make it hold for
  all $s,g,h$ w.p. $1-\delta/2$.
  At (4) and (5) we apply the Cauchy-Schwarz's inequality.

  Now, we impose that this quantity is smaller than $\epsilon$:
  \begin{align*}
    &8\sqrt{\frac{SH\overline{\cG}\ln\frac{2S\overline{\cG}}{\delta}}{N}}
  \sqrt{\ln\frac{2S\overline{\cG}}{\delta}+
  (A-1)\ln\Bigr{e\Bigr{1+\frac{N}{A-1}}}}\le\epsilon\\
  &\qquad\qquad\iff
N\ge \frac{64SH\overline{\cG}\ln^2\frac{2S\overline{\cG}}{\delta}}
{\epsilon^2}
+
\frac{64SH\overline{\cG}(A-1)\ln\frac{2S\overline{\cG}}{\delta}}
{\epsilon^2}
\ln\Bigr{\frac{eN}{A-1}+e}.
  \end{align*}
Thanks to Lemma J.3 of \citet{lazzati2024offline}, we know that this inequality
is satisfied with:
\begin{align*}
  N\le \frac{128SH\overline{\cG}\ln^2\frac{2S\overline{\cG}}{\delta}}
{\epsilon^2}
+ \frac{192SH\overline{\cG}(A-1)\ln\frac{2S\overline{\cG}}{\delta}}
{\epsilon^2}
\ln\Bigr{\frac{128eSH\overline{\cG}\ln\frac{2S\overline{\cG}}{\delta}}
{\epsilon^2}}
+A-1.
\end{align*}
Rearranging and applying a final union bound concludes the proof.
\end{proof}

\subsubsection{Additional Discussion on RS-BC}\label{apx: more discussion rsbc}

First, we observe that the sample complexity bound in Theorem \ref{thr: rsbc}
cannot be improved if we extend our analysis with that of
\citet{foster2024bcallyouneed}. Indeed, \citet{foster2024bcallyouneed} also
provides an $1/\epsilon^2$ dependence as our proof, that combined with the
additional $1/\epsilon$ due to discretization, would give the same
$1/\epsilon^3$ rate. Moreover, \citet{foster2024bcallyouneed} would not allow to
improve even the $H^6$ dependence in the horizon in our proof, as their
Corollary 3.1 combined with a simple variation of our proof that considers an
MDP with an augmented state space, would still provide an $H^4$ dependence that
should be combined with the $H^2$ arising from bounding the Wasserstein with $H$
times the total variation, and taking the square.
Note that these considerations on \citet{foster2024bcallyouneed} implicitly
assumed that the proof of \citet{foster2024bcallyouneed} can be extended to
non-Markovian expert's policies with the same rate, which has to be demonstrated
as well.

Second, we mention that \citet{rajaraman2020fundamentalimitationlearning}
provides for IL a $1/\epsilon$ dependence instead of $1/\epsilon^2$. However,
note we remark that the result is in \emph{expectation}, and not with \emph{high
probability}, as remarked also by \citet{foster2024bcallyouneed} in their
footnote 21. Indeed, this explains why the $1/\epsilon$ rate of
\citet{rajaraman2020fundamentalimitationlearning} seems to overcome the
$1/\epsilon^2$ in the lower bound of Theorem G.1 of
\citet{foster2024bcallyouneed}.

Lastly, we mention that \rsbc (and also the theoretical guarantees in Theorem
\ref{thr: rsbc}) can be easily extended to the setting in which $r^E$ is unknown
but \emph{observed}, namely, in which expert's trajectories are
state-action-reward trajectories $(s_1,a_1,r_1,\dotsc)$. Indeed, looking at
Algorithm \ref{alg: rsbc}, note that the computation of the returns
$G(\omega;r^E)$ does not require knowledge of $r^E$ in state-action pairs never
observed. This is different from \rskt, in which we require knowledge of $r^E$ everywhere.

\subsection{Known-Transition Setting}
\label{apx: known transition one r}

In Appendix \ref{apx: proof thr rskt}, we prove Theorem \ref{thr: rskt}, while
in Appendix \ref{apx: details LP}, we write down explicitly the LP in Eq.
\eqref{eq: opt problem LP rskt}.

\subsubsection{Proof of Theorem \ref{thr: rskt}}
\label{apx: proof thr rskt}

\rsktupperbound*
\begin{proof}
We begin by showing that the estimate of return distribution $\widehat{\eta}$
computed by \rskt at Line \ref{line: kt estimate expert ret distrib} is close to
the expert's return distribution with high probability.
To this aim, we write:
\begin{align*}
  \cW\Bigr{
      \eta^{\pi^E}_{r^E},\widehat{\eta}
    }
    &\markref{(1)}{\le}
    \cW\Bigr{
      \eta^{\pi^E}_{r^E},\com{\eta^{\pi^E}_{r^E_\theta}}
    }
    +
    \cW\Bigr{
      \com{\eta^{\pi^E}_{r^E_\theta}},
      \widehat{\eta}
    }\\
    &\markref{(2)}{\le}
    \com{H\theta/2}
    +
    \cW\Bigr{
      \eta^{\pi^E}_{r^E_\theta},
      \widehat{\eta}
    }\\
    &\markref{(3)}{=}
    H\theta/2
    +
    \com{\int\limits_0^H \Biga{F_{\eta^{\pi^E}_{r^E_\theta}}(x)-
    F_{\widehat{\eta}}(x)}dx}
    \\
    &\le
    H\theta/2
    +
    \int\limits_0^H \com{\sup\limits_{x'\in[0,H]}}\Biga{F_{\eta^{\pi^E}_{r^E_\theta}}\com{(x')}-
    F_{\widehat{\eta}}\com{(x')}}dx
    \\
    &=
    H\theta/2
    +\sup\limits_{x'\in[0,H]}\Biga{F_{\eta^{\pi^E}_{r^E_\theta}}(x')-
    F_{\widehat{\eta}}(x')}
    \com{\int\limits_0^H dx}
    \\
    &=
    H\theta/2
    +\com{H}\sup\limits_{x'\in[0,H]}\Biga{F_{\eta^{\pi^E}_{r^E_\theta}}(x')-
    F_{\widehat{\eta}}(x')}
    \\
    &\markref{(4)}{\le}
    H\theta/2
    +H\com{\epsilon'},
\end{align*}
where at (1) we use triangle's inequality, at (2) we apply Lemma \ref{lemma:
different r same p}, at (3) we use symbol $F_q$ for the distribution function of
any probability measure $q$, at (4) we apply the DKW inequality
\citep{kiefer1959dkw,massart1990dkw}, as $F_{\widehat{\eta}}(x)=\sum_{x'\le
x}\widehat{\eta}(x') = \sum_{x'\le x} \frac{1}{N}\sum_{i\in\dsb{N}}
\indic{\sum_{h=1}^H r^E_{\theta,h}(s_{h}^i,a_{h}^i)=x'} =
\frac{1}{N}\sum_{i\in\dsb{N}} \indic{\sum_{h=1}^H
r^E_{\theta,h}(s_{h}^i,a_{h}^i)\le x'}$ corresponds to the empirical
distribution function of $\eta^{\pi^E}_{r^E_\theta}$. Specifically, in our
setting, the DKW inequality tells us that, for any $\epsilon'>0$, it holds that:
\begin{align}
  \P^{\pi^E}\Bigr{\sup\limits_{x'\in[0,H]}\Biga{F_{\eta^{\pi^E}_{r^E_\theta}}(x')-
    F_{\widehat{\eta}}(x')}\le\epsilon'}\ge 1-2e^{-2N(\epsilon')^2}.
\end{align}
By imposing the term on the right hand side to be $1-\delta$, and solving w.r.t.
$N$, we get that:
\begin{align*}
  N\le \frac{1}{2(\epsilon')^2}\ln\frac{2}{\delta}.
\end{align*}
Now, building on this result, we can write:
\begin{align*}
  \cW\Bigr{
      \eta^{\pi^E}_{r^E},\eta^{\widehat{\pi}}_{r^E}
    }
    &\markref{(5)}{\le}
    \cW\Bigr{
      \eta^{\pi^E}_{r^E},\com{\widehat{\eta}}
    }
    +
    \cW\Bigr{
      \com{\widehat{\eta}},
      \com{\eta^{\widehat{\pi}}_{r^E_\theta}}
    }
    +
    \cW\Bigr{
      \com{\eta^{\widehat{\pi}}_{r^E_\theta}},
      \eta^{\widehat{\pi}}_{r^E}
    }
    \\
    &\markref{(6)}{\le}
   \com{H\theta/2
    +H\epsilon'}
    +
    \cW\Bigr{
      \widehat{\eta},
      \eta^{\widehat{\pi}}_{r^E_\theta}
    }
    +
    \cW\Bigr{
      \eta^{\widehat{\pi}}_{r^E_\theta},
      \eta^{\widehat{\pi}}_{r^E}
    }
    \\
    &\markref{(7)}{\le}
    \com{H\theta}
    +H\epsilon'
    +
    \cW\Bigr{
      \widehat{\eta},
      \eta^{\widehat{\pi}}_{r^E_\theta}
    }
    \\
    &\markref{(8)}{=}
    H\theta
    +H\epsilon'
    +
    \com{\min\limits_{\pi\in\Pi(r^E_\theta)}}\cW\Bigr{
      \widehat{\eta},
      \eta^{\com{\pi}}_{r^E_\theta}
    }
    \\
    &\markref{(9)}{\le}
    H\theta
    +H\epsilon'
    +
    \cW\Bigr{\widehat{\eta},
      \eta^{\com{\pi_{r^E_\theta}}}_{r^E_\theta}
    }
    \\
    &\markref{(10)}{\le}
    H\theta
    +H\epsilon'
    +
    \cW\Bigr{\widehat{\eta},
      \com{\eta^{\pi^E}_{r^E}}
    }
    +
    \cW\Bigr{\com{\eta^{\pi^E}_{r^E}},
      \com{\eta^{\pi_{r^E_\theta}}_{r^E}}
    }
    +
    \cW\Bigr{\com{\eta^{\pi_{r^E_\theta}}_{r^E}},
      \eta^{\pi_{r^E_\theta}}_{r^E_\theta}
    }
    \\
    &\markref{(11)}{\le}
    \com{2H\theta}
    +\com{2H\epsilon'}
    +
    \cW\Bigr{\eta^{\pi^E}_{r^E},
      \eta^{\pi_{r^E_\theta}}_{r^E}
    }
    +
    \cW\Bigr{\eta^{\pi_{r^E_\theta}}_{r^E},
      \eta^{\pi_{r^E_\theta}}_{r^E_\theta}
    }
    \\
    &\markref{(12)}{\le}
    \com{\frac{5}{2}H\theta}
    +2H\epsilon'
    +
    \cW\Bigr{\eta^{\pi^E}_{r^E},
      \eta^{\pi_{r^E_\theta}}_{r^E}
    }
    \\
    &\markref{(12)}{\le}
    \com{\frac{7}{2}H\theta}
    +2H\epsilon',
\end{align*}
  where at (5) we apply triangle's inequality, at (6) we use the result above,
  at (7) we apply Lemma \ref{lemma: different r same p}, at (8) we use the
  definition of $\widehat{\pi}$ and the hypothesis of solving the minimization
  problem exactly, at (9) we upper bound with a specific choice of policy, i.e.,
  $\pi_{r^E_\theta}$ (recall Eq. \ref{eq: def policy imitate same occ meas}), at
  (10) we apply triangle's inequality again, at (11) we apply the result above
  again, at (12) we use again Lemma \ref{lemma: different r same p}, and
  finally, at (13), we apply Lemma \ref{lemma: apx policies}.

  If we now choose $\theta=\epsilon/(7H)$, and $\epsilon'=\epsilon/(4H)$, we get that,
with probability $1-\delta$, it holds that:
\begin{align*}
  \cW\Bigr{
      \eta^{\pi^E}_{r^E},\eta^{\widehat{\pi}}_{r^E}
  }\le \epsilon,
\end{align*}
with:
\begin{align*}
  N\le \frac{2H^2}{\epsilon^2}\ln\frac{2}{\delta}.
\end{align*}
\end{proof}

\subsubsection{Explicit Formulation of the LP}
\label{apx: details LP}

The optimization problem in Eq. \eqref{eq: opt problem LP rskt} can be written more
explicitly as follows:
\begin{align}
  &\min\limits_{d\in\RR^{SAH|\cY^\theta|}_{\ge0},\eta\in\RR^{|\cY^\theta|}_{\ge0},
  t\in\RR^{|\cY^\theta|}_{\ge0},x\in\RR^{|\cY^\theta|}_{\ge0}}
  \sum\limits_{g\in\cY^\theta}t(g)\nonumber\\
  &\qquad\text{s.t. }\sum\limits_{(s,a,g)\in\cS\times\cA\times\cY^\theta}
  d_1(s,g,a)=1\label{constr: init1}\\
  &\qquad\sum\limits_{a\in\cA} d_1(s_0,0,a)=1\label{constr: init2}\\
  &\qquad\sum_{a\in\cA} d_h(s,g,a)=
\sum_{s',g',a'}d_{h-1}(s',g',a')p_{h-1}(s|s',a')\indic{r^E_{\theta,h}(s',a')=g-g'}\nonumber\\
&\qquad\qquad\qquad\qquad\qquad\qquad\qquad\qquad\qquad\qquad\qquad
\forall (s,g,h)\in\cS\times\cY^\theta\times\{2,\dotsc,H\}\label{constr: flow}\\
&\qquad \eta(g)=\sum\limits_{(s,a)\in\SA} d_H(s,g-r_{\theta,H}^E(s,a),a)
  \qquad\forall g\in \cY^\theta\label{constr: rel eta d}\\
  &\qquad x(g)=\sum\limits_{g'\in\cY^\theta:\;g'\le g}\Bigr{\eta(g')-\widehat{\eta}(g')}
  \qquad \forall g\in\cY^\theta\label{constr: norm1}\\
  &\qquad -t(g)\le x(g)\le t(g)\qquad \forall g\in\cY^\theta\label{constr: norm2}
\end{align}
where Eqs. \eqref{constr: init1}-\eqref{constr: init2}-\eqref{constr: flow} denote
the flow constraints, i.e., define the set of feasible occupancy measures $\cK$,
Eq. \eqref{constr: rel eta d} enforces that $\eta$ is the return distribution in
$\overline{\cM}$ corresponding to occupancy measure $d$, while Eqs.
\eqref{constr: norm1}-\eqref{constr: norm2} permit to rewrite the Wasserstein
distance in a linear manner.

Observe that the number of optimization variables is
$SAH|\cY^\theta|+|\cY^\theta|=\cO(SAH|\cY^\theta|)$, while the number of
constraints is $2+(H-1)S|\cY^\theta|+3|\cY^\theta|=\cO(SAH|\cY^\theta|)$.

\subsection{When $r^E$ belongs to a finite set of $d$ Rewards}\label{apx: rE
in finite set}

In this appendix, we consider a variant of the known-reward setting, in which
$r^E$ is unknown, but we have knowledge of a set $\cR=\{r^1,\dotsc,r^d\}$
containing $d\ge1$ reward functions, and we also know that $r^E\in\cR$.
We consider the following robust variant of RDM for this setting:
\begin{align}\label{eq: RDM variant set}
    \widehat{\pi}\in\argmin_{\pi\in\Pi^{\text{NM}}}\max\limits_{r\in\cR}
  \cW\Bigr{\eta_{r}^{\pi},\eta_{r}^{\pi^E}}.
\end{align}
To tackle this problem, we will proceed to Section \ref{sec: r known}, by
considering the \emph{cartesian product} of all the rewards in set $\cR$.
Specifically, we first present a class of non-Markovian policies sufficiently
expressive for addressing this task, and then we present two variants of \rsbc
and \rskt. Crucially, we will have exponential dependencies in the number of
rewards $d$ for both the computational and sample complexities.

Let us begin by presenting $\Pi(\cR)$, a generalization of $\Pi(r)$ to
multiple rewards:
\begin{align*}
    \Pi(\cR)\coloneqq\Bigc{\pi&\in\Pi^{\text{NM}}\,\Big|\,
    \exists \phi\in\Delta^\cA_{\dsb{H}\times\cS\times\cG_{r^1}\times\dotsc\times\cG_{r^d}}:\\
    &\pi(a|s,\omega)=\phi_h(a|s,G(\omega;r^1),\dotsc,G(\omega;r^d))\;
   \forall s\in\cS,a\in\cA,h\in\dsb{H},\omega\in\Omega_h
    }.
\end{align*}
Intuitively, $\Pi(\cR)$ contains policies that depend on the amount of rewards
collected so far for every possible reward $r^i$ in $\cR$. 

Now, let us define $\pi_\cR\in\Pi(\cR)$, an analogous of policy $\pi_r$ (Eq.
\ref{eq: def policy imitate same occ meas}).
For any set of rewards $\cR$ and expert policy $\pi^E\in\Pi^{\text{NM}}$, define
$\pi_{\cR}\in\Pi(\cR)$ as the policy whose probability of taking an action $a$ in
state $s$ with history $\omega\in\Omega_h$ coincides with the ``average''
probability with which $\pi^E$ selects $a$ in $s$ after accumulating
$G(\omega;r^i)$ reward for each reward $r^i\in\cR$:
\begin{align}\label{eq: def policy imitate same occ meas set}
    \pi_\cR(a|s,\omega)\coloneqq
      \frac{\P^{\pi^E}(s_h=s,\;a_h=a,\;G^1_h=G(\omega;r^1),\;\dotsc,\;G^d_h=G(\omega;r^d))}{
        \P^{\pi^E}(s_h=s,\;G^1_h=G(\omega;r^1),\;\dotsc,\;G^d_h=G(\omega;r^d))},
\end{align}
where we defined the random return at stage $h$ under reward $r^i\in\cR$ as
$G_h^i\coloneqq\sum_{h'=1}^{h-1}r_{h'}^i(s_{h'},a_{h'})$.
If the denominator is zero, we set $\pi_{r}(a|s,\omega)=1/A$.

We have the following result replicating Lemma \ref{lemma: same return
distribution}:
\begin{restatable}{lemma}{lemmasameretdistribset}\label{lemma: same return
distribution set}
    Let $\cM$ be any \MDPr, $\cR=\{r^1,\dotsc,r^d\}$ any set of $d\ge1$ rewards
    containing the unknown $r^E$, and $\pi^E\in\Pi^{\text{NM}}$ be any policy.
    Then, the policy $\pi_\cR\in\Pi(\cR)$ is a minimizer of Eq. \eqref{eq: RDM
    variant set}, and satisfies $\eta^{\pi_\cR}_{r^i}(g)=\eta^{\pi^E}_{r^i}(g)$
    for all $g$.
\end{restatable}
We prove it in Appendix \ref{apx: proofs set rewards}.
However, $\Pi(\cR)=\Pi^{\text{NM}}$ of course for some rewards. Thus, we can
discretize, by defining:
$\cR^\theta\coloneqq\{r^1_\theta,\dotsc,r^d_\theta\}$, i.e., by discretizing
each reward inside $\cR$. Then, it should be clear that the memory required for
storing a policy in $\Pi(\cR^\theta)$ scales as $\cO(SAH|\cY^\theta|^d)$
(because we do the cartesian product).
Analogously to Lemma \ref{lemma: apx policies}, we can bound the approximation
error (proof in Appendix \ref{apx: proofs set rewards}):
\begin{restatable}{lemma}{lemmasameretdistribwithlesspoliciesset}\label{lemma:
apx policies set}
  Let $\theta\in(0,1]$.
  Let $\cM$ be any \MDPr, $\cR=\{r^1,\dotsc,r^d\}$ any set of $d\ge1$ rewards
containing the unknown $r^E$, and $\pi^E\in\Pi^{\text{NM}}$ be any policy.
  Then, the policy $\pi_{\cR^\theta}\in\Pi(\cR^\theta)$ satisfies
  $    \cW\bigr{\eta^{\pi_{\cR^\theta}}_{r^E},\eta^{\pi^E}_{r^E}}
    \le H\theta$.
\end{restatable}
Now, we address this problem with a variant of \rsbc for the no-interaction
setting, and a variant of \rskt for the known-transition setting.

\begin{figure}[t!]
\centering
\begin{minipage}[t]{0.98\linewidth}
\input{rs_bc-set.tex}
\end{minipage}
\end{figure}

We begin with a variant of \rsbc, reported in Algorithm \ref{alg: rsbc set}.
Simply, we count the (discretized) occurrences for every reward in $\cR$.
We have the following result (proof in Appendix \ref{apx: proofs set rewards
rsbc}):
\begin{restatable}{thr}{thrbcsamplecompknownrset}
\label{thr: rsbc set}
Let $\epsilon\in(0,H]$ and $\delta\in(0,1)$.
Let $\cM$ be any \MDPr, $\cR=\{r^1,\dotsc,r^d\}$ any set of $d\ge1$ rewards
containing the unknown $r^E$, and $\pi^E\in\Pi^{\text{NM}}$ be any policy.
Then, choosing $\theta=\epsilon/(4H)$, with probability at least $1-\delta$, the
policy $\widehat{\pi}$ output by Algorithm~\ref{alg: rsbc set} satisfies
$\cW(\eta^{\pi^E}_{r^E}, \eta^{\widehat{\pi}}_{r^E})\le\epsilon$,
with a number of samples:
\begin{align}\label{eq: sample complexity rsbc set}
    N\le \widetilde{\cO}\biggr{\frac{SH^{4+2d}d^2\ln\frac{1}{\delta}}
{\epsilon^{2+d}}\Bigr{A+ \ln\frac{1}{\delta}}}.
\end{align}
\end{restatable}
Observe that, for $d=1$, we retrieve the known-reward setting, and the number of
samples in Eq. \eqref{eq: sample complexity rsbc set} matches that of \rsbc (Eq.
\ref{eq: sample complexity rsbc}).

\begin{figure}[t!]
\centering
\begin{minipage}[t]{0.98\linewidth}
\input{rs_kt-set.tex}
\end{minipage}
\end{figure}

Now, we do the same for \rskt. See Algorithm \ref{alg: rskt set} for a variant
of the algorithm. We have (proof in Appendix \ref{apx: proofs set rewards
rskt}):
\begin{restatable}{thr}{rsktupperboundset}\label{thr: rskt set}
Let $\epsilon\in(0,H]$ and $\delta\in(0,1)$.
Let $\cM$ be any \MDPr, $\cR=\{r^1,\dotsc,r^d\}$ any set of $d\ge1$ rewards
containing the unknown $r^E$, and $\pi^E\in\Pi^{\text{NM}}$ be any policy.
Assume that the optimization problem in
Line~\ref{line: kt compute policy set} is solved exactly.
Then, choosing $\theta=\epsilon/(4H)$, with probability $1-\delta$, the policy
$\widehat{\pi}$ output by Algorithm~\ref{alg: rskt set} satisfies
$\cW(\eta^{\pi^E}_{r^E}, \eta^{\widehat{\pi}}_{r^E})\le\epsilon$,
with:
\begin{align}\label{eq: sample complexity rskt set}
    N\le \cO\biggr{\frac{H^2}{\epsilon^2}\ln\frac{d}{\delta}}.
\end{align}
\end{restatable}
Interestingly, the bound here is still polynomial.
Now, we show an extension of the LP formulation in Eq. \eqref{eq: opt problem LP
rskt} for addressing Line \ref{line: kt compute policy set} of Algorithm
\ref{alg: rskt set}.
Specifically, we just construct an augmented MDP that keeps track of the past
rewards for every possible reward in $\cR$, and note that $\Pi(\cR^\theta)$
describes the set of Markovian policies in this MDP.
So, we want to match a sort of augmented return distribution for this problem:
\begin{align}\label{eq: opt problem LP rskt set}
  &\min_{d\in\cK,\eta^1\in\Delta^{(\cY^\theta)},\dotsc,\eta^d\in\Delta^{(\cY^\theta)}}
  \max\limits_{r^i\in\cR}\cW\!\left(\eta^i,\widehat{\eta}_{r^i}\right)\\
  &\quad\text{s.t. }\;\eta^i(g)=\sum_{s,a,g^1,\dotsc,g^{i-1},g^{i+1},\dotsc,g^d}
  d_H(s,g^1,\dotsc,g^{i-1},g-r_{\theta,H}^i(s,a),g^{i+1},\dotsc,g^d,a)\\
  &\qquad\qquad\qquad\qquad\quad
  \forall g\in \cY^\theta\forall i\in\dsb{d}.
\end{align}
Intuitively, the constraints above enforce that $\eta^i$ is the return
distribution induced by the occupancy measure $d$ w.r.t. the reward
$r^i_\theta$, for all the rewards $r^i\in\cR$, and $\cK$ denotes the set of
feasible occupancy measures in this augmented MDP \citep{puterman1994markov}:
\begin{align*}
  \scalebox{0.95}{$  \displaystyle
  \cK\coloneqq\Bigc{d\in\Delta_{\dsb{H}}^{\overline{\cS}\times\cA}\,\Big|\,
  \sum_a d_1(\overline{s}_0,a)=1
  \wedge\forall \overline{s}\in\overline{\cS},h\ge 2:\;
  \sum_a d_h(\overline{s},a)=
\sum_{\overline{s}',a'}d_{h-1}(\overline{s}',a')\overline{p}_{h-1}(\overline{s}|\overline{s}',a')}
  $},
\end{align*}
where the state space is
$\overline{\cS}\coloneqq\cS\times\cY^\theta\times\dotsc\times\cY^\theta$ $d$
times, $\overline{s}_0\coloneqq(s_0,0,\dotsc,0)$ and the transition model is:
\begin{align*}
  p_h(s',g^1,\dotsc,g^d|s,\overline{g}^1,\dotsc,\overline{g}^d,a)\coloneqq
  p_h(s'|s,a)\indic{r_h^1(s,a)+\overline{g}^1=g^1}\dotsc
  \indic{r_h^d(s,a)+\overline{g}^d=g^d}.
\end{align*}
In words, Eq.~\eqref{eq: opt problem LP rskt set} searches for an occupancy
measure $d\in\cK$ that induces the return distribution $\eta^i$ closest to
$\widehat{\eta}_{r^i}$.
From such a solution, a policy $\widehat{\pi}\in\Pi(\cR^\theta)$ with occupancy
$d^{\widehat{\pi}}=d$ (and thus return distribution
$\eta^{\widehat{\pi}}_{r^i_\theta}=\eta^i$ for all $i$) can be recovered via:
\begin{align*}
\widehat{\pi}(a|s,\omega)=\frac{d_h(s,G(\omega;r_\theta^1),\dotsc,G(\omega;r_\theta^d),a)}{
\sum_{a'}d_h(s,G(\omega;r_\theta^1),\dotsc,G(\omega;r_\theta^d),a')} \quad
\forall h\in\dsb{H},\;s\in\cS,\;a\in\cA,\;\omega\in\Omega_h,
\end{align*}
when the denominator is nonzero, and $\widehat{\pi}(a|s,\omega)=1/A$
otherwise \citep{syed2008allinear}.
We remark that, being the set $\cR$ finite, then the minmax above can be
formulated as an LP minimization problem.

\subsubsection{Technical Results and Proofs for the Policy Class}\label{apx: proofs set rewards}

\lemmasameretdistribset*
\begin{proof}
  To prove this result, we show that, for any $r^i\in\cR$, the return
  distribution $\eta^{\pi_\cR}_{r^i}$ coincides with the expert's return
  distribution $\eta^{\pi^E}_{r^i}$.

  For any $r^i\in\cR$ and $g'\in[0,H]$, we can write (we use
  $G_h^j\coloneqq\sum_{h'=1}^{h-1}r_{h'}^j(s_{h'},a_{h'})$ for all
  $j\in\dsb{d}$):
  \begin{align*}
    \eta^{\pi_\cR}_{r^i}(g')&=\P^{\pi_\cR}(G_H^i+r^i_H(s_H,a_H)=g')\\
    &\markref{(1)}{=}\sum\limits_{g^i\in\cG_{r^i,H}}\sum\limits_{s\in\cS}
    \P^{\pi_\cR}(G_H^i=g^i,s_H=s,r^i_H(s,a_H)=g'-g^i)\\
    &=\com{\sum\limits_{g^1\in\cG_{r^1,H}}\dotsc
    \sum\limits_{g^d\in\cG_{r^d,H}}}
    \sum\limits_{s\in\cS}\\
    &\qquad\qquad
    \P^{\pi_\cR}(\com{G_H^1=g^1,\dotsc,G_H^d=g^d},s_H=s,r^i_H(s,a_H)=g'-g^i)\\
    &=\sum\limits_{g^1\in\cG_{r^1,H}}\dotsc
    \sum\limits_{g^d\in\cG_{r^d,H}}
    \sum\limits_{s\in\cS}\com{\sum\limits_{a\in\cA:r^i_H(s,a)=g'-g^i}}\\
    &\qquad\qquad
    \P^{\pi_\cR}(G_H^1=g^1,\dotsc,G_H^d=g^d,s_H=s,\com{a_H=a})\\
    &\markref{(2)}{=}\sum\limits_{g^1\in\cG_{r^1,H}}\dotsc
    \sum\limits_{g^d\in\cG_{r^d,H}}
    \sum\limits_{s\in\cS}\sum\limits_{a\in\cA:r^i_H(s,a)=g'-g^i}\\
    &\qquad\qquad
    \P^{\pi_\cR}(G_H^1=g^1,\dotsc,G_H^d=g^d,s_H=s)
    \com{\pi_\cR(a|s,g^1,\dotsc,g^d)}\\
    &\markref{(3)}{=}\sum\limits_{g^1\in\cG_{r^1,H}}\dotsc
    \sum\limits_{g^d\in\cG_{r^d,H}}
    \sum\limits_{s\in\cS}\sum\limits_{a\in\cA:r^i_H(s,a)=g'-g^i}\\
    &\qquad\qquad
    \com{\P^{\pi^E}}(G_H^1=g^1,\dotsc,G_H^d=g^d,s_H=s)
    \pi_\cR(a|s,g^1,\dotsc,g^d)\\
    &=\sum\limits_{g^1\in\cG_{r^1,H}}\dotsc
    \sum\limits_{g^d\in\cG_{r^d,H}}
    \sum\limits_{s\in\cS}\com{\P^{\pi^E}(G_H^1=g^1,\dotsc,G_H^d=g^d,s_H=s)}\\
    &\qquad\qquad
    \sum\limits_{a\in\cA:r^i_H(s,a)=g'-g^i}
    \pi_\cR(a|s,g^1,\dotsc,g^d)\\
    &\markref{(4)}{=}\sum\limits_{g^1\in\cG_{r^1,H}}\dotsc
    \sum\limits_{g^d\in\cG_{r^d,H}}
    \sum\limits_{s\in\cS}\P^{\pi^E}(G_H^1=g^1,\dotsc,G_H^d=g^d,s_H=s)\\
    &\qquad\qquad
    \sum\limits_{a\in\cA:r^i_H(s,a)=g'-g^i}
    \com{\frac{\P^{\pi^E}(s_H=s,\;a_H=a,\;G^1_H=g^1,\;\dotsc,\;G^d_H=g^d)}{
        \P^{\pi^E}(s_H=s,\;G^1_H=g^1,\;\dotsc,\;G^d_H=g^d)}}\\
    &=\sum\limits_{g^1\in\cG_{r^1,H}}\dotsc
    \sum\limits_{g^d\in\cG_{r^d,H}}
    \sum\limits_{s\in\cS}\sum\limits_{a\in\cA}\com{\indic{r^i_H(s,a)=g'-g^i}}\\
    &\qquad\qquad
    \P^{\pi^E}(s_H=s,\;a_H=a,\;G^1_H=g^1,\;\dotsc,\;G^d_H=g^d)\\
    &=\com{\sum\limits_{g^i\in\cG_{r^i,H}}}
    \sum\limits_{s\in\cS}\sum\limits_{a\in\cA}\indic{r^i_H(s,a)=g'-g^i}
    \P^{\pi^E}(s_H=s,\;a_H=a,\;\com{G^i_H=g^i})\\
    &=\eta^{\pi^E}_{r^i}(g'),
  \end{align*}
  where at (1) we define symbol $\cG_{r,H}\coloneqq\{g\in[0,H-1]\,| \,\exists
\omega\in\Omega_H:\, G(\omega;r)=g\}$ for any $r$, at (2) we recognize that, by
definition, $\pi_\cR$ takes actions only depending on the current state, stage
and past rewards for any $r^i$, and we denote with brevity this fact with
$\pi_\cR(a|s,g^1,\dotsc,g^d)$, at (3) we use Lemma \ref{lemma: Psg equal Psg
set}, at (4) we use the definition of $\pi_\cR(a|s,g)$ (Eq. \ref{eq: def policy
imitate same occ meas set}) where the denominator is not 0, noting that, in that
case, the entire expression evaluates to 0.
\end{proof}

\lemmasameretdistribwithlesspoliciesset*
\begin{proof}
    For any reward $r^i\in\cR$, we can write:
  \begin{align*}
    \cW\Bigr{\eta^{\pi_{\cR^\theta}}_{r^i},\eta^{\pi^E}_{r^i}}&\markref{(1)}{\le}
    \cW\Bigr{\eta^{\pi_{\cR^\theta}}_{r^i},\com{\eta^{\pi_{\cR^\theta}}_{r^i_\theta}}}
    + \cW\Bigr{\com{\eta^{\pi_{\cR^\theta}}_{r^i_\theta}},\com{\eta^{\pi^E}_{r^i_\theta}}}
    +\cW\Bigr{\com{\eta^{\pi^E}_{r^i_\theta}},\eta^{\pi^E}_{r^i}}\\
    &\markref{(2)}{\le}
    2H\|r^i-r^i_\theta\|_\infty
    +
    \cW\Bigr{\com{\eta^{\pi_{\cR^\theta}}_{r^i_\theta}},{\eta^{\pi^E}_{r^i_\theta}}}\\
    &\markref{(3)}{\le}
    H\theta
    +
    \cW\Bigr{\com{\eta^{\pi_{\cR^\theta}}_{r^i_\theta}},{\eta^{\pi^E}_{r^i_\theta}}}\\
    &\markref{(4)}{\le}
    H\theta,
  \end{align*}
  where at (1) we apply twice the triangle's inequality, at (2) we apply twice
  Lemma \ref{lemma: different r same p}, at (3) we realize that, by definition
  of $r^i_\theta$, it holds that $\|r^i-r^i_\theta\|_\infty\le\theta/2$, and
  finally, at (4), we apply Lemma \ref{lemma: same return distribution set} with
  set $\cR^\theta$ and expert's policy $\pi^E$.
  
  The proof is concluded after having observed that $r^E\in\cR$ by hypothesis,
  and so these passages hold also for $r^E$.
\end{proof}

\begin{restatable}{lemma}{etanmequaletamset}\label{lemma: Psg equal Psg set}
Let $\cM$ be any \MDPr, $\cR=\{r^1,\dotsc,r^d\}$ any set of $d\ge1$ rewards
containing the unknown $r^E$, and $\pi\in\Pi^{\text{NM}}$ be any policy.
Let $\pi_\cR\in\Pi(\cR)$ be the policy defined as in Eq. \eqref{eq: def policy
imitate same occ meas set} for expert's policy $\pi$.
Then, for all $h\in\dsb{H}$, $s\in\cS$ and $g^1,\dotsc,g^d\in[0,h-1]$, it holds that:
  \begin{align*}
    \P^{\pi_\cR}\Bigr{G^1_h=g^1\wedge\dotsc\wedge G^d_h=g^d\wedge s_h=s}=
    \P^{\pi}\Bigr{G^1_h=g^1\wedge\dotsc\wedge G^d_h=g^d \wedge s_h=s},
  \end{align*}
  where we used $G_h^i\coloneqq\sum_{h'=1}^{h-1}r_{h'}^i(s_{h'},a_{h'})$.
\end{restatable}
\begin{proof}
We prove the result by induction.
  Let us begin with the base case: $h=1$. For all $s\in\cS$ and
  $g^,\dotsc,g^d\in\{0\}^d$, we have:
  \begin{align*}
    \P^{\pi_\cR}\Bigr{G^1_1=g^1\wedge\dotsc\wedge G^d_1=g^d\wedge s_1=s}&=\indic{g^1=0\wedge\dotsc\wedge g^d=0}
    \indic{s=s_0}\\&=
    \P^{\pi}\Bigr{G^1_1=g^1\wedge\dotsc\wedge G^d_1=g^d\wedge s_1=s},
  \end{align*}
  where we noticed that, for $h=1$, no action is taken yet.
  Now, let us consider any stage $h\in\{2,3,\dotsc,H\}$, and let us make the
  induction hypothesis that, for all $h'\in\dsb{h-1}$, for all $s\in\cS$ and
  $g^1,\dotsc,g^d\in[0,h'-1]^d$, it holds that:
  \begin{align*}
    \P^{\pi_\cR}\Bigr{G^1_{h'}=g^1\wedge\dotsc\wedge G^d_{h'}=g^d\wedge s_{h'}=s}=
    \P^{\pi}\Bigr{G^1_{h'}=g^1\wedge\dotsc\wedge G^d_{h'}=g^d\wedge s_{h'}=s}.
  \end{align*}
  Then, for any $s'\in\cS$ and
  $\overline{g}^1,\dotsc,\overline{g}^d\in[0,h-1]^d$, we can write:
  \begin{align*}
    &\P^{\pi_\cR}(G^1_{h}=\overline{g}^1\wedge\dotsc\wedge G^d_{h}=\overline{g}^d\wedge s_h=s')\\
    &\qquad\qquad\markref{(1)}{=}
    \sum\limits_{\substack{\omega\in\Omega_{h-1},(s,a)\in\SA:\\G(\omega;r^i)+r_{h-1}^i(s,a)=\overline{g}^i\;\forall i}}
    \P^{\pi_\cR}(\omega_{h-1}=\omega\wedge s_{h-1}=s\wedge a_{h-1}=a\wedge s_h=s')\\
    &\qquad\qquad\markref{(2)}{=}\com{\sum\limits_{g^1\in\cG_{r^1,h-1}}
    \dotsc\sum\limits_{g^d\in\cG_{r^d,h-1}}
    \sum\limits_{\substack{\omega\in\Omega_{h-1}:\\
    G(\omega;r^i)=g^i\;\forall i}}\sum\limits_{\substack{(s,a)\in\SA:\\
    r_{h-1}^i(s,a)=\overline{g}^i-g^i\;\forall i}}}\\
    &\qquad\qquad\qquad\qquad
    \P^{\pi_\cR}(\omega_{h-1}=\omega\wedge s_{h-1}=s\wedge a_{h-1}=a\wedge s_h=s')\\
    &\qquad\qquad\markref{(3)}{=}
    \sum\limits_{g^1\in\cG_{r^1,h-1}}
    \dotsc\sum\limits_{g^d\in\cG_{r^d,h-1}}
    \sum\limits_{\substack{\omega\in\Omega_{h-1}:\\
    G(\omega;r^i)=g^i\;\forall i}}\sum\limits_{\substack{(s,a)\in\SA:\\
    r_{h-1}^i(s,a)=\overline{g}^i-g^i\;\forall i}}\com{\P^{\pi_\cR}(\omega_{h-1}=\omega\wedge s_{h-1}=s)}\\
    &\qquad\qquad\qquad\qquad
    \com{\cdot\P^{\pi_\cR}(a_{h-1}=a|\omega,s)\P^{\pi_\cR}(s_h=s'|\omega,s,a)}\\
    &\qquad\qquad\markref{(4)}{=} \sum\limits_{g^1\in\cG_{r^1,h-1}}
    \dotsc\sum\limits_{g^d\in\cG_{r^d,h-1}}
    \sum\limits_{\substack{\omega\in\Omega_{h-1}:\\
    G(\omega;r^i)=g^i\;\forall i}}\sum\limits_{\substack{(s,a)\in\SA:\\
    r_{h-1}^i(s,a)=\overline{g}^i-g^i\;\forall i}}\P^{\pi_\cR}(\omega_{h-1}=\omega\wedge s_{h-1}=s)\\
    &\qquad\qquad\qquad\qquad
    \cdot\P^{\pi_\cR}(a_{h-1}=a|\omega,s)\com{p_{h-1}(s'|s,a)}\\
    &\qquad\qquad\markref{(5)}{=}
    \sum\limits_{g^1\in\cG_{r^1,h-1}}\dotsc\sum\limits_{g^d\in\cG_{r^d,h-1}}
    \sum\limits_{\substack{\omega\in\Omega_{h-1}:\\G(\omega;r^i)=g^i\;\forall i}}
    \sum\limits_{\substack{(s,a)\in\SA:\\r_{h-1}^i(s,a)=\overline{g}^i-g^i\;\forall i}}
    \P^{\pi_\cR}(\omega_{h-1}=\omega\wedge s_{h-1}=s)\\
    &\qquad\qquad\qquad\qquad
    \cdot\com{\pi_\cR(a|\omega,s)}p_{h-1}(s'|s,a)\\
    &\qquad\qquad\markref{(6)}{=}\sum\limits_{g^1\in\cG_{r^1,h-1}}
    \dotsc\sum\limits_{g^d\in\cG_{r^d,h-1}}
    \com{
      \sum\limits_{\substack{(s,a)\in\SA:\\r_{h-1}^i(s,a)=\overline{g}^i-g^i\;\forall i}}
    \sum\limits_{\substack{\omega\in\Omega_{h-1}:\\G(\omega;r^i)=g^i\;\forall i}}
    }
    \P^{\pi_\cR}(\omega_{h-1}=\omega\wedge s_{h-1}=s)\\
    &\qquad\qquad\qquad\qquad
    \cdot\pi_\cR(a|\omega,s)p_{h-1}(s'|s,a)\\
    &\qquad\qquad\markref{(7)}{=}\sum\limits_{g^1\in\cG_{r^1,h-1}}
    \dotsc\sum\limits_{g^d\in\cG_{r^d,h-1}}
    \sum\limits_{\substack{(s,a)\in\SA:\\r_{h-1}^i(s,a)=\overline{g}^i-g^i\;\forall i}}\\
    &\qquad\qquad\qquad\qquad
    \com{\pi_\cR(a|g^1,\dotsc,g^d,s)p_{h-1}(s'|s,a)}
    \sum\limits_{\substack{\omega\in\Omega_{h-1}:\\G(\omega;r^i)=g^i\;\forall i}}
    \P^{\pi_\cR}(\omega_{h-1}=\omega\wedge s_{h-1}=s)\\
    &\qquad\qquad=\sum\limits_{g^1\in\cG_{r^1,h-1}}
    \dotsc\sum\limits_{g^d\in\cG_{r^d,h-1}}
    \sum\limits_{\substack{(s,a)\in\SA:\\r_{h-1}^i(s,a)=\overline{g}^i-g^i\;\forall i}}
    \pi_\cR(a|g^1,\dotsc,g^d,s)p_{h-1}(s'|s,a)
    \\&\qquad\qquad\qquad\qquad
    \com{\P^{\pi_\cR}(G_{h-1}^1=g^1\wedge\dotsc\wedge G_{h-1}^d=g^d\wedge s_{h-1}=s)}\\
    &\qquad\qquad\markref{(8)}{=}\sum\limits_{g^1\in\cG_{r^1,h-1}}
    \dotsc\sum\limits_{g^d\in\cG_{r^d,h-1}}
    \sum\limits_{\substack{(s,a)\in\SA:\\r_{h-1}^i(s,a)=\overline{g}^i-g^i\;\forall i}}
    \pi_\cR(a|g^1,\dotsc,g^d,s)p_{h-1}(s'|s,a)
    \\&\qquad\qquad\qquad\qquad
   \com{\P^{\pi}}(G_{h-1}^1=g^1\wedge\dotsc\wedge G_{h-1}^d=g^d\wedge s_{h-1}=s)\\
    &\qquad\qquad\markref{(9)}{=}\sum\limits_{g^1\in\cG_{r^1,h-1}}
    \dotsc\sum\limits_{g^d\in\cG_{r^d,h-1}}
    \sum\limits_{\substack{(s,a)\in\SA:\\r_{h-1}^i(s,a)=\overline{g}^i-g^i\;\forall i}}
    \\&\qquad\qquad\qquad\qquad
    \com{\frac{\P^{\pi}(G_{h-1}^1=g^1\wedge\dotsc\wedge G_{h-1}^d=g^d\wedge s_{h-1}=s\wedge a_{h-1}=a)}{
        \P^{\pi}(G_{h-1}^1=g^1\wedge\dotsc\wedge G_{h-1}^d=g^d\wedge s_{h-1}=s)
    }}\\
    &\qquad\qquad\qquad\qquad
    \cdot p_{h-1}(s'|s,a)\P^{\pi}(G_{h-1}^1=g^1\wedge\dotsc\wedge G_{h-1}^d=g^d\wedge s_{h-1}=s)\\
    &\qquad\qquad=\sum\limits_{g^1\in\cG_{r^1,h-1}}
    \dotsc\sum\limits_{g^d\in\cG_{r^d,h-1}}
    \sum\limits_{\substack{(s,a)\in\SA:\\r_{h-1}^i(s,a)=\overline{g}^i-g^i\;\forall i}}
    \\&\qquad\qquad\qquad\qquad
    \P^{\pi}(G_{h-1}^1=g^1\wedge\dotsc\wedge G_{h-1}^d=g^d\wedge s_{h-1}=s\wedge a_{h-1}=a)
     p_{h-1}(s'|s,a)\\
    &\qquad\qquad=\sum\limits_{g^1\in\cG_{r^1,h-1}}
    \dotsc\sum\limits_{g^d\in\cG_{r^d,h-1}}
    \sum\limits_{\substack{(s,a)\in\SA:\\r_{h-1}^i(s,a)=\overline{g}^i-g^i\;\forall i}}
    \\&\qquad\qquad\qquad\qquad
    \com{\P^{\pi}(G_{h-1}^1=g^1\wedge\dotsc\wedge G_{h-1}^d=g^d\wedge s_{h-1}=s\wedge a_{h-1}=a\wedge s_h=s')}\\
    &\qquad\qquad=
    \P^{\pi}(G_{h}^1=\overline{g}^1\wedge\dotsc\wedge G_{h}^d=\overline{g}^d\wedge s_h=s'),
  \end{align*}
  where at (1) we use symbol $\omega_{h''}$ to denote the random trajectory long
  $h''$ stages, i.e., whose realizations belong to $\Omega_{h''}$, for any
  $h''\in\dsb{H}$.
  At (2) we define symbols $\cG_{r,h}\coloneqq\{g\in[0,h-1]\,| \,\exists
\omega\in\Omega_h:\, G(\omega;r)=g\}$ for any $r$,
at (3) we use the chain rule of conditional probabilities, at (4) we use the
Markovianity of the environment, at (5) we note that
$\P^{\pi_\cR}(a_{h-1}=a|\omega,s)$ actually is $\pi_\cR(a|\omega,s)$, at (6) we
exchange the two summations, at (7) we recognize that, by definition,
$\pi_\cR(a|\omega,s)$ takes on the same value for all the trajectories $\omega$
with the same value of return for all rewards $r^i\in\cR$, and thus we can bring
this quantity outside the summation over the $\omega$. We use symbol
$\pi_\cR(a|g^1,\dotsc,g^d,s)$ to denote this fact for brevity. We do the same
also for $p_{h-1}(s'|s,a)$.
  At (8) we use the induction hypothesis, at (9) we replace $\pi_\cR(a|g^1,\dotsc,g^d,s)$ with
  its definition when $\P^{\pi}(G_{h-1}^1=g^1\wedge G_{h-1}^d=g^d\wedge s_{h-1}=s)>0$ as in
  the opposite case the entire formula takes on value zero.
\end{proof}

\subsubsection{Proof of Theorem \ref{thr: rsbc set}}
\label{apx: proofs set rewards rsbc}

\thrbcsamplecompknownrset*
\begin{proof}
We can write:
\begin{align*}
  &\cW\Bigr{\eta^{\pi^E}_{r^E},
  \eta^{\widehat{\pi}}_{r^E}}\\
  &\qquad\qquad\markref{(1)}{\le}
  \cW\Bigr{\eta^{\pi^E}_{r^E},
    \com{\eta^{\pi_{\cR^\theta}}_{r^E}}}
  +
  \cW\Bigr{\com{\eta^{\pi_{\cR^\theta}}_{r^E}},
  \com{\eta^{\pi_{\cR^\theta}}_{r^E_\theta}}}
  +
  \cW\Bigr{\com{\eta^{\pi_{\cR^\theta}}_{r^E_\theta}},
\com{\eta^{\widehat{\pi}}_{r^E_\theta}}}
  + \cW\Bigr{\com{\eta^{\widehat{\pi}}_{r^E_\theta}},
  \eta^{\widehat{\pi}}_{r^E}}\\
&\qquad\qquad\markref{(2)}{\le}
  \com{2H\theta}
  + \cW\Bigr{\eta^{\pi_{\cR^\theta}}_{r^E_\theta},
  \eta^{\widehat{\pi}}_{r^E_\theta}}\\
  &\qquad\qquad\markref{(3)}{\le}
  2H\theta
  +\com{H\Big\|}\eta^{\pi_{\cR^\theta}}_{r^E_\theta}-
  \eta^{\widehat{\pi}}_{r^E_\theta}\com{\Big\|_1}\\
  &\qquad\qquad\markref{(4)}{\le}
  2H\theta+H\sum\limits_{h\in\dsb{H}}\sum\limits_{g^1\in \cG_{r^1_\theta,h}}\dotsc
  \sum\limits_{g^d\in \cG_{r^d_\theta,h}}\sum\limits_{s\in\cS}
  \P^{\pi_{\cR^\theta}}(G_h^1=g^1\wedge\dotsc\wedge G_h^d=g^d\wedge s_{h}=s)\\
  &\qquad\qquad\qquad\qquad
  \Bign{\pi_{\cR^\theta,h}(\cdot|s,g^1,\dotsc,g^d)-\widehat{\pi}_{h}(\cdot|s,g^1,\dotsc,g^d)}_1\\
  &\qquad\qquad\markref{(5)}{=}
  2H\theta+H\sum\limits_{h\in\dsb{H}}\sum\limits_{g^1\in \cG_{r^1_\theta,h}}\dotsc
  \sum\limits_{g^d\in \cG_{r^d_\theta,h}}\sum\limits_{s\in\cS}
  \com{\P^{\pi^E}}(G_h^1=g^1\wedge\dotsc\wedge G_h^d=g^d\wedge s_{h}=s)\\
  &\qquad\qquad\qquad\qquad
  \Bign{\pi_{\cR^\theta,h}(\cdot|s,g^1,\dotsc,g^d)-\widehat{\pi}_{h}(\cdot|s,g^1,\dotsc,g^d)}_1\\
  &\qquad\qquad\markref{(6)}{\le}
  2H\theta+H\epsilon',
\end{align*}
where at (1) we apply triangle's inequality, at (2) we apply Lemma \ref{lemma:
apx policies set} and Lemma \ref{lemma: different r same p} twice, at (3) we use
Particular Case 6.13 of \citet{Villani2008OptimalTO}, which tells us that we can
upper bound the Wasserstein distance between two distributions supported on set
$\cX$ by the diameter of $\cX$ ($\max_{x,x'\in\cX}|x-x'|$) times the one norm
between the two distributions. Since $\cX=[0,H]$ in our case, we get the
expression written above.
At (4) we apply Lemma \ref{lemma: error propagation set} with the notation
defined in that lemma with set $\cR_\theta$ and policies $\pi_{\cR^\theta}$ and
$\widehat{\pi}$. Moreover, this holds recalling that $r^E\in\cR$ and so
$r^E_\theta\in\cR^\theta$.
At (5) we use Lemma \ref{lemma: Psg equal Psg set}.
Lastly, at (6) we apply Lemma \ref{lemma: concentration set} with accuracy
$\epsilon'$.

The result follows by imposing that $\epsilon'\le\frac{\epsilon}{2H}$ and
$2H\theta\le\frac{\epsilon}{2}$, which can be achieved by taking
$\epsilon'=\frac{\epsilon}{2H}$ and $\theta=\frac{\epsilon}{4H}$, and by
observing that:
\begin{align*}
  \overline{\cG}&\coloneqq
  \sum_{h\in\dsb{H}}\prod_{i\in\dsb{d}}|\cG_{r^i_\theta,h}|\\
  &\le \sum_{h\in\dsb{H}}\prod_{i\in\dsb{d}}|\cY^\theta_h|\\
  &= \sum_{h\in\dsb{H}}|\cY^\theta_h|^d\\
  &\le \sum_{h\in\dsb{H}}(1+(h-1)/\theta)^d\\
  &\le \cO\Bigr{H(H/\theta)^d}\\
  &\markref{(10)}{=} \cO\Bigr{H(H^2/\epsilon)^d},
\end{align*}
where at (11) we used the previous choice $\theta=\frac{\epsilon}{4H}$.

Replacing into the number of samples in Lemma \ref{lemma: concentration set} (and
also $\epsilon'=\frac{\epsilon}{2H}$) we get the result:
\begin{align*}
    N\le \widetilde{\cO}\biggr{\frac{SH^{4+2d}d^2\ln\frac{1}{\delta}}
{\epsilon^{2+d}}\Bigr{A+ \ln\frac{1}{\delta}}}.
\end{align*}
By using $\widetilde{\cO}$ notation to hide logarithmic terms in
$S,A,H,\frac{1}{\epsilon},\ln\frac{1}{\delta},d$, we get the result.
\end{proof}

\begin{restatable}[Error Propagation]{lemma}{errorpropagationset}
\label{lemma: error propagation set}
Let $\cM$ be any \MDPr and $\cR=\{r^1,\dotsc,r^d\}$ any set of $d\ge1$ rewards.
For any pair of policies $\pi,\pi'\in\Pi^{\text{NM}}$ such that, for all
$h\in\dsb{H}$, $a\in\cA$, $s\in\cS$ and $\omega,\omega'\in\Omega_h$ with
$G(\omega;r^i)=G(\omega';r^i)$ $\forall i\in\dsb{d}$:
\begin{align*}
  \pi(a|\omega,s)=\pi(a|\omega',s)\qquad\wedge\qquad \pi'(a|\omega,s)=\pi'(a|\omega',s),
\end{align*}
it holds that, for any $r^i\in\cR$:
\begin{align*}
  \Big\|\eta^{\pi}_{r^i}-
  \eta^{\pi'}_{r^i}\Big\|_1
  &\le \sum\limits_{h\in\dsb{H}}\sum\limits_{g^1\in \cG_{r^1,h}}\dotsc
  \sum\limits_{g^d\in \cG_{r^d,h}}\sum\limits_{s\in\cS}
  \P^{\pi}(G_h^1=g^1\wedge\dotsc\wedge G_h^d=g^d\wedge s_{h}=s)\\
  &\qquad\qquad
  \Bign{\pi_{h}(\cdot|s,g^1,\dotsc,g^d)-\pi_{h}'(\cdot|s,g^1,\dotsc,g^d)}_1,
\end{align*}
where $\cG_{r,h}\coloneqq\{g\in[0,h-1]\,| \,\exists \omega\in\Omega_h:\,
G(\omega;r)=g\}$ for any reward $r$, $G_h^i\coloneqq\sum_{h'=1}^{h-1}
r_{h'}^i(s_{h'},a_{h'})$ denotes the \emph{random} return at stage $h$ under
reward $r^i$, and $\pi_{h}(\cdot|s,g^1,\dotsc,g^d)$ and
$\pi_{h}'(\cdot|s,g^1,\dotsc,g^d)$ denote the unique probability with which the
policies $\pi$ and $\pi'$ prescribe actions in $s$ at $h$ under any trajectory
$\omega\in\Omega_h$ with $G(\omega;r^i)=g^i$ $\forall i\in\dsb{d}$. 
\end{restatable}
\begin{proof}
  To prove the result, we first demonstrate by induction that, for
  all $h\in\dsb{2,H}$, it holds that:
  \begin{align*}
  &\sum\limits_{g^1\in \cG_{r^1,h}}\dotsc
  \sum\limits_{g^d\in \cG_{r^d,h}}\sum\limits_{s\in\cS}
  \Big|
    \P^{\pi}(G_h^1=g^1\wedge\dotsc\wedge G_h^d=g^d\wedge s_{h}=s)\\
    &\qquad\qquad-
    \P^{\pi'}(G_h^1=g^1\wedge\dotsc\wedge G_h^d=g^d\wedge s_{h}=s)
  \Big|\\
  &\qquad\qquad
  \le \sum\limits_{h'\in\dsb{h-1}}\sum\limits_{g^1\in \cG_{r^1,h'}}\dotsc
  \sum\limits_{g^d\in \cG_{r^d,h'}}\sum\limits_{s\in\cS}
  \P^{\pi}(G_{h'}^1=g^1\wedge\dotsc\wedge G_{h'}^d=g^d\wedge s_{h'}=s)
  \\&\qquad\qquad\qquad\qquad
  \Bign{\pi_{h'}(\cdot|s,g^1,\dotsc,g^d)-\pi_{h'}'(\cdot|s,g^1,\dotsc,g^d)}_1.
\end{align*}
We begin with the base case $h=2$. We can write:
\begin{align*}
  &\sum\limits_{g^1\in \cG_{r^1,2}}\dotsc
  \sum\limits_{g^d\in \cG_{r^d,2}}\sum\limits_{s\in\cS}
  \Big|
    \P^{\pi}(G_2^1=g^1\wedge\dotsc\wedge G_2^d=g^d\wedge s_{2}=s)\\
    &\qquad\qquad-
    \P^{\pi'}(G_2^1=g^1\wedge\dotsc\wedge G_2^d=g^d\wedge s_{2}=s)
  \Big|\\
  &\qquad\qquad\markref{(1)}{=}
  \sum\limits_{g^1\in \cG_{r^1,2}}\dotsc
  \sum\limits_{g^d\in \cG_{r^d,2}}\sum\limits_{s'\in\cS}\Big|
  \sum\limits_{\substack{(s,a)\in\SA:\\r_1^i(s,a)=g^i\forall i}}\Big(
  \com{p(s'|s,a)\indic{s=s_0}\pi(a|s)}\\
  &\qquad\qquad\qquad\qquad-
  {p(s'|s,a)\indic{s=s_0}\pi'(a|s)}\Big)
  \Big|\\
  &\qquad\qquad\markref{(2)}{\le}
  \sum\limits_{g^1\in \cG_{r^1,2}}\dotsc
  \sum\limits_{g^d\in \cG_{r^d,2}}\mathop{\com{\sum}}\limits_{
    \substack{\com{a\in\cA:}\\\com{r_1^i(s_0,a)=g^i\forall i}}}
  \sum\limits_{s'\in\cS}p(s'|\com{s_0},a)\com{\Big|}
  \pi(a|\com{s_0})-\pi'(a|\com{s_0})
  \com{\Big|}\\
  &\qquad\qquad=
  \sum\limits_{g^1\in \cG_{r^1,2}}\dotsc
  \sum\limits_{g^d\in \cG_{r^d,2}}\sum\limits_{\substack{a\in\cA:\\r_1^i(s_0,a)=g^i\forall i}}
  \Big|\pi(a|s_0)-\pi'(a|s_0)\Big|\\
  &\qquad\qquad=
  \com{\sum\limits_{a\in\cA}}
  \Big|\pi(a|s_0)-\pi'(a|s_0)\Big|\\
  &\qquad\qquad=
  \Big\|\pi(\cdot|s_0)-\pi'(\cdot|s_0)\Big\|_1\\
  &\qquad\qquad\markref{(3)}{=}
  \sum\limits_{h'\in\dsb{1}}\sum\limits_{g^1\in \cG_{r^1,h'}}\dotsc
  \sum\limits_{g^d\in \cG_{r^d,h'}}\sum\limits_{s\in\cS}
  \P^{\pi}(G_{h'}^1=g^1\wedge\dotsc\wedge G_{h'}^d=g^d\wedge s_{h'}=s)\\
  &\qquad\qquad\qquad\qquad
  \Bign{\pi_{h'}(\cdot|s,g^1,\dotsc,g^d)-\pi_{h'}'(\cdot|s,g^1,\dotsc,g^d)}_1,
\end{align*}
where at (1) we realize that in $\cM$ the initial state is always $s_0$, and
that the transition model is Markovian and independent of the policy, at (2) we
apply triangle's inequality and keep only $s_0$ because of the indicator, and at
(3) we have simply rewritten the expression in a more convenient way for proving
the result.

Now, let us consider any stage $h\in\dsb{3,H}$. Let us make the inductive
hypothesis that:
  \begin{align*}
  &\sum\limits_{g^1\in \cG_{r^1,h-1}}\dotsc
  \sum\limits_{g^d\in \cG_{r^d,h-1}}\sum\limits_{s\in\cS}
  \Big|
    \P^{\pi}(G_{h-1}^1=g^1\wedge\dotsc\wedge G_{h-1}^d=g^d\wedge s_{h-1}=s)\\
    &\qquad\qquad-
    \P^{\pi'}(G_{h-1}^1=g^1\wedge\dotsc\wedge G_{h-1}^d=g^d\wedge s_{h-1}=s)
  \Big|\\
  &\qquad\qquad
  \le \sum\limits_{h'\in\dsb{h-2}}\sum\limits_{g^1\in \cG_{r^1,h'}}\dotsc
  \sum\limits_{g^d\in \cG_{r^d,h'}}\sum\limits_{s\in\cS}
  \P^{\pi}(G_{h'}^1=g^1\wedge\dotsc\wedge G_{h'}^d=g^d\wedge s_{h'}=s)
  \\&\qquad\qquad\qquad\qquad
  \Bign{\pi_{h'}(\cdot|s,g^1,\dotsc,g^d)-\pi_{h'}'(\cdot|s,g^1,\dotsc,g^d)}_1.
\end{align*}
Then, we can write (we use symbol $\omega_h$ to denote the random trajectory
$(s_1,a_1,\dotsc,s_h,a_h)$ up to stage $h$):
\begin{align*}
&
\sum\limits_{g^1\in \cG_{r^1,h}}\dotsc\sum\limits_{g^d\in \cG_{r^d,h}}
  \sum\limits_{s\in\cS}
  \Big|
    \P^{\pi}(G_{h}^1=g^1\wedge\dotsc\wedge G_{h}^d=g^d\wedge s_{h}=s)\\
    &\qquad\qquad-
    \P^{\pi'}(G_{h}^1=g^1\wedge\dotsc\wedge G_{h}^d=g^d\wedge s_{h}=s)
  \Big|\\
  &\qquad\qquad\markref{(4)}{=}
  \sum\limits_{g^1\in \cG_{r^1,h}}\dotsc\sum\limits_{g^d\in \cG_{r^d,h}}
  \sum\limits_{s'\in\cS}
  \Big|
    \sum\limits_{\overline{g}^1\in \cG_{r^1,h-1}}\dotsc\sum\limits_{\overline{g}^d\in \cG_{r^d,h-1}}
    \sum\limits_{\substack{\omega\in\Omega_{h-1}:\\G(\omega;r^i)=\overline{g}^i\forall i}}
    \sum\limits_{\substack{(s,a)\in\SA:\\r_{h-1}^i(s,a)=g^i-\overline{g}^i\forall i}}
    \Big(\\
  &\qquad\qquad\qquad\qquad
      \P^{\pi}(\com{\omega_{h-2}=\omega\wedge s_{h-1}=s})
      \P^{\pi}(a_{h-1}=a\wedge s_{h}=s'|\com{\omega,s})\\
  &\qquad\qquad\qquad\qquad
      -\P^{\pi'}(\com{\omega_{h-2}=\omega\wedge s_{h-1}=s})
      \P^{\pi'}(a_{h-1}=a\wedge s_{h}=s'|\com{\omega,s})
    \Big)\Big|\\
  &\qquad\qquad\markref{(5)}{=}
  \sum\limits_{g^1\in \cG_{r^1,h}}\dotsc\sum\limits_{g^d\in \cG_{r^d,h}}\sum\limits_{s'\in\cS}
  \Big|
    \sum\limits_{\overline{g}^1\in \cG_{r^1,h-1}}\dotsc\sum\limits_{\overline{g}^d\in \cG_{r^d,h-1}}
    \sum\limits_{\substack{\omega\in\Omega_{h-1}:\\G(\omega;r^i)=\overline{g}^i\forall i}}
    \sum\limits_{\substack{(s,a)\in\SA:\\r_{h-1}^i(s,a)=g^i-\overline{g}^i\forall i}}
    \Big(\\
  &\qquad\qquad\qquad\qquad
      \P^{\pi}(\omega_{h-2}=\omega\wedge s_{h-1}=s)
      \com{\pi(a|\omega,s)p_{h-1}(s'|s,a)}\\
  &\qquad\qquad\qquad\qquad
      -\P^{\pi'}(\omega_{h-2}=\omega\wedge s_{h-1}=s)
      \com{\pi'(a|\omega,s)p_{h-1}(s'|s,a)}
    \Big)\Big|\\
  &\qquad\qquad=
  \sum\limits_{g^1\in \cG_{r^1,h}}\dotsc\sum\limits_{g^d\in \cG_{r^d,h}}\sum\limits_{s'\in\cS}
  \Big|
    \sum\limits_{\overline{g}^1\in \cG_{r^1,h-1}}\dotsc\sum\limits_{\overline{g}^d\in \cG_{r^d,h-1}}
    \sum\limits_{\substack{\omega\in\Omega_{h-1}:\\G(\omega;r^i)=\overline{g}^i\forall i}}
    \sum\limits_{\substack{(s,a)\in\SA:\\r_{h-1}^i(s,a)=g^i-\overline{g}^i\forall i}}
    \\&\qquad\qquad\qquad\qquad
    \com{p_{h-1}(s'|s,a)}\Big(
      \P^{\pi}(\omega_{h-2}=\omega\wedge s_{h-1}=s)
      \pi(a|\omega,s)\\
  &\qquad\qquad\qquad\qquad
      -\P^{\pi'}(\omega_{h-2}=\omega\wedge s_{h-1}=s)
      \pi'(a|\omega,s)\\
  &\qquad\qquad\qquad\qquad
  \com{\pm\P^{\pi}(\omega_{h-2}=\omega\wedge s_{h-1}=s)
      \pi'(a|\omega,s)}
    \Big)\Big|\\
  &\qquad\qquad\markref{(6)}{\le}
  \sum\limits_{g^1\in \cG_{r^1,h}}\dotsc\sum\limits_{g^d\in \cG_{r^d,h}}\sum\limits_{s'\in\cS}
    \sum\limits_{\overline{g}^1\in \cG_{r^1,h-1}}\dotsc\sum\limits_{\overline{g}^d\in \cG_{r^d,h-1}}
    \sum\limits_{\substack{\omega\in\Omega_{h-1}:\\G(\omega;r^i)=\overline{g}^i\forall i}}
    \sum\limits_{\substack{(s,a)\in\SA:\\r_{h-1}^i(s,a)=g^i-\overline{g}^i\forall i}}
    \\&\qquad\qquad\qquad\qquad
    p_{h-1}(s'|s,a)
  \cdot\P^{\pi}(\omega_{h-2}=\omega\wedge s_{h-1}=s)\com{\Big|
      \pi(a|\omega,s)-\pi'(a|\omega,s)\Big|}\\
  &\qquad\qquad\qquad\qquad
  +\sum\limits_{g^1\in \cG_{r^1,h}}\dotsc\sum\limits_{g^d\in \cG_{r^d,h}}\sum\limits_{s'\in\cS}
    \sum\limits_{\overline{g}^1\in \cG_{r^1,h-1}}\dotsc\sum\limits_{\overline{g}^d\in \cG_{r^d,h-1}}
    \\&\qquad\qquad\qquad\qquad
    \com{\Big|}
    \sum\limits_{\substack{\omega\in\Omega_{h-1}:\\G(\omega;r^i)=\overline{g}^i\forall i}}
    \sum\limits_{\substack{(s,a)\in\SA:\\r_{h-1}^i(s,a)=g^i-\overline{g}^i\forall i}}
    p_{h-1}(s'|s,a)\pi'(a|\omega,s)
      \\
  &\qquad\qquad\qquad\qquad
  \cdot\Bigr{\com{\P^{\pi}(\omega_{h-2}=\omega\wedge s_{h-1}=s)-
  \P^{\pi'}(\omega_{h-2}=\omega\wedge s_{h-1}=s)}}\com{\Big|}\\
  &\qquad\qquad\markref{(7)}{\le}
  \sum\limits_{g^1\in \cG_{r^1,h}}\dotsc\sum\limits_{g^d\in \cG_{r^d,h}}
    \sum\limits_{\overline{g}^1\in \cG_{r^1,h-1}}\dotsc\sum\limits_{\overline{g}^d\in \cG_{r^d,h-1}}
    \sum\limits_{\substack{\omega\in\Omega_{h-1}:\\G(\omega;r^i)=\overline{g}^i\forall i}}
    \sum\limits_{\substack{(s,a)\in\SA:\\r_{h-1}^i(s,a)=g^i-\overline{g}^i\forall i}}
    \\
  &\qquad\qquad\qquad\qquad
  \P^{\pi}(\omega_{h-2}=\omega\wedge s_{h-1}=s)\Big|
      \pi(a|\omega,s)-\pi'(a|\omega,s)\Big|\\
  &\qquad\qquad\qquad\qquad
  +\sum\limits_{g^1\in \cG_{r^1,h}}\dotsc\sum\limits_{g^d\in \cG_{r^d,h}}\sum\limits_{s'\in\cS}
    \sum\limits_{\overline{g}^1\in \cG_{r^1,h-1}}\dotsc\sum\limits_{\overline{g}^d\in \cG_{r^d,h-1}}
    \mathop{\com{\sum}}\limits_{\substack{\com{(s,a)\in\SA:}\\\com{r_{h-1}^i(s,a)=g^i-\overline{g}^i\forall i}}}
    \\&\qquad\qquad\qquad\qquad
    \com{p_{h-1}(s'|s,a)}
    \com{\Big|}
    \sum\limits_{\substack{\omega\in\Omega_{h-1}:\\G(\omega;r^i)=\overline{g}^i\forall i}}
    \pi'(a|\omega,s)
      \\
  &\qquad\qquad\qquad\qquad
  \cdot\Bigr{\P^{\pi}(\omega_{h-2}=\omega\wedge s_{h-1}=s)-
  \P^{\pi'}(\omega_{h-2}=\omega\wedge s_{h-1}=s)}\com{\Big|}\\
  &\qquad\qquad\markref{(8)}{=}
    \sum\limits_{\overline{g}^1\in \cG_{r^1,h-1}}\dotsc\sum\limits_{\overline{g}^d\in \cG_{r^d,h-1}}
    \com{\sum\limits_{\substack{(s,a)\in\SA}}\Big|
      \pi_{h-1}(a|s,\overline{g}^1,\dotsc,\overline{g}^d)-\pi_{h-1}'(a|s,\overline{g}^1,\dotsc,\overline{g}^d)\Big|}
    \\
  &\qquad\qquad\qquad\qquad
  \cdot\sum\limits_{\substack{\omega\in\Omega_{h-1}:\\G(\omega;r^i)=\overline{g}^i\forall i}}
  \P^{\pi}(\omega_{h-2}=\omega\wedge s_{h-1}=s)\\
  &\qquad\qquad\qquad\qquad
  +\sum\limits_{g^1\in \cG_{r^1,h}}\dotsc\sum\limits_{g^d\in \cG_{r^d,h}}
    \sum\limits_{\overline{g}^1\in \cG_{r^1,h-1}}\dotsc\sum\limits_{\overline{g}^d\in \cG_{r^d,h-1}}
    \sum\limits_{\substack{(s,a)\in\SA:\\r_{h-1}^i(s,a)=g^i-\overline{g}^i\forall i}}
    \\&\qquad\qquad\qquad\qquad
    \com{\pi_{h-1}'(a|s,\overline{g}^1,\dotsc,\overline{g}^d)}\Big|
    \sum\limits_{\substack{\omega\in\Omega_{h-1}:\\G(\omega;r^i)=\overline{g}^i\forall i}}
      \\
  &\qquad\qquad\qquad\qquad
  \cdot\Bigr{\P^{\pi}(\omega_{h-2}=\omega\wedge s_{h-1}=s)-
  \P^{\pi'}(\omega_{h-2}=\omega\wedge s_{h-1}=s)}\Big|\\
  &\qquad\qquad=
    \sum\limits_{\overline{g}^1\in \cG_{r^1,h-1}}\dotsc\sum\limits_{\overline{g}^d\in \cG_{r^d,h-1}}
    \sum\limits_{\substack{(s,a)\in\SA}}\Big|
      \pi_{h-1}(a|s,\overline{g}^1,\dotsc,\overline{g}^d)-\pi_{h-1}'(a|s,\overline{g}^1,\dotsc,\overline{g}^d)\Big|
      \\&\qquad\qquad\qquad\qquad
  \com{\P^{\pi}(G_{h-1}^1=\overline{g}^1\wedge\dotsc,\wedge G_{h-1}^d=\overline{g}^d\wedge s_{h-1}=s)}\\
  &\qquad\qquad\qquad\qquad
  +\sum\limits_{g^1\in \cG_{r^1,h}}\dotsc\sum\limits_{g^d\in \cG_{r^d,h}}
    \sum\limits_{\overline{g}^1\in \cG_{r^1,h-1}}\dotsc\sum\limits_{\overline{g}^d\in \cG_{r^d,h-1}}
    \sum\limits_{\substack{(s,a)\in\SA:\\r_{h-1}^i(s,a)=g^i-\overline{g}^i\forall i}}
    \\&\qquad\qquad\qquad\qquad
    \pi_{h-1}'(a|s,\overline{g}^1,\dotsc,\overline{g}^d)
  \\&\qquad\qquad\qquad\qquad
  \Big|\com{\P^{\pi}(G_{h-1}^1=\overline{g}^1\wedge\dotsc,\wedge G_{h-1}^d=\overline{g}^d\wedge s_{h-1}=s)}
  \\&\qquad\qquad\qquad\qquad
  -
  \com{\P^{\pi'}(G_{h-1}^1=\overline{g}^1\wedge\dotsc,\wedge G_{h-1}^d=\overline{g}^d\wedge s_{h-1}=s)}\Big|\\
  &\qquad\qquad\markref{(9)}{=}
    \sum\limits_{\overline{g}^1\in \cG_{r^1,h-1}}\dotsc\sum\limits_{\overline{g}^d\in \cG_{r^d,h-1}}
    \sum\limits_{\com{s\in\cS}}
    \P^{\pi}(G_{h-1}^1=\overline{g}^1\wedge\dotsc,\wedge G_{h-1}^d=\overline{g}^d\wedge s_{h-1}=s)
    \\&\qquad\qquad\qquad\qquad
    \com{\Big\|
      \pi_{h-1}(\cdot|s,\overline{g}^1,\dotsc,\overline{g}^d)-
      \pi_{h-1}'(\cdot|s,\overline{g}^1,\dotsc,\overline{g}^d)\Big\|_1} \\
  &\qquad\qquad\qquad\qquad
  +\sum\limits_{\overline{g}^1\in \cG_{r^1,h-1}}\dotsc\sum\limits_{\overline{g}^d\in \cG_{r^d,h-1}}
  \sum\limits_{\com{s\in\cS}}\\
  &\qquad\qquad\qquad\qquad
  \Big|\P^{\pi}(G_{h-1}^1=\overline{g}^1\wedge\dotsc,\wedge G_{h-1}^d=\overline{g}^d\wedge s_{h-1}=s)
  \\&\qquad\qquad\qquad\qquad
  -
  \P^{\pi'}(G_{h-1}^1=\overline{g}^1\wedge\dotsc,\wedge G_{h-1}^d=\overline{g}^d\wedge s_{h-1}=s)\Big|\\
  &\qquad\qquad\markref{(10)}{\le}
  \sum\limits_{\overline{g}^1\in \cG_{r^1,h-1}}\dotsc\sum\limits_{\overline{g}^d\in \cG_{r^d,h-1}}\sum\limits_{s\in\cS}
    \P^{\pi}(G_{h-1}^1=\overline{g}^1\wedge\dotsc,\wedge G_{h-1}^d=\overline{g}^d\wedge s_{h-1}=s)
    \\&\qquad\qquad\qquad\qquad
    \Big\|
      \pi_{h-1}(\cdot|s,\overline{g}^1,\dotsc,\overline{g}^d)-\pi_{h-1}'(\cdot|s,\overline{g}^1,\dotsc,\overline{g}^d)\Big\|_1
    \\
  &\qquad\qquad\qquad\qquad  
  +\com{\sum\limits_{h'\in\dsb{h-2}}\sum\limits_{g^1\in \cG_{r^1,h'}}\dotsc
  \sum\limits_{g^d\in \cG_{r^d,h'}}\sum\limits_{s\in\cS}}
  \\&\qquad\qquad\qquad\qquad
  \com{\P^{\pi}(G_{h'}^1=g^1\wedge\dotsc\wedge G_{h'}^d=g^d\wedge s_{h'}=s)}
  \\&\qquad\qquad\qquad\qquad
\com{\Bign{\pi_{h'}(\cdot|s,g^1,\dotsc,g^d)-\pi_{h'}'(\cdot|s,g^1,\dotsc,g^d)}_1}\\
  &\qquad\qquad=
  {\sum\limits_{h'\in\dsb{\com{h-1}}}\sum\limits_{g^1\in \cG_{r^1,h'}}\dotsc
  \sum\limits_{g^d\in \cG_{r^d,h'}}\sum\limits_{s\in\cS}}
  \\&\qquad\qquad\qquad\qquad
  {\P^{\pi}(G_{h'}^1=g^1\wedge\dotsc\wedge G_{h'}^d=g^d\wedge s_{h'}=s)}
  \\&\qquad\qquad\qquad\qquad
{\Bign{\pi_{h'}(\cdot|s,g^1,\dotsc,g^d)-\pi_{h'}'(\cdot|s,g^1,\dotsc,g^d)}_1},
\end{align*}
where at (4) we use the chain rule of conditional probabilities, at (5) we do it
again, and we recognize the policies $\pi$ and $\pi'$, and also that the
transition model is Markovian, at (6) we use triangle's inequality to split the
summations and bring the absolute value inside, at (7), in the first term, we
note that $p_{h-1}(s'|s,a)$ is the only term that depends on $s'$ and that it
sums to 1, while in the second term we exchange the order of two summations and
apply triangle's inequality to bring one inside, at (8), in the first term, we
first remove the summation on $g^i$ along with the indicator function that
forces us to consider a subset of state-action pairs, and then we exchange two
other summations and note that the policies do not depend by hypothesis on the
entire past trajectory, but just on the return so far for any $r^i\in\cR$.
Instead, in the second term, we use that $\sum_{s'\in\cS}p_{h-1}(s'|s,a)=1$, and
also that, by hypothesis, $\pi'$ does not depend on the entire past trajectory,
but just on $g^1,\dotsc,g^d$. At (9), i.a., we use that $\sum_{g^1\in
\cG_{r^1,h}}\dotsc\sum_{g^d\in
\cG_{r^d,h}}\indic{r_{h-1}^i(s,a)=g^i-\overline{g}^i\forall i}=1$ and that
$\sum_{a\in\cA}\pi_{h-1}'(a|s,\overline{g}^1,\dotsc,\overline{g}^d)=1$. Finally,
at (10), we apply the inductive hypothesis.

Thanks to this result, we can finally prove the claim in the lemma, using
passages analogous to those above, with the difference that we do not have the
summation over the states at the current stage (i.e., $H+1$). For any $r^j\in\cR$:
\begin{align*}
  &\Big\|\eta^{\pi}_{r^j}-
  \eta^{\pi'}_{r^j}\Big\|_1
  =\sum\limits_{g\in \cG_{r^j,H+1}}
  \Big|\eta^{\pi}_{r^j}(g)-\eta^{\pi'}_{r^j}(g)\Big|\\
  &\qquad\qquad=\sum\limits_{g\in \cG_{r^j,H+1}}
  \Big|\com{\P^{\pi}(G_{H+1}^j=g)}-\com{\P^{\pi'}(G_{H+1}^j=g)}\Big|\\
    &\qquad\qquad=
  \sum\limits_{g\in \cG_{r^j,H+1}}
  \Big|
    \sum\limits_{g^1\in \cG_{r^1,H}}\dotsc\sum\limits_{g^d\in \cG_{r^d,H}}
    \sum\limits_{\substack{\omega\in\Omega_{H}:\\G(\omega;r^i)=g^i\forall i}}
    \sum\limits_{\substack{(s,a)\in\SA:\\r_{H}^j(s,a)=g-g^j}}
    \\&\qquad\qquad\qquad\qquad
    \Big(
      \P^{\pi}(\omega_{H-1}=\omega\wedge s_{H}=s)
      \pi(a|\omega,s)\\
  &\qquad\qquad\qquad\qquad
      -\P^{\pi'}(\omega_{H-1}=\omega\wedge s_{H}=s)
      \pi'(a|\omega,s)\\
  &\qquad\qquad\qquad\qquad
  \com{\pm\P^{\pi}(\omega_{H-1}=\omega\wedge s_{H}=s)
      \pi'(a|\omega,s)}
    \Big)\Big|\\
  &\qquad\qquad\markref{(11)}{\le}
  \sum\limits_{g^1\in \cG_{r^1,H}}\dotsc\sum\limits_{g^d\in \cG_{r^d,H}}\sum\limits_{s\in\cS}
    \P^{\pi}(G_{H}^1=g^1\wedge\dotsc,\wedge G_{H}^d=g^d\wedge s_{H}=s)
    \\&\qquad\qquad\qquad\qquad
    \Big\|
      \pi_{H}(\cdot|s,g^1,\dotsc,g^d)-\pi_{H}'(\cdot|s,g^1,\dotsc,g^d)\Big\|_1
    \\
  &\qquad\qquad\qquad\qquad  
  +{\sum\limits_{h'\in\dsb{H-1}}\sum\limits_{g^1\in \cG_{r^1,h'}}\dotsc
  \sum\limits_{g^d\in \cG_{r^d,h'}}\sum\limits_{s\in\cS}}
  \\&\qquad\qquad\qquad\qquad
  {\P^{\pi}(G_{h'}^1=g^1\wedge\dotsc\wedge G_{h'}^d=g^d\wedge s_{h'}=s)}
  \\&\qquad\qquad\qquad\qquad
{\Bign{\pi_{h'}(\cdot|s,g^1,\dotsc,g^d)-\pi_{h'}'(\cdot|s,g^1,\dotsc,g^d)}_1},
\end{align*}
where at (11) we made the same passages as above from (6) on, all in one, with
the only differences that we sum only over $g\in \cG_{r^j,H+1}$ instead of doing
so for any reward of $\cR$, and also that we do not have the sum over the next
state $s'$. The result follows by summing the two terms in the last expression
written.

This concludes the proof.
\end{proof}

\begin{restatable}[Concentration]{lemma}{concentrationset}
\label{lemma: concentration set}
Let $\epsilon\in(0,H]$ and $\delta\in(0,1)$.
Let $\cM$ be any \MDPr and $\cR=\{r^1,\dotsc,r^d\}$ any set of $d\ge1$ rewards,
$\pi^E\in\Pi^{\text{NM}}$ be any expert's policy, and $\widehat{\pi}$ be the
output of Algorithm \ref{alg: rsbc set}.
Then, with probability $1-\delta$, we have that (we use the notation in Lemma
\ref{lemma: error propagation set}):
\begin{align*}
\sum\limits_{h\in\dsb{H}}\sum\limits_{g^1\in \cG_{r^1_\theta,h}}\dotsc
  \sum\limits_{g^d\in \cG_{r^d_\theta,h}}\sum\limits_{s\in\cS}
  \P^{\pi^E}&(G_h^1=g^1\wedge\dotsc\wedge G_h^d=g^d\wedge s_{h}=s)\\
  &
  \Bign{\pi_{\cR^\theta,h}(\cdot|s,g^1,\dotsc,g^d)-\widehat{\pi}_{h}(\cdot|s,g^1,\dotsc,g^d)}_1\le\epsilon,
\end{align*}
with a number of samples:
\begin{align*}
  N\le \frac{193SH\overline{\cG}\ln\frac{2S\overline{\cG}}{\delta}}
{\epsilon^2}\biggr{
  \ln\frac{2S\overline{\cG}}{\delta}+(A-1)
  \ln\Bigr{\frac{128eSH\overline{\cG}\ln\frac{2S\overline{\cG}}{\delta}}
{\epsilon^2}}
},
\end{align*}
where $\overline{\cG}\coloneqq
  \sum_{h\in\dsb{H}}\prod_{i\in\dsb{d}}|\cG_{r^i_\theta,h}|$.
\end{restatable}
\begin{proof}
  We can write:
  \begin{align*}
    &  \sum\limits_{h\in\dsb{H}}
    \sum\limits_{g^1\in \cG_{r^1_\theta,h}}\dotsc\sum\limits_{g^d\in \cG_{r^d_\theta,h}}
    \sum\limits_{s\in\cS}
  \P^{\pi^E}(G_h^1=g^1\wedge\dotsc\wedge G_h^d=g^d\wedge s_{h}=s)\\
  &\qquad\qquad
  \Bign{\pi_{\cR^\theta,h}(\cdot|s,g^1,\dotsc,g^d)-\widehat{\pi}_{h}(\cdot|s,g^1,\dotsc,g^d)}_1\\
  &\qquad\qquad=
  \sum\limits_{h\in\dsb{H}}\sum\limits_{g^1\in \cG_{r^1_\theta,h}}\dotsc\sum\limits_{g^d\in \cG_{r^d_\theta,h}}
  \sum\limits_{s\in\cS}
  \P^{\pi^E}(G_h^1=g^1\wedge\dotsc\wedge G_h^d=g^d\wedge s_{h}=s)
  \\&\qquad\qquad\qquad\qquad
  \sum\limits_{a\in\cA}\Big|\pi_{\cR^\theta,h}(a|s,g^1,\dotsc,g^d)
  -\Big(
    \frac{M_h(s,g^1,\dotsc,g^d,a)}{\sum_{a'}M_h(s,g^1,\dotsc,g^d,a')}
  \\&\qquad\qquad\qquad\qquad
    \cdot \indic{\sum_{a'}M_h(s,g^1,\dotsc,g^d,a')>0}
    +\frac{1}{A}\indic{\sum_{a'}M_h(s,g^1,\dotsc,g^d,a')=0}
  \Big)\Big|\\
  &\qquad\qquad\markref{(1)}{\le}
  \sum\limits_{h\in\dsb{H}}\sum\limits_{g^1\in \cG_{r^1_\theta,h}}\dotsc\sum\limits_{g^d\in \cG_{r^d_\theta,h}}
  \sum\limits_{s\in\cS}
  \P^{\pi^E}(G_h^1=g^1\wedge\dotsc\wedge G_h^d=g^d\wedge s_{h}=s)
  \\
  &\qquad\qquad\qquad\qquad
  \cdot \com{2\sqrt{2}\sqrt{\frac{\ln\frac{2S\overline{\cG}}{\delta}+
  (A-1)\ln\bigr{e\bigr{1+\frac{\sum_{a'}M_h(s,g^1,\dotsc,g^d,a')}{A-1}}}}{\sum_{a'}M_h(s,g^1,\dotsc,g^d,a')}}}
  \\
  &\qquad\qquad\markref{(2)}{\le}
  2\sqrt{2}\sqrt{\ln\frac{2S\overline{\cG}}{\delta}+
  (A-1)\ln\Bigr{e\Bigr{1+\frac{\com{N}}{A-1}}}}\\
  &\qquad\qquad\qquad\qquad
  \cdot \sum\limits_{h\in\dsb{H}}\sum\limits_{g^1\in \cG_{r^1_\theta,h}}\dotsc\sum\limits_{g^d\in \cG_{r^d_\theta,h}}
  \sum\limits_{s\in\cS}
  \P^{\pi^E}(G_h^1=g^1\wedge\dotsc\wedge G_h^d=g^d\wedge s_{h}=s)
  \\&\qquad\qquad\qquad\qquad
  \sqrt{\frac{1}{\sum_{a'}M_h(s,g^1,\dotsc,g^d,a')}}
  \\
  &\qquad\qquad\markref{(3)}{\le}
  2\sqrt{2}\sqrt{\ln\frac{2S\overline{\cG}}{\delta}+
  (A-1)\ln\Bigr{e\Bigr{1+\frac{N}{A-1}}}}\\
  &\qquad\qquad\qquad\qquad
  \cdot \sum\limits_{h\in\dsb{H}}\sum\limits_{g^1\in \cG_{r^1_\theta,h}}\dotsc\sum\limits_{g^d\in \cG_{r^d_\theta,h}}
  \sum\limits_{s\in\cS}
  \\&\qquad\qquad\qquad\qquad
  \sqrt{\com{\frac{8\ln\frac{2S\overline{\cG}}{\delta}\P^{\pi^E}(G_h^1=g^1\wedge\dotsc\wedge G_h^d=g^d\wedge s_{h}=s)}{N}}}
  \\
  &\qquad\qquad=
  \com{8\sqrt{\frac{\ln\frac{2S\overline{\cG}}{\delta}}{N}}}
  \sqrt{\ln\frac{2S\overline{\cG}}{\delta}+
  (A-1)\ln\Bigr{e\Bigr{1+\frac{N}{A-1}}}}\\
  &\qquad\qquad\qquad\qquad
  \cdot \sum\limits_{h\in\dsb{H}}\sum\limits_{g^1\in \cG_{r^1_\theta,h}}\dotsc\sum\limits_{g^d\in \cG_{r^d_\theta,h}}
  \sum\limits_{s\in\cS}
  \\&\qquad\qquad\qquad\qquad
  \sqrt{\com{\P^{\pi^E}(G_h^1=g^1\wedge\dotsc\wedge G_h^d=g^d\wedge s_{h}=s)}}
  \\
  &\qquad\qquad\markref{(4)}{\le}
  8\sqrt{\frac{\ln\frac{2S\overline{\cG}}{\delta}}{N}}
  \sqrt{\ln\frac{2S\overline{\cG}}{\delta}+
  (A-1)\ln\Bigr{e\Bigr{1+\frac{N}{A-1}}}}
  \com{\sqrt{S\overline{\cG}}}
  \\&\qquad\qquad\qquad\qquad
  \sqrt{\com{\sum\limits_{h\in\dsb{H}}\sum\limits_{g^1\in \cG_{r^1_\theta,h}}\dotsc\sum\limits_{g^d\in \cG_{r^d_\theta,h}}
  \sum\limits_{s\in\cS}}
  \P^{\pi^E}(G_h^1=g^1\wedge\dotsc\wedge G_h^d=g^d\wedge s_{h}=s)}
  \\
  &\qquad\qquad=
  8\sqrt{\frac{\com{S\overline{\cG}}\ln\frac{2S\overline{\cG}}{\delta}}{N}}
  \sqrt{\ln\frac{2S\overline{\cG}}{\delta}+
  (A-1)\ln\Bigr{e\Bigr{1+\frac{N}{A-1}}}}
  \sqrt{\sum\limits_{h\in\dsb{H}}\com{1}}
  \\
  &\qquad\qquad\markref{(5)}{\le}
  8\sqrt{\frac{S\com{H\overline{\cG}}\ln\frac{2S\overline{\cG}}{\delta}}{N}}
  \sqrt{\ln\frac{2S\overline{\cG}}{\delta}+
  (A-1)\ln\Bigr{e\Bigr{1+\frac{N}{A-1}}}},
  \end{align*}
  where at (1) we use that, if $\sum_{a'}M_h(s,g^1,\dotsc,g^d,a')=0$, then:
  \begin{align*}
    &\sum\limits_{a\in\cA}\Big|\pi_{\cR^\theta,h}(a|s,g^1,\dotsc,g^d)-\Bigr{
    \frac{M_{h}(s,g,a)}{\sum_{a'}M_h(s,g^1,\dotsc,g^d,a')}\indic{\sum_{a'}M_h(s,g^1,\dotsc,g^d,a')>0}
  \\&\qquad\qquad\qquad\qquad
    +\frac{1}{A}\indic{\sum_{a'}M_h(s,g^1,\dotsc,g^d,a')=0}
  }\Big|\\
  &\qquad\qquad=\sum\limits_{a\in\cA}\Big|\pi_{\cR^\theta,h}(a|s,g^1,\dotsc,g^d)-\frac{1}{A}\Big|\\
  &\qquad\qquad\le2,
  \end{align*}
  as we the total variation distance between two probability distributions
  cannot exceed 1. Instead, if $\sum_{a'}M_h(s,g^1,\dotsc,g^d,a')>0$, \emph{conditioning}
  on $\sum_{a'}M_h(s,g^1,\dotsc,g^d,a')$, at all $s,g^1,\dotsc,g^d$ where $\P^{\pi^E}(G_h=g \wedge
  s_{h}=s)>0$, we note that $M_{h}(s,g^1,\dotsc,g^d,a)/\sum_{a'}M_h(s,g^1,\dotsc,g^d,a')$ is the
  empirical vector of probabilities of $\pi_{\cR^\theta,h}(a|s,g^1,\dotsc,g^d)$ (recall its
  definition from Eq. \ref{eq: def policy imitate same occ meas set}), thus we can
  apply Lemma 8 of \citet{kaufmann2021adaptive} to get that, for any
  $\delta\in(0,1)$:
  \begin{align*}
    &\P^{\pi^E}\Big(
      KL\Bigr{\frac{M_{h}(s,g^1,\dotsc,g^d,\cdot)}{\sum_{a'}M_h(s,g^1,\dotsc,g^d,a')}\Big\| \pi_{\cR^\theta,h}(\cdot|s,g^1,\dotsc,g^d)}\\
     &\qquad\qquad \le
      \frac{\ln\frac{1}{\delta}+ (A-1)\ln\bigr{e\bigr{1+\frac{\sum_{a'}M_h(s,g^1,\dotsc,g^d,a')}{A-1}}}}{\sum_{a'}M_h(s,g^1,\dotsc,g^d,a')}
    \Big)\ge 1-\delta.
  \end{align*}
  Combining this result with the Pinsker's inequality, that tells us that
  $\|x-y\|_1\le \sqrt{2 KL(x\| y)}$, and with a union bound over all
  $h\in\dsb{H}$, $s\in\cS$,
  $g^1\in\cG_{r^1_\theta,h},\dotsc,g^d\in\cG_{r^d_\theta,h}$, we get the passage
  in (1) w.p. $1-\delta/2$. Note that we add an additional 2 for the case
  $\sum_{a'}M_h(s,g^1,\dotsc,g^d,a')=0$, and we define $\overline{\cG}\coloneqq
  \sum_{h\in\dsb{H}}\prod_{i\in\dsb{d}}|\cG_{r^i_\theta,h}|$.
  At (2) we bound $\sum_{a'}M_h(s,g^1,\dotsc,g^d,a')\le N$, and bring that
  quantity outside, at (3) we apply Lemma A.1 of \citet{xie2021bridging}, after
  having noticed that $\sum_{a'}M_h(s,g^1,\dotsc,g^d,a')\sim\text{Bin}\Bigr{ N,
  \P^{\pi^E}(G_h^1=g^1\wedge\dotsc\wedge G_h^d=g^d\wedge s_{h}=s)}$, and make it
  hold for all $s,g^1,\dotsc,g^d,h$ w.p. $1-\delta/2$.
  At (4) and (5) we apply the Cauchy-Schwarz's inequality.

  Now, we impose that this quantity is smaller than $\epsilon$:
  \begin{align*}
    &8\sqrt{\frac{SH\overline{\cG}\ln\frac{2S\overline{\cG}}{\delta}}{N}}
  \sqrt{\ln\frac{2S\overline{\cG}}{\delta}+
  (A-1)\ln\Bigr{e\Bigr{1+\frac{N}{A-1}}}}\le\epsilon\\
  &\qquad\qquad\iff
N\ge \frac{64SH\overline{\cG}\ln^2\frac{2S\overline{\cG}}{\delta}}
{\epsilon^2}
+
\frac{64SH\overline{\cG}(A-1)\ln\frac{2S\overline{\cG}}{\delta}}
{\epsilon^2}
\ln\Bigr{\frac{eN}{A-1}+e}.
  \end{align*}
Thanks to Lemma J.3 of \citet{lazzati2024offline}, we know that this inequality
is satisfied with:
\begin{align*}
  N\le \frac{128SH\overline{\cG}\ln^2\frac{2S\overline{\cG}}{\delta}}
{\epsilon^2}
+ \frac{192SH\overline{\cG}(A-1)\ln\frac{2S\overline{\cG}}{\delta}}
{\epsilon^2}
\ln\Bigr{\frac{128eSH\overline{\cG}\ln\frac{2S\overline{\cG}}{\delta}}
{\epsilon^2}}
+A-1.
\end{align*}
Rearranging and applying a final union bound concludes the proof.
\end{proof}

\subsubsection{Proof of Theorem \ref{thr: rskt set}}
\label{apx: proofs set rewards rskt}

\rsktupperboundset*
\begin{proof}
Observe that, in the same way as in the proof of Theorem \ref{thr: rskt} or
\ref{thr: est any policy all rewards}, it is simple to see that:
\begin{align*}
  \max\limits_{r\in\cR} \cW\Bigr{
      \eta^{\pi^E}_{r},\widehat{\eta}_r
  }\le H\theta/2+H\epsilon',
\end{align*}
with probability $1-\delta$, by using:
\begin{align*}
  N\le\frac{1}{2(\epsilon')^2}\ln\frac{2d}{\delta},
\end{align*}
data, where we made a union bound over the $d$ rewards in $\cR^\theta$.
Then, conditioning on this event and proceeding as in the proof of Theorem
\ref{thr: upper bound exp compl}, we have:
\begin{align*}
  \cW\Bigr{
      \eta^{\pi^E}_{r^E},\eta^{\widehat{\pi}}_{r^E}
  }&\le
  \max\limits_{r\in\cR}\cW\Bigr{
      \eta^{\pi^E}_r,\eta^{\widehat{\pi}}_r}\\
    &\markref{(1)}{\le}
    \max\limits_{r\in\cR}\cW\Bigr{
      \eta^{\pi^E}_r,\com{\widehat{\eta}_r}
    }
    +
    \max\limits_{r\in\cR}
    \cW\Bigr{
      \com{\widehat{\eta}_r},\com{\eta^{\widehat{\pi}}_{r_\theta}}
    }
    +
    \max\limits_{r\in\cR}
    \cW\Bigr{
      \com{\eta^{\widehat{\pi}}_{r_\theta}},\eta^{\widehat{\pi}}_r
    }\\
    &\markref{(2)}{\le}
    \com{H\epsilon'+H\theta}
    +
    \max\limits_{r\in\cR}
    \cW\Bigr{
      \widehat{\eta}_r,\eta^{\widehat{\pi}}_{r_\theta}
    }\\
    &\markref{(3)}{=}
    H\epsilon'+H\theta
    +
    \com{\min\limits_{\pi\in\Pi(\cR^\theta)}}\max\limits_{r\in\cR}
    \cW\Bigr{
      \widehat{\eta}_r,\com{\eta^{\pi}_{r_\theta}}
    }\\
    &\markref{(4)}{\le}
    H\epsilon'+H\theta
    +
    \max\limits_{r\in\cR}
    \cW\Bigr{
      \widehat{\eta}_r,\com{\eta^{\pi^E}_{r_\theta}}
    }\\
    &\markref{(5)}{\le}
    \com{2H\epsilon'+2H\theta},
\end{align*}
where at (1) we use triangle's inequality, at (2) we use that event above holds
and Lemma \ref{lemma: different r same p}, at (3) we use the definition of
$\widehat{\pi}$, at (4) we upper bound the minimum with a specific choice of
reward in $\Pi^{\text{NM}}$, i.e., $\pi^E$, and finally, at (5), we apply again
that event holds and Lemma \ref{lemma: different r same p}.

  If we now choose $\theta=\epsilon/(4H)$, and $\epsilon'=\epsilon/(4H)$, we get that,
with probability $1-\delta$, the claim of the theorem holds with data:
\begin{align*}
  N\le\frac{8H^2}{\epsilon^2}\ln\frac{2d}{\delta}.
\end{align*}
\end{proof}

\subsection{When $r^E$ is linear in a known feature map}\label{apx: rE
linear}

In this appendix, we consider a variant of the known-reward setting, in which
$r^E$ is unknown, but we have knowledge of a $d-$dimensional feature map
$\phi:\SA\to[-1,+1]^d$ such that the expert reward $r^E$ can be written as:
\begin{align*}
  r^E_h(s,a)=\dotp{\phi(s,a),w_h}=\sum\limits_{i\in\dsb{d}}\phi_i(s,a) w_{h,i},
\end{align*}
for some unknown vectors $w_h\in[-1,+1]^d$. Note that the linear reward setting
is common in the IL literature \citep{abbeel2004apprenticeship}.
We consider the following robust variant of RDM for this setting:
\begin{align}\label{eq: RDM variant linear}
    \widehat{\pi}\in\argmin_{\pi\in\Pi^{\text{NM}}}\max\limits_{w:\dsb{H}\to[-1,+1]^d}
  \cW\Bigr{\eta_{\phi w}^{\pi},\eta_{\phi w}^{\pi^E}},
\end{align}
where notation $\phi w$ denotes the reward obtained through the dot product
between $\phi$ and $w$.

We now sketch how this setting can be easily addressed through the technique
presented in Appendix \ref{apx: rE in finite set}.

First, consider each \emph{known} feature $\phi^i:\SA\to[-1,+1]$ as a reward
function, and define the set of ``rewards''
$\Phi\coloneqq\{\phi^1,\dotsc,\phi^d\}$. Then, consider the set of policies
$\Pi(\Phi)$ using definition in Appendix \ref{apx: rE in finite set}, and let
$\pi_{\Phi}$ be the policy defined as in Eq. \eqref{eq: def policy imitate same
occ meas set}. Then, from Lemma \ref{lemma: same return distribution
set},\footnote{Modulo some slight difference as rewards are in $[0,1]$ but
features are in $[-1,+1]$.} we have the guarantee that:
\begin{align*}
  \P^{\pi_\Phi}\Big(\sum\limits_{h=1}^H \phi^i(s_h,a_h)=g\Big)=
  \P^{\pi^E}\Big(\sum\limits_{h=1}^H \phi^i(s_h,a_h)=g\Big)
  \qquad\forall g\in[-1,+1],\forall i\in\dsb{d}.
\end{align*}
As a consequence, we have that, for any $w_h\in[-1,+1]^d$, it holds that:
\begin{align*}
  \eta^{\pi_\Phi}_{\phi w}(g)=\eta^{\pi^E}_{\phi w}(g) \qquad\forall g,
\end{align*}
since the cumulative feature map collected is the same.
Simply put, this means that set $\Pi(\Phi)$ (and also policy $\pi_\Phi$) suffice
for this ``linear'' variant of the RDM problem. However, as the feature map is
arbitrary, $\Pi(\Phi)$ might be too large. Thus, we may want to discretize.

Extend the discretization approach of Section \ref{sec: policy class} to
``rewards'' in $[-1,+1]$ and define set
$\Phi^\theta\coloneqq\{\phi^1_\theta,\dotsc,\phi^d_\theta\}$.
Then, observe that the policy $\pi_{\Phi^\theta}$ satisfies a variant of Lemma
\ref{lemma: apx policies set}:
\begin{align*}
  &\max\limits_{w:\dsb{H}\to[-1,+1]^d}
  \cW\Bigr{\eta_{\phi w}^{\pi_{\Phi^\theta}},\eta_{\phi w}^{\pi^E}}\\
  &\qquad\qquad\markref{(1)}{\le}  \max\limits_{w:\dsb{H}\to[-1,+1]^d}
  \biggr{
  \cW\Bigr{\eta_{\phi w}^{\pi_{\Phi^\theta}},\eta_{\phi_\theta w}^{\pi_{\Phi^\theta}}}
  +
  \cW\Bigr{\eta_{\phi_\theta w}^{\pi_{\Phi^\theta}},\eta_{\phi_\theta w}^{\pi^E}}
  +
  \cW\Bigr{\eta_{\phi_\theta w}^{\pi^E},\eta_{\phi_ w}^{\pi^E}}
    }\\
  &\qquad\qquad\markref{(2)}{\le}  
   2dH^2\theta
  +\max\limits_{w:\dsb{H}\to[-1,+1]^d}
  \cW\Bigr{\eta_{\phi_\theta w}^{\pi_{\Phi^\theta}},\eta_{\phi_\theta w}^{\pi^E}}\\
  &\qquad\qquad\markref{(3)}{=}  
   2dH^2\theta,
\end{align*}
where at (1) we use triangle's inequality and denote $\phi_\theta$ the feature
map which is discretized in each dimension, at (2) we apply Lemma \ref{lemma:
different r same p} twice\footnote{We upper bound with an additional factor of 2
to keep into account that now rewards are in $[-1,+1]$.} and observe that,
using Holder's inequality and the definition of discretization: $\max_{w}\|\phi
w-\phi_\theta w\|_\infty\le \max_\phi \|\phi-\phi_\theta\|_1\le d \max_\phi
\|\phi-\phi_\theta\|_\infty \le dH\theta/2$, and at (3) we use the
aforementioned property (i.e., Lemma \ref{lemma: same return distribution set}).

Therefore, the policy $\pi_{\Phi^\theta}$, and so the set $\Pi(\Phi^\theta)$,
suffice for this new robust RDM problem. It is immediate to extend \rsbc and
\rskt to this setting by simply extending Algorithms \ref{alg: rsbc set} and
\ref{alg: rskt set} using input rewards taking values in $[-1,+1]$, and we are
done. Then, by adjusting Theorems \ref{thr: rsbc set} and \ref{thr: rskt set} to
keep track of this small variation, we can have also theoretical guarantees.
Specifically, for the number of samples in Eqs. \eqref{eq: sample complexity
rsbc set} and \eqref{eq: sample complexity rskt set}, we can guarantee that the
policy output by the newly constructed algorithms has a return distribution
close to that of $\eta_{\phi w}^{\pi_{\Phi^\theta}}$ for any $w$, which in turn
has a return distribution close to that of the expert for any $w$ as shown
above. To do this for the variant of \rsbc, note that we just need to extend
the proof of Lemma \ref{lemma: error propagation set}. The crucial insight in
doing so is that both $\pi_{\Phi^\theta}$ and our estimate $\widehat{\pi}$
play actions with same probability at all trajectories with the same cumulative
discretized features. Regarding the variant of \rskt, the proof is immediate as
we just need a union bound over all the features.

\subsection{Generalization to Arbitrary Problems}
\label{apx: gen bc arb problems}

In this appendix, we sketch how to extend \rsbc and its analysis to arbitrary IL
problems of the following kind, in which we aim to find a policy $\widehat{\pi}$
that minimizes:
\begin{align*}
  \widehat{\pi}\in\argmin\limits_{\pi\in\Pi^{\text{NM}}}
  \sum\limits_{h\in\dsb{H}}\sum\limits_{x\in\cX_h}\Biga{\P^\pi(\omega_h\in x)-
  \P^{\pi^E}(\omega_h\in x)},
\end{align*}
where $\cX=\{\cX_h\}_h$ is any partition of the set of trajectories satisfying a
certain property, specifically:
\begin{align*}
  &\bigcup\limits_{x\in\cX_h} x = \Omega_{h+1}\qquad\forall h\in\dsb{H},\\
  &x\cap x'=\{\} \qquad\forall x,x'\in\cX_h,\forall h\in\dsb{H},\\
  &\forall x\in\cX_h,\; \forall (s,a)\in\SA,\;
  \exists x'\in\cX_{h+1},\; \forall \omega\in x:\;
  \omega\cdot s\cdot a \in x' \;\forall h\in\dsb{H-1},
\end{align*}
where $\omega\cdot s\cdot a$ denotes the trajectory obtained by concatenating
$\omega$ with $s,a$.

First, define:
\begin{align}\label{eq: def policy imitate same occ meas any}
    \pi_\cX(a|s,\omega)\coloneqq
      \frac{\P^{\pi^E}(s_h=s,\;a_h=a,\;\omega_{h-1}\in\cX_{h-1}(\omega))}{
        \P^{\pi^E}(s_h=s,\;\omega_{h-1}\in\cX_{h-1}(\omega))},
\end{align}
when the denominator is not 0, and $1/A$ otherwise, and $\cX_{h'}(\cdot)$
denotes the set of trajectories to which $\cdot$ belongs. Then, this policy
satisfies:
\begin{thr}
It holds that:
\begin{align*}
  \P^{\pi_\cX}(s_h=s\wedge \omega_{h-1}\in x)=\P^{\pi^E}(s_h=s\wedge \omega_{h-1}\in x) \qquad
  \forall h\in\dsb{H+1},s\in\cS,x\in\cX_{h-1}.
\end{align*}
\end{thr}
\begin{proof}
The proof follows that of Lemma \ref{lemma: Psg equal Psg} and \ref{lemma: Psg
equal Psg set}. Specifically, we prove it by induction. At $h=1$, note that it
trivially holds as $\cX_0=\emptyset$. Now, make the induction hypothesis that at
any $h'<h$, we have:
\begin{align*}
  \P^{\pi_\cX}(s_{h'}=s\wedge \omega_{h'-1}\in x)=\P^{\pi^E}(s_{h'}=s\wedge \omega_{h'-1}\in x) \qquad
  s\in\cS,x\in\cX_{h'-1}.
\end{align*}
Then, at $h$, for any $s'\in\cS$ and $x'\in\cX_{h-1}$, we can write:
\begin{align*}
  \P^{\pi_\cX}(s_h=s'\wedge \omega_{h-1}\in x')&=\sum\limits_{s,a,\omega}\indic{\omega\cdot s\cdot a\in x'}\\
  &\qquad\qquad
  \P^{\pi_\cX}(s_h=s'\wedge s_{h-1}=s\wedge a_{h-1}=a\wedge \omega_{h-2}=\omega)\\
  &=\sum\limits_{x\in\cX_{h-1}}\sum\limits_{\omega\in x}
  \sum\limits_{s,a}\indic{\omega\cdot s\cdot a\in x'}\\
  &\qquad\qquad
  p_h(s'|s,a)\pi_\cX(a|s,\omega)\P^{\pi_\cX}(s_{h-1}=s\wedge \omega_{h-2}=\omega)\\  
  &\markref{(1)}{=}\sum\limits_{x\in\cX_{h-1}}
  \sum\limits_{s,a}\indic{x\cdot s\cdot a\in x'}
  p_h(s'|s,a)\pi_\cX(a|s,x)\\
  &\qquad\qquad
  \sum\limits_{\omega\in x}\P^{\pi_\cX}(s_{h-1}=s\wedge \omega_{h-2}=\omega)\\
  &=\sum\limits_{x\in\cX_{h-1}}
  \sum\limits_{s,a}\indic{x\cdot s\cdot a\in x'}
  p_h(s'|s,a)\pi_\cX(a|s,x)\\
  &\qquad\qquad
  \P^{\pi_\cX}(s_{h-1}=s\wedge \omega_{h-2}\in x)\\
  &\markref{(2)}{=}\sum\limits_{x\in\cX_{h-1}}
  \sum\limits_{s,a}\indic{x\cdot s\cdot a\in x'}
  p_h(s'|s,a)\pi_\cX(a|s,x)\\
  &\qquad\qquad
  \P^{\pi^E}(s_{h-1}=s\wedge \omega_{h-2}\in x)\\  
  &\markref{(3)}{=}\sum\limits_{x\in\cX_{h-1}}
  \sum\limits_{s,a}\indic{x\cdot s\cdot a\in x'}
  p_h(s'|s,a)\\
  &\qquad\qquad\P^{\pi^E}(s_{h-1}=s\wedge a_{h-1}=a\wedge \omega_{h-2}\in x)\\
  &=  \P^{\pi^E}(s_h=s'\wedge \omega_{h-1}\in x'),
\end{align*}
where at (1) we use notation $x\cdot s\cdot a$ to denote that all the
trajectories in $x$ when combined to a given $s,a$ give birth to the same $x'$
by hypothesis, at (2) we use the induction hypothesis, and at (3) the definition
of $\pi_\cX$.
\end{proof}

Then, \rsbc can be easily extended by counting the number of occurrences in each
set of trajectories, and also the theoretical guarantees can be easily extended
to this setting.

\section{Additional Results and Proofs for Section \ref{sec: r unknown}}
\label{apx: sec r unknown}

\estanypolicyallrewards*
\begin{proof}
  For any reward $r:\SAH\to[0,1]$, we can write:
  \begin{align*}
    \cW\Bigr{\eta_r^{\pi^E},\widehat{\eta}_r}
    &\markref{(1)}{\le}
    \cW\Bigr{\eta_r^{\pi^E},\com{\eta_{r_\theta}^{\pi^E}}}
    +
    \cW\Bigr{\com{\eta_{r_\theta}^{\pi^E}},\widehat{\eta}_r}\\
    &\markref{(2)}{\le}
    \com{\frac{H\theta}{2}}
    +
    \cW\Bigr{\eta_{r_\theta}^{\pi^E},\widehat{\eta}_r}\\
    &\markref{(3)}{\le}
    \frac{H\theta}{2}
    +
    H\epsilon',
  \end{align*}
  where at (1) we use triangle's inequality, at (2) we apply Lemma \ref{lemma:
  different r same p}, and at (3) we use the same derivation as in the proof of
  Theorem \ref{thr: rskt}, after having noticed that the estimate
  $\widehat{\eta}_r$ is the same estimate used in Line \ref{line: kt estimate
  expert ret distrib} of \rskt. So, by the DKW inequality, we have that the last
  passage holds with probability $1-\delta$ using a number of samples:
  \begin{align*}
    N\le\frac{1}{2(\epsilon')^2}\ln\frac{2}{\delta}.
  \end{align*}
  This holds for a single reward $r:\SAH\to[0,1]$. To make this hold for any
  possible reward, observe that it suffices to guarantee that it holds for all
  the rewards in the set $\cR$, defined as the set of all reward functions
  taking on discretized values:
  \begin{align*}
    \cR\coloneqq\Bigc{r:\SAH\to\cY_2^\theta}.
  \end{align*}
  Indeed, $\cR$ represents an $H\theta/2$-covering of the set of all the
  real-valued reward functions. Therefore, the result follows through the
  application of a union bound over all the rewards in $\cR$. Since they are
  $|\cY_2^\theta|^{SAH}$, then we obtain a number of samples:
  \begin{align*}
    N\le\frac{2SAH^3}{\epsilon^2}\ln\frac{2|\cY^\theta_2|}{\delta}
    \le \widetilde{\cO}\Bigr{\frac{SAH^3}{\epsilon^2}\ln\frac{1}{\delta}},
  \end{align*}
  to guarantee that $\cW\Bigr{\eta_r^{\pi^E},\widehat{\eta}_r}\le\epsilon$ for
  all $r$, where we set $\epsilon'=\epsilon/(2H)$, $\theta=\epsilon/H$, and
  $|\cY^\theta_2|\le 1+1/\theta=1+H/\epsilon$.
\end{proof}

\upperboundexpcompl*
\begin{proof}
Define the good event $\cE$ as:
\begin{align*}
  \cE\coloneqq\biggc{\max\limits_{r:\SAH\to[0,1]}\cW\Bigr{
      \eta^{\pi^E}_r,\widehat{\eta}_r
    }\le \epsilon}.
\end{align*}
Then, under $\cE$, it holds that:
\begin{align*}
  \max\limits_{r:\SAH\to[0,1]}\cW\Bigr{
      \eta^{\pi^E}_r,\eta^{\widehat{\pi}}_r
    }&\markref{(1)}{\le}
    \max\limits_{r:\SAH\to[0,1]}\cW\Bigr{
      \eta^{\pi^E}_r,\com{\widehat{\eta}_r}
    }
    +
    \max\limits_{r:\SAH\to[0,1]}
    \cW\Bigr{
      \com{\widehat{\eta}_r},\eta^{\widehat{\pi}}_r
    }\\
    &\markref{(2)}{\le}
    \com{\epsilon}
    +
    \max\limits_{r:\SAH\to[0,1]}
    \cW\Bigr{
      \widehat{\eta}_r,\eta^{\widehat{\pi}}_r
    }\\
    &\markref{(3)}{\le}
    \epsilon
    +
    \com{\min\limits_{\pi\in\Pi^{\text{NM}}}}\max\limits_{r:\SAH\to[0,1]}
    \cW\Bigr{
      \widehat{\eta}_r,\com{\eta^{\pi}_r}
    }\\
    &\markref{(4)}{\le}
    \epsilon
    +
    \max\limits_{r:\SAH\to[0,1]}
    \cW\Bigr{
      \widehat{\eta}_r,\com{\eta^{\pi^E}_r}
    }\\
    &\markref{(5)}{\le}
    \com{2\epsilon},
\end{align*}
where at (1) we use triangle's inequality, at (2) we use that event $\cE$ holds,
at (3) we use the definition of $\widehat{\pi}$, at (4) we upper bound the
minimum with a specific choice of reward in $\Pi^{\text{NM}}$, i.e., $\pi^E$,
and finally, at (5), we apply again that event $\cE$ holds.
The proof is concluded by applying Theorem \ref{thr: est any policy all
rewards}, which shows that, with the samples in Eq. \eqref{eq: sample complexity
est any r}, the event $\cE$ holds w.p. $1-\delta$.
\end{proof}

\section{Additional Details on Numerical Simulations}
\label{apx: additional details exps}

We describe how we sample an MDP and an expert's policy in the experiments in
Appendix \ref{apx: details on sampling}. In Appendix \ref{apx: RDM with markov
policies no LP}, we discuss the (im)possibility of implementing a variant of \rskt
that works with Markovian policies, in Appendix \ref{apx: impl details mimic-md}
we provide implementation details on \mimic, in Appendix \ref{apx: example} we
provide a graphical example of expert's return distribution and its estimates
from the four algorithms considered in Section \ref{sec: num simulations}, and,
finally, in Appendices \ref{apx: details Q1} and \ref{apx: details Q2} we
provide additional details and results on the simulations conducted to address,
respectively, questions 1 and 2. We mention that all simulations took place in
some hours on a AMD Ryzen 5 5500U processor.

\subsection{Details on Sampling the MDPs and the Expert's Policies}
\label{apx: details on sampling}

All the MDPs are sampled as follows. The initial state $s_0\in\cS$ is sampled
uniformly at random from $\cS$. The reward function $r^E:\SAH\to[0,1]$ is
sampled, in each $s,a,h$, uniformly at random from a set
$\{0,\rho,2\rho,\dotsc,\floor{1/\rho}\rho\}$, whose values are controlled by a
parameter $\rho\in(0,1]$ (intuitively, the difference $\theta-\rho$ gives
insights into the approximation error).
The transition model $p$ is obtained in two steps. First, we sample it uniformly
at random in each $s,a,h$ from the simplex $\Delta^\cS$ (by sampling from the
Dirichlet distribution), but then we make the transition of each $s,a,h$
deterministic with probability 0.7 to increase the variety of the MDP.

Regarding the expert's policy, when we sample Markovian expert's policies, we
simply sample them uniformly from the simplex $\Delta^\cA$ in every
$\cS\times\dsb{H}$. Instead, when we mention ``non-Markovian'' expert's
policies, then, of course, we cannot sample them uniformly at random from
$\Pi^{\text{NM}}$ due to the curse of dimensionality. Instead, what we do is
sample randomly from a parameteric subset of $\Pi^{\text{NM}}$ defined as
follows. We map states $\{s_1,\dotsc,s_S\}=\cS$ to $\{0,\dotsc,S-1\}$, and
actions $\{a_1,\dotsc,a_A\}=\cA$ to $\{0,\dotsc,A-1\}$. Then, each policy
projects each past history $(s_1,a_1,s_2,a_2,\dotsc,s_{h},a_{h})$ to $\RR^{16}$
after having mapped it to integers and padded with zeros to reach size $2H$, by
using a projection matrix of size $(2H,16)$ randomly generated.
Then, we obtain the probabilities of playing each action by multiplying such
16-dimensional vector by a matrix of weigths of size $(16,A)$ corresponding to
the current state $s_h$ (we randomly generated one of these weight matrices for
every state).

\subsection{On Addressing Return Distribution Matching with Markovian Policies}
\label{apx: RDM with markov policies no LP}

We mention that finding the best Markovian policy $\pi'\in\Pi^{\text{M}}$ with
return distribution closest to a given array $\eta$, i.e., addressing
$\min_{\pi\in\Pi^{\text{M}}}\cW(\eta_{r^E_\theta}^{\pi},\eta)$, seems a problem
that cannot be formulated as an LP and not even as a more general convex
optimization problem, as it seems to require bilinear constraints. If so,
generating random MDPs and (non-Markovian) expert policies to understand how
much Markovian policies are outperformed by non-Markovian policies for return
distribution matching may be difficult. We believe that understanding more
in-depth this fact will be interesting for future works.

\subsection{Implementation Details of MIMIC-MD}\label{apx: impl details mimic-md}

Algorithm \mimic has been devised by
\citet{rajaraman2020fundamentalimitationlearning} but an efficient LP
formulation is provided in \citet{rajaraman2021provablybreakingquadraticerror}
(see its proof of Theorem 2). Although it concerns only deterministic expert's
policies, we extend it to stochastic policies as in the full algorithm as
follows.
Simply, in addition to the constraints for $d$ being a valid occupancy measure,
we replace constraint $d_h(s,a)=\indic{a=\pi^E_h(s)}$ (valid for deterministic
policies) with $d_h(s,a)=\frac{N_h(s,a)}{N_h(s)}\sum_{a'}d_h(s,a')$ where
$N_h(s)\neq0$ to extend to stochatic experts. Note that it is linear in $d$.

\subsection{An Example of Expert's Return Distribution and its Estimates}
\label{apx: example}

In Figure \ref{fig: experiments}, we plot the (discretized) return distribution
of the expert's policy $\eta^{\pi^E}_{r^E}$ (in blue, $\eta^E$) obtained after
having sampled at random an MDP with size $S,A,H=(3,2,8)$ and the policy as
well. We discretized it with bins at a distance of $0.1$ for computation and
plotting.

Then, we have sampled a dataset $\cD^E$ of $N=300$ trajectories from $\pi^E$,
and computed an estimate of $\eta^E$ as in Line \ref{line: kt
estimate expert ret distrib} of Alg. \ref{alg: rskt} using $\theta=0.05$ (in
green, $\widehat{\eta}$). Note that $\widehat{\eta}$ is quite close to
$\eta^E$ with this amount of trajectories.

We have then given $\cD^E$ (and potentially the transition model) in input to
the four algorithms considered in Section \ref{sec: num simulations}, i.e.,
\rsbc, \rskt, \bc and \mimic, obtaining the return distributions
$\eta^{\text{RS-BC}}$, $\eta^{\text{RS-KT}}$, $\eta^{\text{BC}}$ and
$\eta^{\text{MIMIC-MD}}$.

Observe that \rsbc tends to approximate $\eta^E$ directly, while \rskt tries to
match $\widehat{\eta}$ instead. Anyway, both return distributions are close to
$\eta^E$. Instead, \bc and \mimic, by working with Markovian policies, are
biased, and their return distributions are not that close to $\eta^E$.

\begin{figure}[!h]
  \centering
  \begin{minipage}[t!]{0.9\textwidth}
    \centering
    \includegraphics[width=0.98\linewidth]{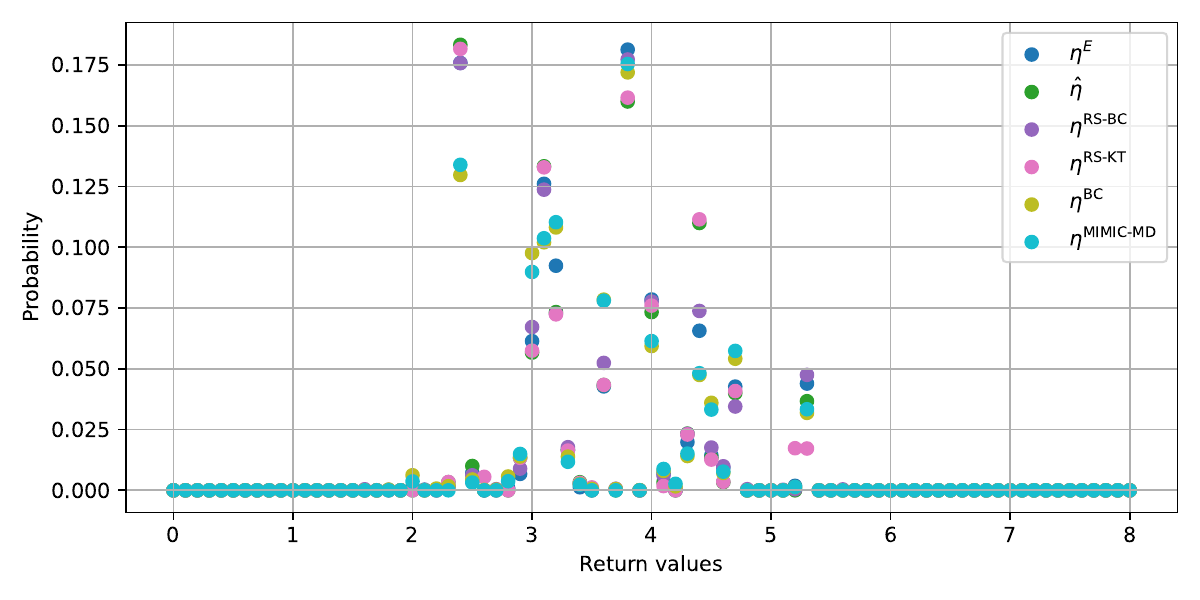}
\end{minipage}
\caption{An example of expert'r return distribution and various estimates
computed with \rsbc, \rskt, \bc and \mimic.}
\label{fig: experiments}
 \end{figure}

\subsection{Additional Details on Q1}\label{apx: details Q1}

\begin{table}[!t]
  \centering
  \resizebox{0.9\columnwidth}{!}{
  \begin{tabular}{||c | c c c c c||} 
   \hline
    & $N=20$ & $N=80$ & $N=300$ & $N=1000$ & $N=10000$\\
   \hline
   \rsbc & \small \bf0.101±0.041 & \small \bf0.059±0.017 & \small \bf0.032±0.011 & \small \bf0.016±0.006 & \small \bf0.005±0.002\\
   \hline
    \rskt & \small 0.164±0.04 & \small 0.084±0.021 & \small 0.052±0.013 &
    \small 0.033±0.01 & \small  0.02±0.006 \\
   \hline
    \bc & \small  \bf0.104±0.041 & \small \bf0.058±0.018 & \small \bf0.035±0.012 & \small 0.024±0.01 & \small 0.018±0.009 \\
   \hline
    \mimic & \small 0.139±0.058 & \small 0.079±0.028 & \small 0.045±0.016 & \small 0.029±0.012 & \small 0.019±0.009 \\
   \hline
  \end{tabular}%
  }
  \caption{Results of simulation with $S,A,H=(50,5,5)$ for Q1.}
\label{table: exp 3}
\end{table}

\begin{table}[!t]
  \centering
  \resizebox{0.9\columnwidth}{!}{
  \begin{tabular}{||c | c c c c c||} 
   \hline
    & $N=20$ & $N=80$ & $N=300$ & $N=1000$ & $N=10000$\\
   \hline
   \rsbc & \small \bf0.193±0.086 & \small \bf0.087±0.035 & \small \bf0.046±0.019 & \small \bf0.027±0.01 & \small \bf0.012±0.005\\
   \hline
    \rskt & \small 0.223±0.066 & \small 0.115±0.034 & \small 0.072±0.023 & \small 0.053±0.017 & \small  0.041±0.018 \\
   \hline
    \bc & \small  0.208±0.08 & \small 0.162±0.083 & \small 0.156±0.086 & \small 0.151±0.086 & \small 0.151±0.085 \\
   \hline
    \mimic & \small 0.265±0.106 & \small 0.18±0.086 & \small 0.159±0.084 & \small 0.153±0.087 & \small 0.15±0.085 \\
   \hline
  \end{tabular}%
  }
  \caption{Results of simulation with $S,A,H=(2,2,20)$ for Q1.}
\label{table: exp 6}
\end{table}

For simulations, we set $\rho=0.03<\theta$, meaning that there is some
approximation error.

Regarding Table \ref{table: exp 3}, we would like to discuss some points. First,
we mention that the increase in size of $S$ and $A$ is not sufficiently big to
permit to \rskt to outperform the sample complexity of \rsbc, as discussed in
Q4.
Second, increasing $S,A,H$ makes \rskt much more time-consuming, as it requires
solving an LP with much more variables and constraints.
Third, \rskt performs comparably to \bc and \mimic, but this is due to an
increment of approximation error due to the increase of $S,A$, as clear from
Table \ref{table: exp 5}, where in absence of approximation error \rskt
outperforms \bc and \mimic. Note that this is not the fact for \rsbc, which
seems more robust to approximation error for this problem size (intuitively, the
reason is that it is strictly more expressive than \bc for any choice of
$\theta$).

Regarding Table \ref{table: exp 6}, we mention that a larger $H$ implies a
larger approximation error in particular for \rskt, as clear from Table
\ref{table: exp 8}, where in absence of approximation error \rskt outperforms
\bc and \mimic.

\subsection{Additional Details on Q2}\label{apx: details Q2}

\begin{table}[!t]
  \centering
  \resizebox{0.9\columnwidth}{!}{
  \begin{tabular}{||c | c c c c c||} 
   \hline
    & $N=20$ & $N=80$ & $N=300$ & $N=1000$ & $N=10000$\\
   \hline
   \rsbc & \small \bf 0.081±0.036 & \small \bf0.035±0.016 & \small \bf0.019±0.011 & \small \bf0.011±0.005 & \small \bf0.004±0.002\\
   \hline
    \rskt & \small 0.095±0.042 & \small 0.043±0.019 & \small 0.024±0.012 & \small \bf0.013±0.005 & \small \bf0.004±0.002 \\
   \hline
    \bc & \small 0.104±0.053 & \small 0.076±0.048 & \small 0.068±0.048 & \small 0.066±0.049 & \small 0.065±0.049 \\
   \hline
    \mimic & \small 0.13±0.063 & \small 0.085±0.046 & \small 0.071±0.046 & \small 0.068±0.049 & \small 0.065±0.049 \\
   \hline
  \end{tabular}%
  }
  \caption{Results of simulation with $S,A,H=(2,2,5)$ and $\theta=\rho$ for Q2.}
\label{table: exp 1}
\end{table}

\begin{table}[!t]
  \centering
  \resizebox{0.9\columnwidth}{!}{
  \begin{tabular}{||c | c c c c c||} 
   \hline
    & $N=20$ & $N=80$ & $N=300$ & $N=1000$ & $N=10000$\\
   \hline
   \rsbc & \small \bf0.088±0.026 & \small \bf0.048±0.022 & \small \bf0.025±0.011 & \small \bf0.012±0.005 & \small\bf 0.004±0.002\\
   \hline
    \rskt & \small 0.135±0.037 & \small 0.068±0.02 & \small 0.034±0.01 & \small 0.019±0.005 & \small \bf 0.006±0.002 \\
   \hline
    \bc & \small  \bf0.092±0.032 & \small 0.054±0.029 & \small 0.038±0.023 & \small 0.031±0.022 & \small 0.028±0.024 \\
   \hline
    \mimic & \small 0.118±0.042 & \small 0.066±0.028 & \small 0.044±0.023 & \small 0.033±0.021 & \small 0.029±0.024 \\
   \hline
  \end{tabular}%
  }
  \caption{Results of simulation with $S,A,H=(20,3,5)$ for Q2.}
\label{table: exp 5}
\end{table}

\begin{table}[!t]
  \centering
  \resizebox{0.9\columnwidth}{!}{
  \begin{tabular}{||c | c c c c c||} 
   \hline
    & $N=20$ & $N=80$ & $N=300$ & $N=1000$ & $N=10000$\\
   \hline
   \rsbc & \small \bf 0.177±0.067 & \small \bf 0.091±0.038 & \small \bf
   0.047±0.018 & \small \bf 0.023±0.008 & \small \bf0.008±0.003\\
   \hline
    \rskt & \small 0.224±0.083 & \small 0.108±0.039 & \small 0.057±0.018 &
    \small \bf0.031±0.01 & \small \bf0.011±0.004 \\
   \hline
    \bc & \small 0.196±0.104 & \small 0.159±0.102 & \small 0.148±0.103 & \small
    0.145±0.104 & \small 0.144±0.106 \\
   \hline
    \mimic & \small 0.246±0.115 & \small 0.174±0.103 & \small 0.151±0.103 & \small 0.145±0.104 & \small 0.144±0.106 \\
   \hline
  \end{tabular}%
  }
  \caption{Results of simulation with $S,A,H=(2,2,20)$ for Q2.}
\label{table: exp 8}
\end{table}

The three additional simulations have all a non-Markovian expert and
$\rho=\theta$ (to enforce no approximation error) with parameters
$S,A,H,\theta\in\{(2,2,5,5e-2),(20,3,5,5e-2),(2,2,20,1e-1)\}$, and the results
are reported respectively in Tables \ref{table: exp 1}, \ref{table: exp 5} and
\ref{table: exp 8}.

By comparing these tables respectively with Table \ref{table: Q} (top),
\ref{table: exp 3} and \ref{table: exp 6}, where there is approximation error
due to $\theta=0.05>0.03=\rho$, we realize that the approximation error mostly
concerns \rskt and with larger horizons $H$ (as expected from Lemma \ref{lemma:
apx policies}, and from knowledge of how \rsbc works).

\end{document}

%% file: rs_bc.tex
\RestyleAlgo{ruled}
\LinesNumbered
\begin{algorithm}[H]\small
    \caption{\rsbc (\rsbclong)}
    \label{alg: rsbc}
\small
\DontPrintSemicolon
\SetKwInOut{Input}{Input}

\Input{Dataset $\cD^E$ $=\{(s_1^i,
a_1^i,\dotsc,s_{H}^i,a_{H}^i)\}_{i\in\dsb{N}}$, reward $r^E$,
parameter $\theta$
}

\Comment{Count the state-action-cumulative reward occurrences}

$M_h(s,g,a)\,$\hbox{$\gets \sum_{i\in\dsb{N}}\indic{s_h^i=s,a_h^i=a,
\sum_{h'=1}^{h-1}r^E_{\theta,h'}(s_{h'}^i,a_{h'}^i)=g}$\quad
{\thinmuskip=2.5mu \medmuskip=2.5mu
\thickmuskip=2.5mu
$\forall h\in\dsb{H},s\in\cS,g\in\cY^\theta_h,a\in\cA$
}}
\label{line: bc init N}





\Comment{Retrieve the policy}

$\widehat{\pi}$\hbox{$(a|s,\omega)\gets\begin{cases}
    \frac{M_h(s,G(\omega;r_\theta^E),a)}{\sum_{a'}M_h(s,G(\omega;r_\theta^E),a')}
    &\text{if }{\thinmuskip=1.5mu \medmuskip=1.5mu
\thickmuskip=1.5mu\sum_{a'}M_h(s,G(\omega;r_\theta^E),a')>0}\\
    \frac{1}{A}&\text{otherwise} \end{cases}\,
    {\thinmuskip=2.5mu \medmuskip=2.5mu
\thickmuskip=2.5mu
\forall h\in\dsb{H},s\in\cS,\omega\in\Omega_h,a\in\cA
}$}
\label{line: bc retrieve policy}
    
\textbf{Return} $\widehat{\pi}$

\end{algorithm}

%% file: rs_kt.tex
\RestyleAlgo{ruled}
\LinesNumbered
\begin{algorithm}[H]\small
    \caption{\rskt (\rsktlong)}
    \label{alg: rskt}
\small
\DontPrintSemicolon
\SetKwInOut{Input}{Input}

\Input{ Dataset $\cD^E$ $=\{(s_1^i,
a_1^i,\dotsc,s_{H}^i,a_{H}^i)\}_{i\in\dsb{N}}$, reward $r^E$,
parameter $\theta$, transition model $p$ }

\Comment{Estimate the return distribution of the expert $\eta_{r^E}^{\pi^E}$}

$\widehat{\eta}(g)\gets \frac{1}{N}\sum_{i\in\dsb{N}}
\indic{\sum_{h=1}^H r^E_{\theta,h}(s_{h}^i,a_{h}^i)=g}$
\quad $\forall g\in\cY^{\theta}$
  \label{line:
  kt estimate expert ret distrib}

\Comment{Compute the policy in $\Pi(r_\theta^E)$ closest
to $\widehat{\eta}$ via Eq. \eqref{eq: opt problem LP rskt}}

$\widehat{\pi}\in\argmin\nolimits_{\pi\in\Pi(r_\theta^E)}
\cW\bigr{\eta^{\pi}_{r_\theta^E}, \widehat{\eta}}$\label{line: kt compute policy}

\textbf{Return} $\widehat{\pi}$

\end{algorithm}

%% file: rs_bc-set.tex
\RestyleAlgo{ruled}
\LinesNumbered
\begin{algorithm}[H]\small
    \caption{Variant of \rsbc for a set of rewards}
    \label{alg: rsbc set}
\small
\DontPrintSemicolon
\SetKwInOut{Input}{Input}

\Input{Dataset $\cD^E$ $=\{(s_1^i,
a_1^i,\dotsc,s_{H}^i,a_{H}^i)\}_{i\in\dsb{N}}$, set of rewards
$\cR=\{r^1,\dotsc,r^d\}$, parameter $\theta$ }

\Comment{Count the state-action-cumulative reward occurrences}

$M_h(s,g^1,\dotsc,g^d,a)\gets \sum_{i\in\dsb{N}}\indic{s_h^i=s,a_h^i=a,
\sum_{h'=1}^{h-1}r^1_{\theta,h'}(s_{h'}^i,a_{h'}^i)=g^1,\dotsc,
\sum_{h'=1}^{h-1}r^d_{\theta,h'}(s_{h'}^i,a_{h'}^i)=g^d}$
\quad
$\forall h\in\dsb{H},s\in\cS,g^1\in\cY^\theta_h,\dotsc,g^d\in\cY^\theta_h,a\in\cA$





\Comment{Retrieve the policy}

$\widehat{\pi}(a|s,\omega)\gets \begin{cases}
    \frac{M_h(s,G(\omega;r_\theta^1),\dotsc,G(\omega;r_\theta^d),a)}{\sum_{a'}
    M_h(s,G(\omega;r_\theta^1),\dotsc,G(\omega;r_\theta^d),a')}
    &\text{if denominator}>0\\
    \frac{1}{A}&\text{otherwise} \end{cases}\quad
    {\thinmuskip=2.5mu \medmuskip=2.5mu
\thickmuskip=2.5mu
\forall h\in\dsb{H},s\in\cS,\omega\in\Omega_h,a\in\cA
}$
\label{line: bc retrieve policy set}
    
\textbf{Return} $\widehat{\pi}$

\end{algorithm}

%% file: rs_kt-set.tex
\RestyleAlgo{ruled}
\LinesNumbered
\begin{algorithm}[H]\small
    \caption{Variant of \rskt for a set of rewards}
    \label{alg: rskt set}
\small
\DontPrintSemicolon
\SetKwInOut{Input}{Input}

\Input{ Dataset $\cD^E$ $=\{(s_1^i,
a_1^i,\dotsc,s_{H}^i,a_{H}^i)\}_{i\in\dsb{N}}$, set of rewards
$\cR=\{r^1,\dotsc,r^d\}$,
parameter $\theta$, transition model $p$ }

\Comment{Estimate the return distribution of the expert $\eta_{r}^{\pi^E}$
for any $r\in\cR$}

$\widehat{\eta}_r(g)\gets \frac{1}{N}\sum_{i\in\dsb{N}}
\indic{\sum_{h=1}^H r_{\theta,h}(s_{h}^i,a_{h}^i)=g}$
\quad $\forall g\in\cY^{\theta}$, $\forall r\in\cR$
  \label{line:
  kt estimate expert ret distrib set}

\Comment{Compute the policy in $\Pi(\cR^\theta)$ closest to
$\widehat{\eta}_r\forall r\in\cR$ via Eq. \eqref{eq: opt problem LP rskt set}}

$\widehat{\pi}\in\argmin\nolimits_{\pi\in\Pi(\cR^\theta)}\max\nolimits_{r\in\cR}
\cW\bigr{\eta^{\pi}_{r_\theta}, \widehat{\eta}}$\label{line: kt compute policy
set}

\textbf{Return} $\widehat{\pi}$

\end{algorithm}